\documentclass[12pt]{article}
\usepackage{natbib}
\usepackage{enumitem}
\usepackage[hidelinks]{hyperref}
\usepackage{amsmath}
\allowdisplaybreaks[3]
\usepackage{graphicx}
\usepackage{placeins}
\usepackage[font=small,labelfont=bf]{caption}

	\usepackage[margin=1in]{geometry}
	\hypersetup{
		pdftitle={The Noise Premium in Adversarial Training for Kernel Regression},
		pdfauthor={Yiling Xie and Xiaoming Huo}
	}

\usepackage{amsthm}
\newtheorem{theorem}{Theorem}[section]
\newtheorem{lemma}[theorem]{Lemma}
\newtheorem{corollary}[theorem]{Corollary}
\newtheorem{proposition}[theorem]{Proposition}
\newtheorem{remark}[theorem]{Remark}
\newtheorem{assumption}[theorem]{Assumption}
\usepackage{amsfonts}
\usepackage{amssymb}

\newcommand{\E}{\mathbb{E}}
\newcommand{\R}{\mathbb{R}}
\newcommand{\Hk}{\mathcal{H}_k}
\newcommand{\norm}[1]{\left\lVert #1\right\rVert}
\newcommand{\inner}[2]{\left\langle #1,#2\right\rangle}
\newcommand{\argmin}{\operatornamewithlimits{argmin}}
\newcommand{\sgn}{\operatorname{sgn}}
\newcommand{\proofstep}[1]{\par\medskip\noindent\textit{#1}\par\smallskip}
\begin{document}
	
	\def\spacingset#1{\renewcommand{\baselinestretch}%
		{#1}\small\normalsize} \spacingset{1}
	
	
	\title{\bf The Noise Premium in Adversarial Training for Kernel Regression}
	\author{Yiling Xie\textsuperscript{1} \qquad Xiaoming Huo\textsuperscript{2}\\[0.5em]
		\small \textsuperscript{1}Department of Decision Analytics and Operations,
		City University of Hong Kong\\
		\small \textsuperscript{2}H. Milton Stewart School of Industrial and Systems Engineering\\
		\small Georgia Institute of Technology\\[0.3em]
		\small \texttt{yiling.xie@cityu.edu.hk} \qquad \texttt{huo@gatech.edu}}
	\date{}
	\maketitle
	
	\begin{abstract}
			Adversarial training can improve the robustness of predictive models to bounded perturbations, often at the cost of statistical efficiency. We study this trade-off in kernel regression over a reproducing kernel Hilbert space (RKHS). It is shown that, under squared loss, adversarial training in RKHS introduces a term involving the product of the function norm with the mean absolute value of the response noise, which we call the \textit{noise premium}. Our analysis shows that the noise premium makes the prediction error of adversarial training converge strictly more slowly than the nonparametric minimax benchmark even after balancing approximation and estimation errors. Moreover, for a fixed perturbation budget, once the budget exceeds a certain threshold, the solution to adversarial training collapses to the zero function. To mitigate these effects of the noise premium, we propose noise-debiased adversarial training. The resulting noise-debiased estimator can attain the minimax optimal rate up to a logarithmic factor for the prediction error, raises the collapse threshold, and admits an explicit bound on the increase in adversarial loss. Numerical experiments on synthetic and real data support the theoretical findings and validate the effectiveness of the proposed  noise-debiased method.
	\end{abstract}
	
	\spacingset{1.8}
	\section{Introduction}\label{sec:introduction}
	
		Adversarial training improves model robustness to bounded perturbations by minimizing
		the worst-case loss over a prescribed perturbation neighborhood
		\citep{goodfellow2015explaining,madry2018towards}.  This protection can come
		at the cost of lower accuracy on unperturbed data, giving rise to the
		robustness--accuracy trade-off
		\citep{tsipras2019robustness,zhang2019theoretically}.  In this paper, we study
		this trade-off in nonparametric kernel regression.

		We study this trade-off through feature-space adversarial training in the
		reproducing kernel Hilbert space (RKHS).  Following the formulation in
		\citet{shafieezadeh2019regularization} and \citet{ribeiro2025kernel}, the
		adversary perturbs each kernel feature within an RKHS ball of radius
		\(\delta\).  Notably, the resulting problem is a tractable convex surrogate for
		input-space adversarial training.  Whenever admissible input perturbations
		change the kernel feature
		by at most \(\delta\) in RKHS norm, the feature-space objective upper-bounds
		the corresponding input-space adversarial loss.  For a linear kernel,
		feature-space and input-space adversarial training coincide.  The RKHS formulation also permits a spectral analysis
		of prediction error through the associated kernel integral operator.

		More specifically, we consider the kernel regression model
		\(Y=f^*(X)+\varepsilon\), where \(f^*\) belongs to the reproducing kernel
		Hilbert space \(\mathcal H_k\) induced by a kernel \(k\), and
		\(\varepsilon\) denotes the response noise. Under squared loss,
		feature-space adversarial training with perturbation radius \(\delta\)
		is equivalent to minimizing
		\begin{equation}
			\label{eq:intro-expansion}
			\frac1n\sum_{i=1}^n(Y_i-f(X_i))^2
			+2\delta\norm{f}_{\Hk}\frac1n\sum_{i=1}^n|Y_i-f(X_i)|
			+\delta^2\norm{f}_{\Hk}^2
		\end{equation}
		over \(f\in\Hk\). The middle term couples the RKHS norm with the mean absolute residual. At the regression function $f^\ast$, the residual is \(\varepsilon\), so the associated population objective contains
		$2\delta\E|\varepsilon|\norm{f}_{\Hk}$.
		We call this noise-dependent norm penalty the \emph{noise premium}. The noise premium is first order in \(\delta\), does not vanish at \(f^\ast\), and acts on a different scale from the quadratic penalty \(\delta^2\norm{f}_{\Hk}^2\). We aim to study the statistical consequences of the noise premium.

		To quantify the statistical effect of the noise premium, we analyze the prediction error through the associated kernel integral operator, decomposing the prediction error into approximation and estimation errors and characterizing their dependence on the perturbation radius, sample size, source smoothness, and kernel spectrum. The resulting upper bounds, together with
		a matching lower bound, establish a sharp prediction
		rate.  Even when the perturbation radius balances approximation and estimation
		errors, the noise premium makes adversarial training converge strictly more
		slowly than the classical nonparametric minimax benchmark
		\citep{caponnetto2007optimal}.  The rate gap comes from the population
		approximation error induced by the noise premium.  At a fixed perturbation
		radius, the noise premium lowers the population threshold for signal
		collapse.  We show that the collapse is governed
		by the signal strength relative to the absolute response level, which includes
		irreducible response noise.  Once the radius exceeds the threshold, the population
		solution is the zero function.  Increasing the sample size reduces estimation
		error but does not alter this population threshold.

		Our analysis of the noise premium motivates a two-stage noise-debiased
		procedure. The first stage estimates the mean absolute response noise from
		independent data, and the second subtracts the estimated contribution from the
		adversarial objective. The resulting estimator attains the nonparametric
		minimax prediction rate up to a logarithmic factor and raises the population
		collapse threshold. We also derive an explicit bound on the increase in
		adversarial loss, which quantifies the associated robustness cost. The resulting problem
	    has an equivalent finite-dimensional convex formulation, and
		synthetic and real-data experiments support the theoretical results.
	
	\subsection{Related work}
Statistical analyses of adversarial training have primarily focused on
parametric models
\citep{schmidt2018adversarially,yin2019rademacher,tsipras2019robustness,
	zhang2019theoretically,javanmard2020precise,ribeiro2023regularization,
	xie2024highdimensional,xie2025asymptotic}. These works study sample
complexity, precise asymptotics, robustness--accuracy trade-offs, and minimax
optimality. For example, in high-dimensional
linear regression, adversarial training can attain the minimax prediction
rate \citep{xie2024highdimensional}. Interestingly, in the nonparametric RKHS
setting studied in this paper, our spectral analysis shows that the noise
premium makes the prediction rate of adversarial training strictly slower than
the classical nonparametric minimax rate within the smoothness regime considered.
	
	Robust learning in the RKHS and nonparametric regression has also been studied
	under several perturbation models.  Mass-transport distributional robustness
	leads to regularized kernel procedures
	\citep{shafieezadeh2019regularization}.
	\citet{ribeiro2025kernel} derive the feature-space reformulation studied in this paper and
	analyze the fixed-design behavior.
	\citet{peng2025adversarial} establish minimax rates over H\"older classes
	under test-time input perturbations, whereas \citet{moradi2025adversarial}
	study smoothing splines under adversarial corruption of the training
	responses.  
	In this paper, we instead study the \(L^2(\mu)\) prediction error of  adversarial training over RKHS under random design.
	
Various methods have been proposed to mitigate the robustness--accuracy
trade-off, including modified adversarial training objectives and approaches
that leverage additional unlabeled data
\citep{zhang2019theoretically,raghunathan2020understanding}. These methods
have largely been studied in classification or finite-dimensional models.
	In nonparametric RKHS regression, our noise-debiased procedure removes the
	response-noise component while leaving the rest of the adversarial objective
	unchanged.
	
Finally, our analysis builds on the statistical theory of kernel learning.
Kernel integral operators provide a standard framework for describing source
smoothness, spectral decay, effective dimension, and prediction rates in
nonparametric regression
\citep{caponnetto2007optimal,mendelson2010regularization,
	fischer2020sobolev,smale2007learning,wu2006learning,
	zhang2025optimal}. This operator-theoretic framework has been less developed
for the analysis of adversarial training. We use kernel spectral techniques
to separate the approximation and estimation errors, quantify their
dependence on the perturbation radius, and derive the resulting prediction
rate.

	\subsection{Notation}
	We write \(L^2(\mu)\) for the space of square-integrable functions with
	\(\norm{f}_{L^2(\mu)}^2=\int_{\mathcal X}f(x)^2\,d\mu(x)\).  The norm and
	inner product of \(\Hk\) are denoted by \(\norm{f}_{\Hk}\) and
	\(\inner{f}{h}_{\Hk}\).  For a bounded linear operator \(A\) between Hilbert
	spaces, \(\norm{A}_{\mathrm{op}}\) denotes its operator norm, with the
	underlying spaces determined by context.  For finite-dimensional vectors,
	\(\norm{\cdot}_2\) denotes the Euclidean norm.  We write \(a\lesssim b\) if
	\(a\le Cb\), and \(a\asymp b\) when each quantity is bounded by a constant
	multiple of the other.  The letters \(c\) and \(C\) denote positive
	constants.  Their values may change between derivations but are fixed within
	each theorem statement.

In the rest of this paper,	Sections~\ref{sec:problem}--\ref{sec:collapse} develop the formulation and
	theory, Section~\ref{sec:num} presents the numerical experiments, and all
	proofs are provided in the Appendix.
	
	\section{Problem Formulation}\label{sec:problem}
	This section formulates adversarial training in RKHS and states the
	assumptions used in the analysis.
	\subsection{Kernel setup and assumptions}
	Let \(\mu\) denote the distribution of \(X\) on a measurable input space
	\(\mathcal X\), and let
	\(k:\mathcal X\times\mathcal X\to\mathbb R\) be a measurable
	positive-definite kernel.
	Let \(\Hk\) be the separable reproducing kernel Hilbert space associated with \(k\).
	For \(x\in\mathcal X\), write \(k_x:=k(x,\cdot)\in\Hk\). By the
	reproducing property,
	\[
	f(x)=\langle f,k_x\rangle_{\Hk},
	\qquad
	\norm{k_x}_{\Hk}^2=k(x,x),
	\qquad f\in\Hk.
	\]
	Throughout the paper, we assume that \(k\) is uniformly bounded:
	\(\sup_{x\in\mathcal X}k(x,x)\le1\).
	The kernel integral operator \(T_k:L^2(\mu)\to L^2(\mu)\) associated with
	\(k\) is defined by
	\begin{equation}\label{def:kernel-integral-operator}
		(T_k h)(x)
		=
		\int_{\mathcal X} k(x,x')h(x')\,d\mu(x'),
		\qquad h\in L^2(\mu).
	\end{equation}
	We use the same notation for the associated RKHS realization
	\(T_kh=\E[h(X)k_X]\in\Hk\), which satisfies
	\(\inner{T_kh}{f}_{\Hk}=\E[h(X)f(X)]\) for \(f\in\Hk\).
	The two
	realizations have the same positive eigenvalues, so the spectral and
	effective-dimension notation below is unambiguous.
	Since \(T_k\) is positive, self-adjoint, and compact, let
	\((\sigma_j,\varphi_j)_{j\ge1}\) denote its positive eigenpairs, with
	\[
	T_k\varphi_j=\sigma_j\varphi_j,
	\qquad
	\sigma_1\ge \sigma_2\ge\cdots>0,
	\]
	where \((\varphi_j)_{j\ge1}\) is orthonormal in \(L^2(\mu)\). For any
	\(\alpha>0\), the fractional power of \(T_k\) is defined spectrally by
	\begin{equation}\label{spectralcompute}
		T_k^\alpha h
		=
		\sum_{j\ge1}
		\sigma_j^\alpha
		\langle h,\varphi_j\rangle_{L^2(\mu)}
		\varphi_j,
		\qquad h\in L^2(\mu).
	\end{equation}
	All spectral expansions are understood on the \(L^2(\mu)\)-closure of
	\(\Hk\).  In particular, if \(h=\sum_{j\ge1}h_j\varphi_j\in\Hk\), then
	\[
	\norm{h}_{\Hk}^2
	=
	\sum_{j\ge1}\frac{h_j^2}{\sigma_j}.
	\]

	We impose the following assumptions on the ground-truth data generating mechanism and the noise
	distribution.
	
	\begin{assumption}[Data-generating mechanism]\label{ass:model}
		The response \(Y\) follows the additive regression model \(Y=f^*(X)+\varepsilon\),
		where \(f^*\in\Hk\), and the noise \(\varepsilon\) is square integrable,
		independent of \(X\), and satisfies \(\E[\varepsilon]=0\).
	\end{assumption}

	\begin{assumption}[Quadratic behavior of the absolute risk]\label{ass:noise}
		For some \(0<\kappa<\infty\),
		\[
		0
		\le
		\E|\varepsilon-t|-\E|\varepsilon|
		\le
		\kappa t^2,
		\qquad t\in\R.
		\]
	\end{assumption}
	
	\begin{assumption}[Noise tails]\label{ass:noise-tail}
		For some \(\sigma<\infty\) and every integer \(p\ge2\),
		\[
		\E|\varepsilon|^p\le \frac{p!}{2}\sigma^p.
		\]
	\end{assumption}
	
	Assumption~\ref{ass:model} specifies an additive regression model with
	square-integrable, independent, and mean-zero noise. Assumption~\ref{ass:noise}
	states that the excess absolute risk grows at most quadratically around its
	minimum at zero.  It holds, for example, when the noise has median zero and
	a uniformly bounded density.
	Together, Assumptions~\ref{ass:model} and~\ref{ass:noise} imply
	\(\E|\varepsilon|>0\), so the noiseless model is excluded.
	Assumption~\ref{ass:noise-tail} provides the moment control needed to apply
	Bernstein-type concentration inequalities to the stochastic error terms.

	We then impose two assumptions on the kernel integral operator $T_k$. The source
	condition quantifies the smoothness of \(f^*\) through the spectral structure
	of \(T_k\), whereas the eigenvalue-decay condition controls the spectral
	complexity induced jointly by the kernel \(k\) and the input distribution
	\(\mu\). These conditions will determine the learning rates and are widely used in
	the kernel learning literature; see, e.g.,
	\citet{caponnetto2007optimal}.

	\begin{assumption}[Source condition]\label{ass:source}
		For fixed \(0<\alpha\le1\) and \(R>0\), there exists \(g\in\Hk\) such that
		\[
		f^*=T_k^\alpha g,
		\qquad
		\norm{g}_{\Hk}\le R,
		\]
		where \(T_k^\alpha\) is defined by spectral calculus~\eqref{spectralcompute}.
	\end{assumption}
	The upper bounds stated under Assumption~\ref{ass:source} are uniform over
	all \(g\in\Hk\) with \(\norm{g}_{\Hk}\le R\). Their constants do not depend
	on the particular choice of \(g\).
	
	\begin{assumption}[Polynomial eigenvalue decay]\label{ass:eigendecay}
		We assume that the positive
		eigenvalues of \(T_k\) satisfy
		\[
		c_\beta j^{-\beta}
		\le
		\sigma_j
		\le
		C_\beta j^{-\beta},
		\qquad j\ge 1,
		\]
		for some constants \(0<c_\beta\le C_\beta<\infty\) and \(\beta>1\).
	\end{assumption}
	
	\subsection{Adversarial training in an RKHS}
	
	Given i.i.d. observations \((X_i,Y_i)_{i=1}^n\) and an adversarial robustness level
	\(\delta>0\), this paper studies the adversarial training problem in an RKHS as follows:
	\begin{equation}
		\label{eq:rkhs-at-minimax}
		\min_{f\in\Hk}
		\max_{\substack{
				\Delta_1,\ldots,\Delta_n\in\Hk\\
				\norm{\Delta_i}_{\Hk}\le\delta,\ i=1,\ldots,n
		}}
		\frac1n\sum_{i=1}^n
		\left(
		Y_i-
		\left\langle f,k_{X_i}+\Delta_i\right\rangle_{\Hk}
		\right)^2.
	\end{equation}
	The following proposition gives an equivalent reformulation of
	problem~\eqref{eq:rkhs-at-minimax}.
	\begin{proposition}[Equivalent reformulation, \citet{ribeiro2025kernel}]
		\label{prop:minimax-equivalence}
		For every \(f\in\Hk\),
		\[
		\begin{aligned}
			\max_{\substack{
					\Delta_1,\ldots,\Delta_n\in\Hk\\
					\norm{\Delta_i}_{\Hk}\le\delta,\ i=1,\ldots,n
			}}
			\frac1n\sum_{i=1}^n
			\left(
			Y_i-
			\left\langle f,k_{X_i}+\Delta_i\right\rangle_{\Hk}
			\right)^2=
			\frac1n\sum_{i=1}^n
			\left(
			|Y_i-f(X_i)|+\delta\norm{f}_{\Hk}
			\right)^2.
		\end{aligned}
		\]
		Consequently, problem~\eqref{eq:rkhs-at-minimax} is equivalent to
		\begin{equation}
			\label{problem}
			\min_{f\in\Hk}
			\frac1n\sum_{i=1}^n
			\left(
			|Y_i-f(X_i)|+\delta\norm{f}_{\Hk}
			\right)^2.
		\end{equation}
	\end{proposition}
	
	We next interpret the robustness encoded by
	problem~\eqref{eq:rkhs-at-minimax}. The formulation is directly robust to
	perturbations of the kernel representation, allowing each \(k_{X_i}\) to be
	independently perturbed within an RKHS ball of radius \(\delta\). It also
	provides robustness to input-space perturbations whenever the induced feature
	perturbations satisfy \(\norm{k_{X_i'}-k_{X_i}}_{\Hk}\le\delta\).
	More specifically, let \(\mathcal U(X_i)\subseteq\mathcal X\) denote the set of admissible
	perturbations around \(X_i\), and suppose that
	\[
	\sup_{x'\in\mathcal U(X_i)}
	\norm{k_{x'}-k_{X_i}}_{\Hk}
	\le\delta,
	\qquad i=1,\ldots,n.
	\]
	Then, for every \(f\in\Hk\),
	\[
	\begin{aligned}
		&\sup_{\substack{
				X_i'\in\mathcal U(X_i),\\
				i=1,\ldots,n
		}}
		\frac{1}{n}\sum_{i=1}^n
		\left(
		Y_i-f(X_i')
		\right)^2 \le
		\frac{1}{n}\sum_{i=1}^n
		\left(
		|Y_i-f(X_i)|
		+\delta\norm{f}_{\Hk}
		\right)^2.
	\end{aligned}
	\]
	The feature-space formulation \eqref{eq:rkhs-at-minimax} for adversarial training in an RKHS therefore provides a tractable
	relaxation of input-space adversarial training and upper-bounds the
	corresponding worst-case perturbed loss. The tractability is discussed in the
	next subsection.

	\subsection{Finite-dimensional tractable formulation}
	The objective in problem~\eqref{problem} is convex and weakly lower
	semicontinuous on \(\Hk\). It is also coercive because it is bounded below
	by \(\delta^2\norm{f}_{\Hk}^2\). Hence, a minimizer exists; we denote one
	such minimizer by \(\widehat f_\delta\). Although problem~\eqref{problem} is posed over a potentially
	infinite-dimensional space,  the associated minimizers have finite-dimensional
	representations.
	
	\begin{proposition}[Tractable formulation]
		\label{prop:representer}
		Every minimizer of problem~\eqref{problem} belongs to $\operatorname{span}\{k_{X_1},\ldots,k_{X_n}\}. $
		In particular, there exist \(a_1,\ldots,a_n\in\R\) such that
		\[
		\widehat f_\delta
		=
		\sum_{i=1}^n a_i k_{X_i}.
		\]
	\end{proposition}
	
	Proposition~\ref{prop:representer} follows directly from the generalized
	representer theorem \citep{scholkopf2001generalized}.  The objective depends
	on \(f\) only through \(f(X_1),\ldots,f(X_n)\) and
	\(\norm{f}_{\Hk}\), and is strictly increasing in the norm when the fitted
	values are held fixed.  Therefore, the estimator can be computed through a
	finite-dimensional convex optimization problem in the coefficients
	\(a_1,\ldots,a_n\).
	\section{Prediction Error Rate}
	\label{sec:generalization}
	We analyze the random-design prediction error of the adversarial training
	estimator in problem~\eqref{problem}.  Following the integral-operator approach
	to kernel learning \citep{smale2007learning,wu2006learning,caponnetto2007optimal},
	we separate the radius-dependent approximation error from the sampling error.
	This decomposition describes prediction at a fixed radius and yields learning
	rates when the radius decreases with the sample size.

	Define the population counterpart of problem~\eqref{problem} by
	\begin{equation}\label{populationproblem}
		f_\delta^*
		\in
		\argmin_{f\in\Hk}
		\E\left[
		\left(
		|Y-f(X)|+\delta\norm{f}_{\Hk}
		\right)^2
		\right].
	\end{equation}
	Under Assumption~\ref{ass:model}, this objective is finite at \(f^*\), weakly
	lower semicontinuous, and coercive on \(\Hk\); hence a population minimizer
	exists.
	The prediction error can then be decomposed into an estimation error and
	an approximation error:
	\[
	\underbrace{
		\norm{\widehat f_\delta-f^*}_{L^2(\mu)}
	}_{\text{prediction error}}
	\le
	\underbrace{
		\norm{\widehat f_\delta-f_\delta^*}_{L^2(\mu)}
	}_{\text{estimation error}}
	+
	\underbrace{
		\norm{f_\delta^*-f^*}_{L^2(\mu)}
	}_{\text{approximation error}}.
	\]

	We now use the source condition to characterize how adversarial robustness
	interacts with the spectrum of the kernel integral operator \(T_k\).
	
	\begin{theorem}[Approximation error]
		\label{thm:full-approximation}
		Suppose Assumptions~\ref{ass:model}, \ref{ass:noise}, and
		\ref{ass:source} hold.  If
		\[
		0<\delta\le \min\left\{1,\frac{1}{8\kappa R}\right\},
		\]
		then every minimizer \(f_\delta^*\) of
		problem~\eqref{populationproblem} satisfies
		\[
		\norm{f_\delta^*-f^*}_{L^2(\mu)}
		\lesssim
		\delta^{\frac{\alpha+1/2}{\alpha+1}},
		\]
		where the implicit constant depends only on
		\(\alpha,R,\kappa,\E|\varepsilon|\).
	\end{theorem}
	\begin{remark}[Why the noise premium slows approximation?]
		The resulting approximation rate can be seen from one eigenfunction of
		\(T_k\) with eigenvalue \(\sigma_j\). If \(g\) has unit RKHS norm in this
		direction, the corresponding part of \(f^*=T_k^\alpha g\) has
		\(L^2(\mu)\) norm \(\sigma_j^{\alpha+1/2}\) and RKHS norm
		\(\sigma_j^\alpha\). Keeping this part incurs a noise-premium cost of order
		\(\delta\sigma_j^\alpha\), while removing it incurs squared prediction loss
		of order \(\sigma_j^{2\alpha+1}\). The two costs balance when
		\(\sigma_j\asymp\delta^{1/(\alpha+1)}\). At this eigenvalue, the
		\(L^2(\mu)\) error is
		\(\delta^{(\alpha+1/2)/(\alpha+1)}\), which explains the exponent in
		Theorem~\ref{thm:full-approximation}.
	\end{remark}
	By comparison, the \(L^2(\mu)\) approximation error of kernel ridge
	regression with regularization level \(\delta^2\) is of order
	\(\delta^{\min\{2\alpha+1,\,2\}}\)
	\citep{caponnetto2007optimal}.  Thus, the first-order noise premium, rather
	than the quadratic RKHS penalty, causes the slower approximation order.

	On the other hand, the estimation error is governed by the effective dimension
	\[
	\mathcal N(\lambda)
	:=
	\operatorname{Tr}((T_k+\lambda I)^{-1}T_k),
	\qquad \lambda>0.
	\]
	This quantity measures the effective number of kernel directions that must be
	estimated from the data and controls the leading stochastic term in the
	estimation error.
	
	The following theorem gives the estimation error of the adversarial training
	estimator in problem~\eqref{problem}.
	\begin{theorem}[Estimation error]\label{thm:estimation-error}
		Under Assumptions~\ref{ass:model}, \ref{ass:noise}, \ref{ass:noise-tail}, and
		\ref{ass:source}, there are constants \(c,C>0\), depending only on
		\(R,\sigma,\kappa\), such that the following holds.  Let \(0<\eta<1\) and
		suppose that
		\[
		0<\delta\le c,
		\qquad
		n\delta
		\ge
		C\left(\mathcal N(\delta)+1\right)\log\!\left(\frac{2n}{\eta}\right),
		\qquad
		n
		\ge
		C\log^2\!\left(\frac{2}{\eta}\right).
		\]
		Then, with probability at least
		\(1-\eta\), every empirical minimizer \(\widehat f_\delta\) of
		problem~\eqref{problem} satisfies
		\[
		\norm{\widehat f_\delta-f_\delta^*}_{L^2(\mu)}
		\lesssim
		\sqrt{
			\frac{
				(\mathcal N(\delta)+1)
				\log^3(2n/\eta)
			}{n}
		},
		\]
		where the implicit constant depends only on
		\(\alpha,R,\sigma,\kappa,\E|\varepsilon|\).
	\end{theorem}

	Up to logarithmic factors, Theorem~\ref{thm:estimation-error} gives estimation
	error \(\sqrt{(\mathcal N(\delta)+1)/n} \) relative to \(f_\delta^*\).
	The mixed term controls the RKHS norm at the \(\delta\) scale, which explains
	why the effective dimension is evaluated at \(\delta\).  The population shift
	from \(f^*\) to \(f_\delta^*\), including the noise premium, is bounded
	separately in Theorem~\ref{thm:full-approximation}.

	The prediction error is bounded by the sum of the approximation and
	estimation errors.  Combining
	Theorems~\ref{thm:full-approximation} and \ref{thm:estimation-error} gives the
	following generalization bound.
	
	\begin{corollary}[Prediction error]
		\label{cor:generalization-error}
		Under the assumptions of Theorem~\ref{thm:estimation-error}, suppose that
		\(\delta\) satisfies the conditions of
		Theorems~\ref{thm:full-approximation} and \ref{thm:estimation-error}.
		Then, with probability at least \(1-\eta\), every empirical minimizer
		\(\widehat f_\delta\) of problem~\eqref{problem} satisfies
		\[
		\norm{\widehat f_\delta-f^*}_{L^2(\mu)}
		\lesssim
		\delta^{\frac{\alpha+1/2}{\alpha+1}}
		+
		\sqrt{
			\frac{
				(\mathcal N(\delta)+1)
				\log^3(2n/\eta)
			}{n}
		},
		\]
		where the implicit constant depends only on \(\alpha,R,\sigma,\kappa,\E|\varepsilon|\).
	\end{corollary}
	
	Under Assumption~\ref{ass:eigendecay},
	\(\mathcal N(\delta)\lesssim\delta^{-1/\beta}\), and, up to logarithmic
	factors, the corollary becomes
	\[
	\norm{\widehat f_\delta-f^*}_{L^2(\mu)}
	\lesssim
	\delta^{\frac{\alpha+1/2}{\alpha+1}}
	+
	\frac{\delta^{-1/(2\beta)}}{\sqrt n}.
	\]
	Balancing the approximation and effective-dimension terms gives
	the radius
	\[
	\delta_n
	\asymp
	n^{-\frac{\beta(\alpha+1)}
		{\beta(2\alpha+1)+\alpha+1}},
	\]
	and the corresponding leading order is
	\[
	\norm{\widehat f_{\delta_n}-f^*}_{L^2(\mu)}
	\lesssim
	n^{-\frac{\beta(2\alpha+1)}
		{2(\beta(2\alpha+1)+\alpha+1)}},
	\]
	up to logarithmic factors.  
	
	For fixed \(0<\alpha\le1/2\), \(\beta>1\), and \(R>0\), the minimax
	\(L^2\) prediction rate under Assumptions~\ref{ass:model}--\ref{ass:eigendecay}
	is \(n^{-\beta(\alpha+1/2)/(2\beta(\alpha+1/2)+1)}\). Kernel ridge
	regression (KRR) attains this rate
	\citep{caponnetto2007optimal,mendelson2010regularization,fischer2020sobolev}.
	The exponent obtained for adversarial training
	at the balancing radius is strictly smaller.  The following theorem establishes
	matching sharpness on a fixed RKHS model.
	
	\begin{theorem}[Sharp lower bounds for adversarial training]
		\label{thm:full-at-matching-lower}
		Fix \(0<\alpha\le1\), \(\beta>1\), and \(R>0\).  There exists a fixed
		bounded-kernel model with \(\sigma_j\asymp j^{-\beta}\) and fixed bounded,
		centered noise, independent of \(X\), satisfying Assumptions~\ref{ass:noise} and
		\ref{ass:noise-tail}, for which
		\[
		\sup_{\norm{g}_{\Hk}\le R}
		\norm{f_\delta^*-T_k^\alpha g}_{L^2(\mu)}^2
		\gtrsim
		\delta^{\frac{2\alpha+1}{\alpha+1}}
		\]
		for all sufficiently small \(\delta>0\), where \(f_\delta^*\) is the
		population minimizer under \(f^*=T_k^\alpha g\).  Moreover, for any
		deterministic sequence satisfying
		\[
		\delta_n
		\asymp
		n^{-\frac{\beta(\alpha+1)}
			{\beta(2\alpha+1)+\alpha+1}},
		\]
		every empirical minimizer of problem~\eqref{problem} satisfies, for all
		sufficiently large \(n\),
		\[
		\sup_{\norm{g}_{\Hk}\le R}
		\E\norm{\widehat f_{\delta_n}-T_k^\alpha g}_{L^2(\mu)}^2
		\gtrsim
		n^{-\frac{\beta(2\alpha+1)}
			{\beta(2\alpha+1)+\alpha+1}}.
		\]
		For each \(g\) in the supremum, the expectation is taken under
		\(Y=(T_k^\alpha g)(X)+\varepsilon\).
	\end{theorem}
	
	Theorem~\ref{thm:full-at-matching-lower} matches the squared polynomial order
	of the upper bound in Corollary~\ref{cor:generalization-error}, up to
	logarithmic factors.  The
	population statement shows that the slower approximation scale induced by
	the noise premium is sharp, while the empirical statement establishes
	sharpness of the resulting prediction rate at the balancing radius.
	Consequently, for \(0<\alpha\le1/2\), adversarial training with the
	balancing sequence is minimax-rate suboptimal.  Section~\ref{sec:noise-debiased-at}
	shows that removing the noise premium restores the minimax rate.
	
	\section{Noise-Debiased Adversarial Training}
	\label{sec:noise-debiased-at}
	Section~\ref{sec:generalization} shows that the noise premium slows the approximation rate.
	We propose noise-debiased adversarial training, which subtracts this noise-dependent component
	while preserving the remaining part of the mixed robustness term. At the
	population level, the relevant decomposition is
	\[
	2\delta\norm{f}_{\Hk}\E|Y-f(X)|
	=
	2\delta\E|\varepsilon|\norm{f}_{\Hk}
	+
	2\delta\norm{f}_{\Hk}
	\left(\E|Y-f(X)|-\E|\varepsilon|\right).
	\]
	The first term is the noise premium, while the second term vanishes at
	\(f=f^*\). Removing the noise premium gives the population problem
	\begin{equation}\label{prob:populationndat}
		\widetilde f_\delta^*\in\argmin_{f\in\Hk}
		\left\{
		\norm{f-f^*}_{L^2(\mu)}^2
		+\delta^2\norm{f}_{\Hk}^2
		+2\delta\norm{f}_{\Hk}
		\left(\E|Y-f(X)|-\E|\varepsilon|\right)
		\right\}.
	\end{equation}
	The objective in \eqref{prob:populationndat} is weakly lower semicontinuous and coercive, so a
	minimizer exists.
		\subsection{Approximation error}

	We first quantify the effect of noise-debiasing at the population
	level.
	\begin{theorem}[Approximation error of noise-debiased adversarial training]
		\label{thm:ndat-approx}
		Suppose Assumptions~\ref{ass:model}, \ref{ass:noise}, and
		\ref{ass:source} hold.  If
		\[
		0<\delta\le \min\left\{1,\frac{1}{8\kappa R}\right\},
		\]
		then every minimizer \(\widetilde f_\delta^*\) of
		problem~\eqref{prob:populationndat} satisfies
		\[
		\norm{\widetilde f_\delta^*-f^*}_{L^2(\mu)}
		\lesssim
		\delta^{\min\{2\alpha+1,2\}},
		\]
		where the implicit constant depends only on \(R\).
	\end{theorem}
	
	Removing the noise premium improves the approximation rate from
	\(\delta^{(\alpha+1/2)/(\alpha+1)}\) to the ridge order
	\(\delta^{\min\{2\alpha+1,2\}}\). 
	\subsection{Two-stage estimator and prediction error}
	We describe the definition of the proposed noise-debiased adversarial training in this section.

	We estimate \(\E|\varepsilon|\)
	from independent data and derive the prediction error of the resulting
	two-stage estimator.
	The quantity \(\E|\varepsilon|\) is unknown and must be estimated using data
	independent of the sample used to construct the final estimator. On a sample
	\((X_j^a,Y_j^a)_{j=1}^{N_a}\), compute
	\begin{equation}\label{eq:krr}
		\widehat f^{\rm KRR}
		\in
		\argmin_{f\in\Hk}
		\left\{
		\frac1{N_a}\sum_{j=1}^{N_a}(Y_j^a-f(X_j^a))^2
		+\lambda_a\norm{f}_{\Hk}^2
		\right\},
	\end{equation}
	and use an independent sample \((X_i^e,Y_i^e)_{i=1}^{N_e}\) to set
	\[
	\widetilde m
	=
	\frac1{N_e}\sum_{i=1}^{N_e}
	|Y_i^e-\widehat f^{\rm KRR}(X_i^e)|.
	\]
	Given an independent fitting sample \((X_i,Y_i)_{i=1}^n\), let
	\(S_n=\operatorname{span}\{k_{X_1},\ldots,k_{X_n}\}\) and define
	\begin{equation}
		\label{eq:ndat-estimator}
		(\widetilde f_\delta,\widetilde r_\delta)
		\in
		\argmin_{\substack{f\in S_n,\ r\ge\norm{f}_{\Hk}}}
		\left\{
		\frac1n\sum_{i=1}^n
		\left(|Y_i-f(X_i)|+\delta r\right)^2
		-2\delta\widetilde m r
		\right\}.
	\end{equation}

	The auxiliary variable \(r\) makes the program convex, and
	\(\widetilde f_\delta\) is the reported predictor.  At the population level,
	the minimizing radius equals \(\norm{\widetilde f_\delta^*}_{\Hk}\), so the
	function component agrees with problem~\eqref{prob:populationndat}.
	Theorem~\ref{thm:ndat-generalization} analyzes the empirical program directly and
	does not require the radius constraint to be active.

	\begin{theorem}[Estimation and prediction errors for noise-debiased adversarial training]
		\label{thm:ndat-generalization}
		Under Assumptions~\ref{ass:model}, \ref{ass:noise},
		\ref{ass:noise-tail}, and \ref{ass:source}, there are constants \(c,C>0\), depending only on
		\(R,\sigma,\kappa\), such that the following holds.  Let
		\(0<\eta<1\), suppose that \(\widetilde m\ge0\) is computed independently
		of the fitting sample, and assume
		\[
		0<\delta\le c,
		\qquad
		\delta^2\le\norm{T_k}_{\mathrm{op}},
		\qquad
		n\delta^2
		\ge
		C\log\!\left(\frac{2}{\eta\delta^2}\right).
		\]
		Then, with probability at least \(1-\eta\), every minimizer
		\((\widetilde f_\delta,\widetilde r_\delta)\) of
		problem~\eqref{eq:ndat-estimator} satisfies
		\[
		\begin{aligned}
			\norm{\widetilde f_\delta-\widetilde f_\delta^*}_{L^2(\mu)}
			&\lesssim
			\sqrt{
				\frac{\mathcal N(\delta^2)\log(2/(\eta\delta^2))}{n}
			}
			+
			\left|\widetilde m-\E|\varepsilon|\right|,
			\\[2pt]
			\norm{\widetilde f_\delta-f^*}_{L^2(\mu)}
			&\lesssim
			\delta^{\min\{2\alpha+1,2\}}+
			\sqrt{
				\frac{\mathcal N(\delta^2)\log(2/(\eta\delta^2))}{n}
			}
			+
			\left|\widetilde m-\E|\varepsilon|\right|,
		\end{aligned}
		\]
		where the implicit constants depend only on \(R,\sigma,\kappa\).
	\end{theorem}
	Theorem~\ref{thm:ndat-generalization} separates the estimation error from
	the approximation error established in Theorem~\ref{thm:ndat-approx}. To
	obtain a sample-size dependent rate, it remains to control the effective
	dimension term and the error in \(\widetilde m\). Assumption~\ref{ass:noise}
	ensures that the bias in \(\widetilde m\) caused by the preliminary KRR
	estimator is bounded by its squared prediction error. The preliminary
	estimator therefore need not attain the prediction rate of the proposed
	estimator. Combining these observations with the polynomial eigenvalue decay
	gives the following result.
	
	\begin{corollary}[Two-stage prediction rate]
		\label{cor:ndat-two-stage-rate}
		Under Assumptions~\ref{ass:model}, \ref{ass:noise},
		\ref{ass:noise-tail}, \ref{ass:source}, and \ref{ass:eigendecay}, suppose
		\(0<\alpha\le1/2\).  Fix \(q>0\), let
		\(N_a\asymp N_e\asymp n\), and construct \(\widetilde m\) from the two
		independent samples above with
		\[
		\lambda_a
		=
		N_a^{-\frac{\beta}{2\beta(\alpha+1/2)+1}},
		\qquad
		\delta_n
		=
		n^{-\frac{\beta}{2(2\beta(\alpha+1/2)+1)}}.
		\]
		Then, for all sufficiently large \(n\), with probability at least
		\(1-3n^{-q}\), every minimizing pair in
		problem~\eqref{eq:ndat-estimator} satisfies
		\[
		\norm{\widetilde f_{\delta_n}-f^*}_{L^2(\mu)}
		\lesssim
		n^{-\frac{\beta(\alpha+1/2)}{2\beta(\alpha+1/2)+1}}
		\sqrt{\log n},
		\]
		where the implicit constant depends only on
		\(\alpha,\beta,R,\sigma,\kappa,c_\beta,C_\beta,q\) and the constants implicit in
		\(N_a\asymp N_e\asymp n\).
	\end{corollary}
	Corollary~\ref{cor:ndat-two-stage-rate} shows that the proposed
	noise-debiased estimator attains the minimax-optimal polynomial rate, up to
	the factor \(\sqrt{\log n}\), uniformly over the source ball for
	\(0<\alpha\le1/2\)
	\citep{caponnetto2007optimal,mendelson2010regularization,fischer2020sobolev}.
	This improves upon the slower rate obtained for adversarial training
	in Section~\ref{sec:generalization}.  The improvement results from removing
	the noise premium from the mixed robustness term, while the error in
	estimating \(\E|\varepsilon|\) is of smaller order and does not change the
	optimal exponent.
	
	\subsection{Effect of noise debiasing on adversarial loss}
	Removing the noise premium improves clean prediction but may increase adversarial loss. The following proposition bounds this increase at the same perturbation radius.
	
	\begin{proposition}[Increase in adversarial loss]
		\label{prop:ndat-robustness-cost}
		Let \(\widehat f_\delta\) be any minimizer of problem~\eqref{problem}, and let
		\((\widetilde f_\delta,\widetilde r_\delta)\) be any minimizing pair in
		problem~\eqref{eq:ndat-estimator}.  Then, for \(\delta>0\) and
		\(\widetilde m\ge0\),
		\[
		\begin{aligned}
			0
			&\le
			\frac1n\sum_{i=1}^n
			\bigg[
			\left(
			|Y_i-\widetilde f_\delta(X_i)|
			+\delta\norm{\widetilde f_\delta}_{\Hk}
			\right)^2
			\;-
			\left(
			|Y_i-\widehat f_\delta(X_i)|
			+\delta\norm{\widehat f_\delta}_{\Hk}
			\right)^2
			\bigg]
			\\
			&\le
			2\delta\widetilde m
			\left(
			\widetilde r_\delta
			-\norm{\widehat f_\delta}_{\Hk}
			\right).
		\end{aligned}
		\]
	\end{proposition}
	
	Thus, at the same perturbation radius, the clean-prediction improvement from
	debiasing is accompanied by an increase in adversarial loss
	that is bounded by the estimated correction.

	\subsection{Tractable computation}
	The estimator has the following finite-dimensional convex formulation.
	
	Let \(K\) be the Gram matrix with \(K_{ij}=k(X_i,X_j)\).  Substituting
	\(f=\sum_{j=1}^n a_jk_{X_j}\) into problem~\eqref{eq:ndat-estimator} gives
	\begin{equation}
		\label{eq:ndat-tractable}
		(\widetilde a_\delta,\widetilde r_\delta)
		\in\argmin_{\substack{a\in\R^n,\ r\ge0\\a^\top Ka\le r^2}}
		\left\{
		\frac1n\sum_{i=1}^n
		\left(|Y_i-(Ka)_i|+\delta r\right)^2
		-2\delta\widetilde m r
		\right\}.
	\end{equation}
	
	The following proposition establishes exact equivalence between
	problems~\eqref{eq:ndat-estimator} and \eqref{eq:ndat-tractable}.
	
	\begin{proposition}[Equivalent finite-dimensional formulation]
		\label{prop:ndat-convex-reformulation}
		Problems~\eqref{eq:ndat-estimator} and \eqref{eq:ndat-tractable} are convex,
		have minimizers, and have the same optimal value.  If
		\((\widetilde a_\delta,\widetilde r_\delta)\) minimizes
		problem~\eqref{eq:ndat-tractable}, then
		\[
		\widetilde f_\delta
		=
		\sum_{j=1}^n\widetilde a_{\delta,j}k_{X_j}
		\]
		and \((\widetilde f_\delta,\widetilde r_\delta)\) is a minimizing pair in
		problem~\eqref{eq:ndat-estimator}.  Conversely, if \((f,r)\) minimizes
		problem~\eqref{eq:ndat-estimator} and
		\(f=\sum_{j=1}^n a_jk_{X_j}\), then \((a,r)\) minimizes
		problem~\eqref{eq:ndat-tractable}.
	\end{proposition}

	\section{Signal Collapse}
	\label{sec:collapse}
	We study signal collapse in adversarial training when the robustness radius is fixed in this section.
	
	More data
	reduce estimation error but do not change the population minimizer. The
	following result gives the exact collapse thresholds for adversarial training
	and the proposed noise-debiased counterpart.
	
	\begin{corollary}[Exact collapse thresholds]
		\label{cor:breakdown-radius}
		Suppose Assumptions~\ref{ass:model} and \ref{ass:noise} hold and
		\(f^*\ne0\) in \(L^2(\mu)\).  Define
		\[
		\delta^{\rm AT}
		:=
		\frac{\norm{T_kf^*}_{\Hk}}{\E|Y|},
		\qquad
		\delta^{\rm ND}
		:=
		\frac{\norm{T_kf^*}_{\Hk}}{\E|Y|-\E|\varepsilon|},
		\]
		where \(\delta^{\rm ND}:=\infty\) if \(\E|Y|=\E|\varepsilon|\).
		The zero function minimizes the population adversarial training problem
		\eqref{populationproblem} if and only if \(\delta\ge\delta^{\rm AT}\), and it
		minimizes the population noise-debiased problem \eqref{prob:populationndat} if
		and only if \(\delta\ge\delta^{\rm ND}\).  In either case, the zero function
		is the unique minimizer when the corresponding inequality is strict.  Moreover,
		\[
		\delta^{\rm ND}\ge\delta^{\rm AT},
		\qquad
		\frac{\delta^{\rm ND}}{\delta^{\rm AT}}
		=
		\frac{\E|Y|}{\E|Y|-\E|\varepsilon|}
		\ \ge\
		\frac{\E|Y|}{\kappa\norm{f^*}_{L^2(\mu)}^2},
		\]
		with both sides of the equality interpreted as \(\infty\) when
		\(\E|Y|=\E|\varepsilon|\).
	\end{corollary}
	
	The thresholds differ by the noise premium. Whenever
	\(\delta^{\rm AT}<\delta^{\rm ND}\), adversarial training has the unique zero
	solution for \(\delta^{\rm AT}<\delta<\delta^{\rm ND}\), while the
	noise-debiased population solution remains nonzero. By
	Assumption~\ref{ass:noise},
	we have \(\E|Y|-\E|\varepsilon|\le\kappa\norm{f^*}_{L^2(\mu)}^2\), so this
	separation can be substantial when the signal is weak relative to the noise.

	\section{Numerical Experiments}
	\label{sec:num}
	We report numerical experiments in this section.
	Adversarial training is computed from problem~\eqref{problem}, and
	noise-debiased adversarial training from the equivalent convex program in
	Proposition~\ref{prop:ndat-convex-reformulation}.
	
	\subsection{Synthetic experiments}
	\label{sec:synth-experiment}
	
	The synthetic experiments use the model
	with parameters \(\alpha=1/4\), \(\beta=2\), and the Rademacher kernel
	\(k(x,x')=\sum_{j\ge1}\sigma_jx_jx_j'\), where
	\(\sigma_j\propto j^{-2}\).  We normalize the eigenvalues so that
	\(k(x,x)=1\), set
	\(f^*=T_k^{1/4}g\), \(\norm{g}_{\Hk}=1\), and use standard Gaussian noise.
	Unless stated otherwise, all synthetic experiments truncate the spectral
	representation to the first \(512\) coordinates.

	\subsubsection{Approximation error, prediction error, and signal collapse}
	\label{sec:synth-path}
	We first vary \(\delta\) to examine prediction error, approximation error,
	and signal collapse. Each replication uses a total of \(800\) observations.
	Adversarial training uses all \(800\). Noise-debiased adversarial training
	uses \(240\) observations for preliminary KRR, \(120\) for estimating
	\(\E|\varepsilon|\), and \(440\) for the convex optimization. We
	evaluate eight values of \(\delta\) in \([0.1,0.5]\) over \(50\) replications
	and approximate
	population solutions with a fixed Monte Carlo sample of size \(50{,}000\).
	
	\begin{figure}[!htbp]
		\centering
		\includegraphics[width=\textwidth]{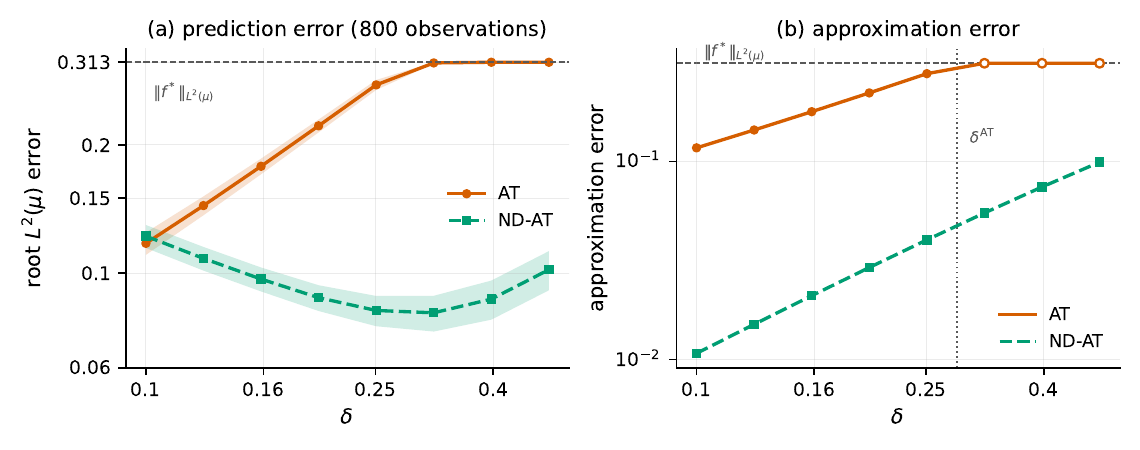}
		\caption{Synthetic radius path. Panel (a) uses \(800\) observations per
			replication. Adversarial training uses all \(800\), whereas noise-debiased
			adversarial training uses the \(240/120/440\) allocation described in the text.
			(a) Root \(L^2(\mu)\) prediction error, with bands of \(\pm2\) standard errors
			over \(50\) replications. (b) Approximation error; open markers
			denote zero population solutions for adversarial training. The horizontal
			dashed line marks the common zero-function error. In panel (b), the vertical dotted line marks the theoretical
			adversarial training collapse threshold \(\delta^{\rm AT}=0.2834\).}
		\label{fig:synthetic}
	\end{figure}
	
	Figure~\ref{fig:synthetic}(b) reports the approximation error.
	The adversarial training approximation error has log--log slope \(0.94\) and
	reaches the zero-function error once \(\delta\) exceeds the population
	collapse threshold \(\delta^{\rm AT}=0.2834\) in
	Corollary~\ref{cor:breakdown-radius}. The
	noise-debiased error has slope \(1.38\), close to the predicted exponent
	\(2\alpha+1=1.5\), and remains below the zero-function error throughout.
	
	Figure~\ref{fig:synthetic}(a) reports the prediction error.  
	Adversarial training is
	zero in \(10\%\) of replications at \(\delta=0.251\), \(88\%\) at
	\(\delta=0.316\), and all replications at the two largest values of \(\delta\).
	Thus, the empirical collapse transition occurs near this threshold. The
	noise-debiased estimator never collapses and reaches
	the minimum error, \(0.081\), at \(\delta=0.316\). Its prediction error is
	lower than that of adversarial training for every reported
	\(\delta\ge0.126\).
	
	\subsubsection{Signal strength and the collapse threshold}
	\label{sec:synth-snr}
	We next examine how the collapse threshold changes with signal strength.  We
	multiply the baseline regression function by
	\(\gamma\in\{0.1,0.2,0.5,1,2\}\) and set the noise standard deviation to
	\(\sigma_\varepsilon\in\{0.25,0.5,1,2\}\).  Each of the resulting \(20\)
	configurations uses \(n=800\) and is repeated \(20\) times.  To isolate the
	collapse mechanism, noise-debiased adversarial training uses the known
	Gaussian value \(\E|\varepsilon|=\sigma_\varepsilon\sqrt{2/\pi}\).
	
	We measure signal strength relative to noise by
	\begin{equation}
		\label{eq:snr-def}
		\mathrm{SNR}
		:=
		\frac{\norm{f^*}_{L^2(\mu)}}{\sigma_\varepsilon}
		=
		\frac{(\E f^*(X)^2)^{1/2}}
		{(\E\varepsilon^2)^{1/2}},
	\end{equation}
	where a larger SNR means a stronger regression signal.  The configurations
	cover \(11\) distinct SNR values between \(0.016\) and \(2.50\).  We report
	the relative prediction error
	\(\norm{\widehat f-f^*}_{L^2(\mu)}/\norm{f^*}_{L^2(\mu)}\), so the zero
	predictor has error one.
	
	\begin{figure}[!htbp]
		\centering
		\includegraphics[width=0.92\textwidth]{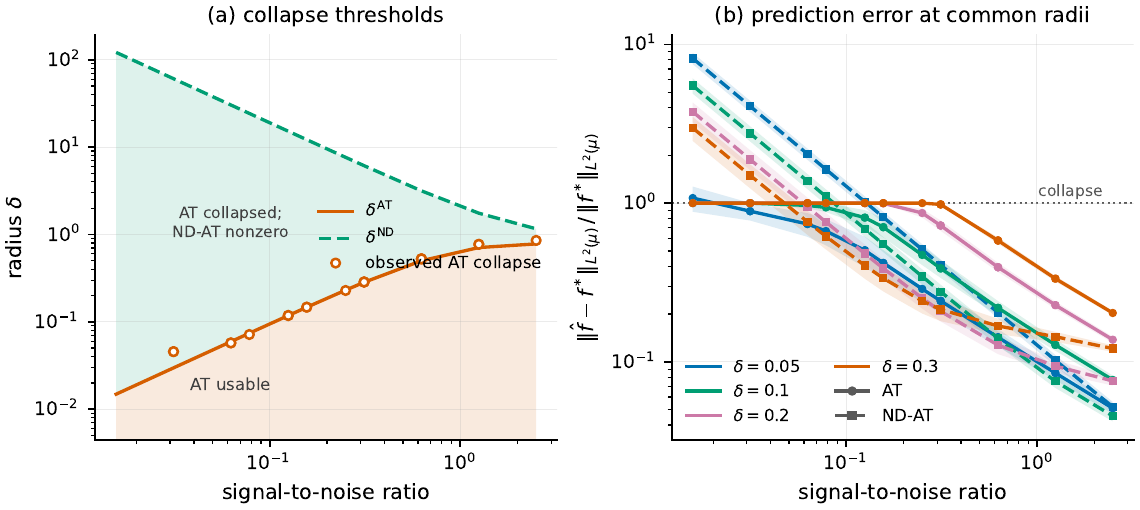}
		\caption{Effect of the signal-to-noise ratio defined in
			\eqref{eq:snr-def}. (a) Population collapse thresholds from
			Corollary~\ref{cor:breakdown-radius} and the empirical value of \(\delta\) at
			which half of the adversarial training estimates are zero. (b) Relative
			prediction error at four fixed values of \(\delta\), with bands of \(\pm2\)
			standard errors. The dotted line is the error of the zero predictor.}
		\label{fig:snr}
	\end{figure}
	
	Figure~\ref{fig:snr}(a) shows that the empirical collapse point follows the
	population threshold.  For nine of the eleven SNR values, the difference is
	less than \(10\%\).  As the SNR decreases, adversarial training collapses at
	a smaller value of \(\delta\).  At the three smallest SNR values, it returns
	the zero predictor in at least \(80\%\) of the replications when
	\(\delta\ge0.1\).
	
	Figure~\ref{fig:snr}(b) shows when noise debiasing improves prediction.  If
	the signal is extremely weak, the zero predictor is already competitive and
	debiasing need not help.  For stronger signals, debiasing postpones collapse
	and can substantially reduce prediction error.  At \(\mathrm{SNR}=0.625\)
	and \(\delta=0.2\), for example, the relative error decreases from \(0.395\)
	to \(0.128\).  The benefit is therefore largest when the signal is present
	but the noise contribution pushes adversarial training toward collapse.
	
In summary, the synthetic experiments support our theoretical findings.
	The population threshold accurately predicts adversarial
	training collapse, and noise debiasing delays collapse and improves prediction
	when the signal is nonnegligible.  
	
	Additional experiments in the Appendix examine the sharp approximation exponent and the effect of increasing
	the sample size at a fixed robustness radius.
	\subsection{Real-data evaluation}
	\label{sec:attack-experiments}
	This subsection evaluates adversarial training and noise-debiased adversarial
	training on real data.
	
	We consider Communities and Crime \citep{redmond2002data}, Abalone
	\citep{nash1994population}, and Diamonds \citep{wickham2016ggplot2}.  Their
	responses are violent crimes per population, abalone rings, and log diamond
	price.  All methods use the same \(20\) random splits, receive \(800\)
	observations per split, and are evaluated on an independent test sample of at
	most \(2{,}000\) observations.  For the two-stage noise-debiased procedure, the
	first stage uses one subsample to compute the preliminary KRR estimator and
	another to estimate noise.  The remaining observations are used
	for the final optimization.  Adversarial training and the two kernel ridge
	regression (KRR) methods compute their predictors from all \(800\) observations.
	KRR with \(\lambda=\delta^2\) removes the mixed term while preserving the same
	quadratic RKHS penalty as adversarial training.  KRR with a cross-validated
	penalty provides a benchmark for clean prediction.  To facilitate comparisons
	across data sets, we report RMSE in units of the response standard deviation.
	
	For the adversarial perturbation analysis, adversarial training, noise-debiased adversarial
	training, and KRR with \(\lambda=\delta^2\) are computed at \(\delta=0.1\),
	while KRR(CV) uses the cross-validated penalty.  The common training radius is
	below the estimated adversarial-training collapse boundary in every split, and
	no adversarial training estimator is zero at this radius.  Thus, the
	perturbation comparison is not driven by collapsed predictors.  We attack each predictor
	separately using projected gradient descent (PGD)
	\citep{madry2018towards}, which approximates the largest squared prediction
	error over \(\norm{\Delta x_{\rm cont}}_2\le\rho\) after the continuous
	covariates are standardized.   The perturbation applies to all covariates in
	Communities and Crime, the seven numerical covariates in Abalone, and the six
	physical covariates in Diamonds.  Categorical variables and responses are held
	fixed, and \(\rho\) ranges from \(0\) to \(0.5\).  Further implementation
	details are given in the Appendix.
	\begin{figure}[htbp]
		\centering
		\includegraphics[width=\textwidth]{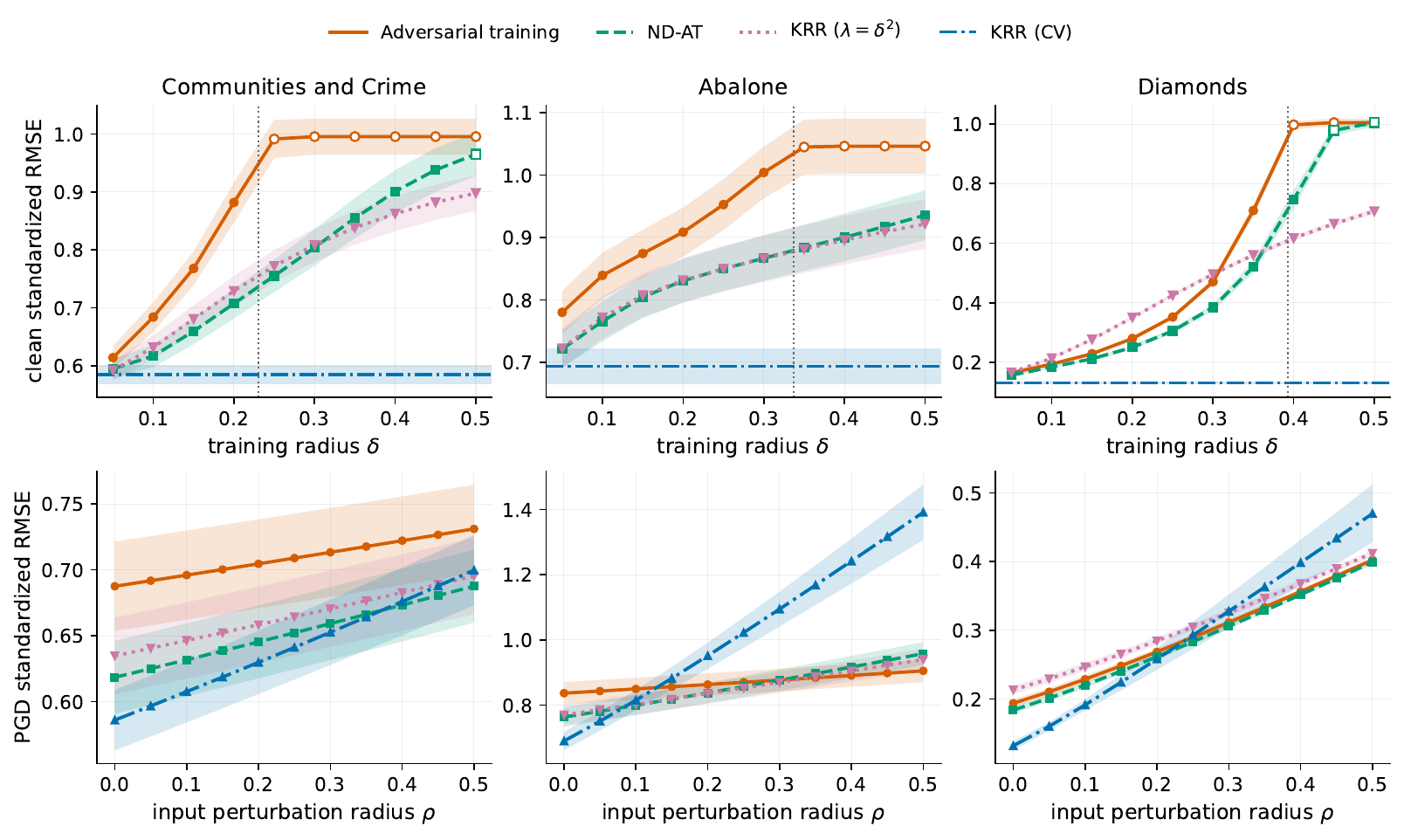}
		\caption{Real-data clean prediction and input-perturbation performance.  The
			upper panels report clean standardized RMSE against the training radius
			\(\delta\).  Vertical dotted lines are the mean estimated adversarial
			training collapse boundaries, and open markers indicate a zero estimator
			in at least one split.  The lower panels report PGD standardized RMSE
			against \(\rho\) for predictors computed at \(\delta=0.1\).  KRR(CV) uses
			the cross-validated penalty.  The same 400 test observations are attacked
			for all methods in each split.  Curves show means, and bands show \(\pm2\)
			standard errors over the same \(20\) splits.}
		\label{fig:attack}
	\end{figure}
	\FloatBarrier
	
	The upper panels show that noise debiasing lowers clean RMSE relative to
	adversarial training on all three data sets.  The estimated collapse boundaries
	closely track the observed transitions to the zero estimator for adversarial
	training, whereas noise-debiased adversarial training remains nonzero over a
	wider range of training radii.  These patterns support the proposed mechanism
	under which subtracting the estimated noise premium reduces premature
	shrinkage toward the zero predictor.  Cross-validated KRR attains
	the smallest clean RMSE on all three data sets, as expected for a benchmark
	tuned specifically for clean prediction.
	
	The lower panels distinguish total PGD RMSE from the attack-induced increase
	relative to \(\rho=0\).  Noise-debiased adversarial training
	starts with lower RMSE than adversarial training, although the RMSE of the
	noise-debiased method increases more as \(\rho\) grows.  This baseline advantage
	keeps the total PGD RMSE of noise-debiased adversarial training lower over the
	entire perturbation path on Communities and Crime as well as Diamonds.  On
	Abalone, noise-debiased adversarial training remains lower through
	\(\rho=0.25\), the two methods are nearly equal at \(\rho=0.3\), and adversarial
	training is lower thereafter.
	
	KRR with \(\lambda=\delta^2\) is the matched-penalty benchmark that omits the whole
	mixed term.  The attack-induced increase for KRR with \(\lambda=\delta^2\) is
	smaller than that of noise-debiased adversarial training on all three data
	sets, but noise-debiased adversarial training again starts with lower RMSE.
	Consequently,
	noise-debiased adversarial training has lower total PGD RMSE through
	\(\rho=0.15\) on all three data sets, with the two methods nearly
	indistinguishable on Abalone at the endpoint.  This reduction in total PGD RMSE
	persists over the full perturbation path on Communities and Crime as well as
	Diamonds.  On Abalone, KRR with \(\lambda=\delta^2\) is lower from \(\rho=0.2\)
	onward.  Cross-validated KRR has the largest attack-induced increase on all
	three data sets, showing that tuning for clean prediction alone does not ensure
	low prediction error under input perturbations.
	
	In conclusion, the real-data results show that noise debiasing mitigates the
	effects of the noise premium on clean prediction and signal collapse while
	also lowering prediction error under small input perturbations.  For
	\(\rho\le0.15\), noise-debiased adversarial training has the lowest mean total
	PGD RMSE among the three radius-dependent methods on all three data sets,
	although the noise-debiased and matched KRR means are nearly equal on Abalone
	at \(\rho=0.15\).  The matched-penalty comparison is consistent with a more
	favorable robustness--accuracy trade-off from the debiased mixed term under
	small attacks.  The steeper PGD increase and the crossover on Abalone show that
	the robustness cost depends on the data set and the perturbation radius.
	\section{Discussion}
	\label{sec:dis}
	\enlargethispage{2\baselineskip}
	Adversarial training in RKHS regression introduces a noise premium that creates
	persistent population bias.  Matching bounds show that the noise premium slows
	the optimally balanced prediction rate relative to the nonparametric minimax
	benchmark and can force collapse to the zero function at a fixed robustness
	radius.  Noise-debiased adversarial training subtracts an estimate of
	\(\E|\varepsilon|\), restores the minimax rate up to a logarithmic factor in the
	stated regime, raises the collapse threshold, and admits an explicit bound on
	the increase in adversarial loss.  The numerical results support these
	conclusions and show improved clean prediction and lower total error under
	small input perturbations, with a robustness cost that varies across data sets
	and attack radii.
	
	The present theory assumes squared-loss regression with independent additive
	noise, quadratic behavior of the absolute risk, Bernstein-type moment
	conditions.  One direction for future work
	is to weaken these noise conditions and allow covariate-dependent and
	heavier-tailed noise.  Another is to study how noise affects adversarial
	training in broader models and applications, especially neural networks with
	learned representations, and whether analogous noise-induced effects can be
	mitigated while controlling the associated robustness cost.

	\clearpage
	\renewcommand{\thesection}{S\arabic{section}}
\renewcommand{\thetheorem}{\thesection.\arabic{theorem}}
\renewcommand{\thefigure}{S\arabic{figure}}
\renewcommand{\theHsection}{S\arabic{section}}
\renewcommand{\theHsubsection}{S\arabic{section}.\arabic{subsection}}
\renewcommand{\theHfigure}{S\arabic{figure}}
\numberwithin{equation}{section}
	\setcounter{section}{0}
	\section{Proofs of Main Results}
	The proofs follow the order in which the results appear in the main text.
	The constant convention of the main text remains in force.  In particular,
	generic \(c,C>0\) may change from line to line in derivations, whereas
	constants stated in a theorem are fixed for that result.  Implicit constants in
	\(\lesssim\) and \(\gtrsim\) inherit the parameter dependence stated in the
	corresponding result in the main text and are independent of \(n,\delta,\eta\), and
	of the particular \(g\) satisfying \(\norm{g}_{\Hk}\le R\).
	\subsection{Proof of Theorem~\ref{thm:full-approximation}: \nameref*{thm:full-approximation}}
	
	We first provide some auxiliary lemmas.
	\begin{lemma}[Consequences of the quadratic absolute-risk condition]
		\label{lem:quadratic-absolute-risk}
		Let \(q(t)=\E|\varepsilon-t|\).  Under
		Assumption~\ref{ass:noise}, zero minimizes \(q\), every
		\(z\in\partial q(t)\) satisfies
		\[
		|z|\le4\kappa|t|,
		\]
		and
		\[
		\E|\varepsilon|\ge\frac{1}{4\kappa}.
		\]
	\end{lemma}
	
	\begin{proof}[Proof of Lemma \ref{lem:quadratic-absolute-risk}]
		
		Assumption~\ref{ass:noise} implies \(q(t)\ge q(0)\), so zero minimizes
		\(q\) and \(\partial q(0)=\{0\}\).
		
		Let \(t>0\) and \(z\in\partial q(t)\).  Convexity and the minimizing
		property of zero yield
		\[
		0\le\frac{q(t)-q(0)}{t}\le z.
		\]
		
		Applying the subgradient inequality at \(2t\) gives
		\[
		z
		\le
		\frac{q(2t)-q(t)}{t}
		\le
		\frac{q(2t)-q(0)}{t}
		\le
		4\kappa t.
		\]
		
		For \(t<0\), the same argument gives \(z\le0\) and
		\[
		zt
		\le
		q(2t)-q(t)
		\le
		q(2t)-q(0)
		\le
		4\kappa t^2.
		\]
		Division by \(t<0\) yields \(z\ge-4\kappa|t|\).
		
		Finally, for every \(a>0\), the triangle inequality gives
		\[
		q(a)+q(-a)
		=
		\E\bigl(|\varepsilon-a|+|\varepsilon+a|\bigr)
		\ge 2a.
		\]
		Assumption~\ref{ass:noise} also gives
		\(q(a)+q(-a)\le2q(0)+2\kappa a^2\).  Therefore,
		\(q(0)\ge a-\kappa a^2\).  Taking \(a=(2\kappa)^{-1}\) proves
		\(\E|\varepsilon|=q(0)\ge(4\kappa)^{-1}\).
	\end{proof}
	
	\begin{lemma}[Population norm control]
		\label{normcontrol}
		Under Assumptions~\ref{ass:model} and \ref{ass:noise}, for every \(\delta>0\)
		and every minimizer \(f_\delta^*\) of problem~\eqref{populationproblem},
		\[
		\norm{f_\delta^*}_{\Hk}
		\le
		\norm{f^*}_{\Hk}.
		\]
	\end{lemma}
	
	\begin{proof}[Proof of Lemma~\ref{normcontrol}]
		For every \(f\in\Hk\), independence and \(\E\varepsilon=0\) give
		\[
		\E[(Y-f(X))^2]
		=
		\norm{f-f^*}_{L^2(\mu)}^2+\E[\varepsilon^2].
		\]
		
		Optimality of \(f_\delta^*\), comparison with \(f^*\), and cancellation of
		\(\E[\varepsilon^2]\) give
		\begin{equation}
			\begin{aligned}\label{eq:opti}
				\norm{f_\delta^*-f^*}_{L^2(\mu)}^2
				+
				2\delta \norm{f_\delta^*}_{\Hk}\E|Y-f_\delta^*(X)|
				+
				\delta^2\norm{f_\delta^*}_{\Hk}^2
				\le
				2\delta\norm{f^*}_{\Hk}\E|\varepsilon|
				+
				\delta^2\norm{f^*}_{\Hk}^2 .
			\end{aligned}
		\end{equation}
		
		Dropping the last two terms on the left gives
		\begin{equation}
			\label{eq:population-objective-error}
			\norm{f_\delta^*-f^*}_{L^2(\mu)}^2
			\le
			\delta^2\norm{f^*}_{\Hk}^2
			+
			2\delta\E|\varepsilon|\norm{f^*}_{\Hk}.
		\end{equation}
		
		Lemma~\ref{lem:quadratic-absolute-risk}, applied conditionally on \(X\),
		implies
		\[
		\E|Y-f_\delta^*(X)|
		\ge
		\E|\varepsilon|.
		\]
		
		Using this lower bound in \eqref{eq:opti} and removing the
		nonnegative \(L^2(\mu)\) term yield
		\[
		\delta^2\norm{f_\delta^*}_{\Hk}^2
		+
		2\delta\E|\varepsilon|\norm{f_\delta^*}_{\Hk}
		\le
		\delta^2\norm{f^*}_{\Hk}^2
		+
		2\delta\E|\varepsilon|\norm{f^*}_{\Hk}.
		\]

		The map \(r\mapsto\delta^2r^2+2\delta\E|\varepsilon|r\) is strictly
		increasing on \([0,\infty)\), proving the result.
	\end{proof}
	
	\begin{lemma}[Population first-order condition]
		\label{lem:population-foc}
		Suppose Assumptions~\ref{ass:model} and \ref{ass:noise} hold. For every
		population minimizer \(f_\delta^*\), there exist a vector
		\[
		v_\delta\in\partial\norm{f_\delta^*}_{\Hk}
		\]
		and a measurable function \(s_\delta:\mathcal X\times\mathbb R\to[-1,1]\).
		The RKHS-norm subdifferential is
		\[
		\partial\norm{f_\delta^*}_{\Hk}
		=
		\begin{cases}
			\left\{f_\delta^*/\norm{f_\delta^*}_{\Hk}\right\},
			& f_\delta^*\ne0,\\[2mm]
			\left\{v\in\Hk:\norm{v}_{\Hk}\le1\right\},
			& f_\delta^*=0.
		\end{cases}
		\]
		The function \(s_\delta\) can be chosen so that
		\[
		s_\delta(x,e)
		\in
		\left.\partial_t|f^*(x)+e-t|\right|_{t=f_\delta^*(x)}
		\]
		and, together with \(v_\delta\), satisfies
		\begin{equation}
			\label{populationfirstorder}
			T_k(f_\delta^*-f^*)
			+
			\delta^2f_\delta^*
			+
			\delta\norm{f_\delta^*}_{\Hk}
			T_k\left(
			\E_\varepsilon
			\left[
			s_\delta(\cdot,\varepsilon)
			\right]
			\right)
			+
			\delta\E|Y-f_\delta^*(X)|
			v_\delta
			=0.
		\end{equation}
		The same selection satisfies
		\begin{equation}
			\label{eq:population-sign-selection}
			\left|
			\E_\varepsilon[s_\delta(x,\varepsilon)]
			\right|
			\le
			4\kappa|f_\delta^*(x)-f^*(x)|
			\quad\text{for almost every \(x\) with respect to the distribution of \(X\)}.
		\end{equation}
	\end{lemma}
	
	\begin{proof}[Proof of Lemma~\ref{lem:population-foc}]
		Suppose first that \(f_\delta^*\ne0\). The subgradient rule for convex
		integral functionals
		\citep[see, e.g.,][]{rockafellar1998variational}
		and the chain rule give a measurable selection \(s_\delta\) for which a
		subgradient of the integrand is
		\[
		2\left(
		|Y-f_\delta^*(X)|+\delta\norm{f_\delta^*}_{\Hk}
		\right)
		\left(
		s_\delta(X,\varepsilon)k_X
		+\delta\frac{f_\delta^*}{\norm{f_\delta^*}_{\Hk}}
		\right).
		\]
		It is integrable because \(|s_\delta|\le1\), \(\norm{k_X}_{\Hk}\le1\),
		and the objective is finite. Fermat's rule gives
		\eqref{populationfirstorder} after division by two and use of
		\[
		|Y-f_\delta^*(X)|s_\delta(X,\varepsilon)
		=
		f_\delta^*(X)-Y
		\]
		for every admissible absolute-loss subgradient.
		
		Conditioning on \(X=x\) gives
		\[
		\E_\varepsilon[s_\delta(x,\varepsilon)]
		\in
		\partial q(f_\delta^*(x)-f^*(x)),
		\qquad
		q(t)=\E|\varepsilon-t|.
		\]
		Lemma~\ref{lem:quadratic-absolute-risk} now gives
		\eqref{eq:population-sign-selection}.
		
		If \(f_\delta^*=0\), the right directional derivative at zero in direction
		\(h\in\Hk\) of one half of the population objective is
		\[
		-\inner{T_kf^*}{h}_{\Hk}
		+\delta\E|Y|\norm{h}_{\Hk}.
		\]
		
		Optimality makes this quantity nonnegative for every \(h\in\Hk\). By
		RKHS duality,
		\[
		\norm{T_kf^*}_{\Hk}\le\delta\E|Y|.
		\]
		
		Hence there is \(v_\delta\in\Hk\) with
		\(\norm{v_\delta}_{\Hk}\le1\) such that
		\[
		T_kf^*=\delta\E|Y|v_\delta.
		\]
		
		The same measurable-selection result gives \(s_\delta\) with
		\[
		\E_\varepsilon[s_\delta(x,\varepsilon)]
		\in
		\partial q(-f^*(x)).
		\]
		
		Lemma~\ref{lem:quadratic-absolute-risk} gives
		\eqref{eq:population-sign-selection}. Since
		\(\norm{f_\delta^*}_{\Hk}=0\), the term involving \(s_\delta\) in
		\eqref{populationfirstorder} is zero, and
		\(T_kf^*=\delta\E|Y|v_\delta\) verifies the remaining terms in that equation.
	\end{proof}
	
	\subsubsection{Main proof}
	\begin{proof}
		Lemma~\ref{normcontrol}, the source condition, and
		\(\norm{T_k}_{\mathrm{op}}\le1\) give
		\begin{equation}
			\label{eq:uniform-norm-bound}
			\norm{f_\delta^*}_{\Hk}
			\le
			\norm{f^*}_{\Hk}
			\le
			\norm{g}_{\Hk}
			\le R.
		\end{equation}
		
		Equation~\eqref{eq:population-objective-error} gives, for
		\(0<\delta\le1\),
		\begin{equation}
			\label{eq:uniform-basic-error}
			\norm{f_\delta^*-f^*}_{L^2(\mu)}
			\le
			\left(R^2+2R\E|\varepsilon|\right)^{1/2}\sqrt{\delta}.
		\end{equation}

		\textbf{We first dispose of the case \(f_\delta^*=0\).}  The right directional
		derivative of the population objective at zero in direction \(h\in\Hk\),
		divided by two, is
		\[
		-\inner{T_kf^*}{h}_{\Hk}
		+
		\delta\E|Y|\norm{h}_{\Hk}.
		\]
		
		Since zero is a minimizer, this derivative is nonnegative for both \(h\)
		and \(-h\).  Taking the supremum over
		\(\norm{h}_{\Hk}\le1\) gives
		\[
		\norm{T_kf^*}_{\Hk}
		\le
		\delta\E|Y|
		\le
		\delta\left(\E|\varepsilon|+\norm{f^*}_{L^2(\mu)}\right)
		\le
		\delta\left(\E|\varepsilon|+R\right).
		\]
		
		For a positive operator, Hölder's inequality in the eigenbasis gives
		\begin{equation}\label{eq:holder}
			\norm{T_k^{\alpha+1/2}u}_{\Hk}
			\le
			\norm{T_k^{\alpha+1}u}_{\Hk}^{
				\frac{\alpha+1/2}{\alpha+1}}
			\norm{u}_{\Hk}^{\frac{1}{2(\alpha+1)}},
			\qquad u\in\Hk.
		\end{equation}
		
		To verify it, write \(u=\sum_j u_j\varphi_j\) and apply Hölder's
		inequality to
		\[
		\sum_j
		\frac{\sigma_j^{2\alpha+1}u_j^2}{\sigma_j}
		=
		\sum_j
		\left(
		\frac{\sigma_j^{2\alpha+2}u_j^2}{\sigma_j}
		\right)^{\frac{\alpha+1/2}{\alpha+1}}
		\left(
		\frac{u_j^2}{\sigma_j}
		\right)^{\frac{1}{2(\alpha+1)}}.
		\]
		
		Applying \eqref{eq:holder} to \(u=g\) yields
		\[
		\begin{aligned}
			\norm{f^*}_{L^2(\mu)}
			&=
			\norm{T_k^{\alpha+1/2}g}_{\Hk}\\
			&\le
			\norm{T_kf^*}_{\Hk}^{\frac{\alpha+1/2}{\alpha+1}}
			\norm{g}_{\Hk}^{\frac{1}{2(\alpha+1)}}
			\lesssim
			\delta^{\frac{\alpha+1/2}{\alpha+1}}.
		\end{aligned}
		\]
		This proves the theorem when \(f_\delta^*=0\). 
		
		\textbf{We now consider the
			remaining case \(f_\delta^*\ne0\).}
		
		Lemma~\ref{lem:population-foc} then gives
		\[
		T_k(f_\delta^*-f^*)
		+\delta^2f_\delta^*
		+\delta\norm{f_\delta^*}_{\Hk}
		T_k\left(
		\E_\varepsilon
		\left[
		s_\delta(\cdot,\varepsilon)
		\right]
		\right)
		+\delta\E|Y-f_\delta^*(X)|
		\frac{f_\delta^*}{\norm{f_\delta^*}_{\Hk}}
		=0.
		\]
		
		Define
		\begin{equation}\label{eq:defoflambda}
			\lambda_\delta
			:=
			\delta^2
			+
			\delta\frac{\E|Y-f_\delta^*(X)|}{\norm{f_\delta^*}_{\Hk}}.
		\end{equation}
		
		Then, we have
		\[
		(T_k+\lambda_\delta I)f_\delta^*
		=
		T_kf^*
		-
		\delta\norm{f_\delta^*}_{\Hk}
		T_k\left(
		\E_\varepsilon
		\left[
		s_\delta(\cdot,\varepsilon)
		\right]
		\right).
		\]
		
		Since
		\(\lambda_\delta\ge\delta^2>0\), the operator
		\(T_k+\lambda_\delta I\) is invertible.  Applying the inverse of \(T_k+\lambda_\delta I\)
		and subtracting \(f^*\) gives
		\[
		f_\delta^*-f^*
		=
		-\lambda_\delta(T_k+\lambda_\delta I)^{-1}f^*
		-
		\delta\norm{f_\delta^*}_{\Hk}(T_k+\lambda_\delta I)^{-1}
		T_k\left(
		\E_\varepsilon
		\left[
		s_\delta(\cdot,\varepsilon)
		\right]
		\right)
		,
		\]
		where we used the resolvent identity
		\[
		(T_k+\lambda_\delta I)^{-1}T_kf^*-f^*
		=
		-\lambda_\delta(T_k+\lambda_\delta I)^{-1}f^*.
		\]

		By the definition of \(\lambda_\delta\) \eqref{eq:defoflambda},
		\[
		\begin{aligned}
			\lambda_\delta\norm{f_\delta^*}_{\Hk}
			&=
			\delta^2\norm{f_\delta^*}_{\Hk}
			+
			\delta\E|Y-f_\delta^*(X)|\\
			&\le
			\delta^2R
			+
			\delta\left(
			\E|\varepsilon|
			+
			\norm{f_\delta^*-f^*}_{L^2(\mu)}
			\right)
			\lesssim
			\delta,
		\end{aligned}
		\]
		where the last inequality follows from
		\eqref{eq:uniform-basic-error} and \(0<\delta\le1\).  
		
		Let
		\[
		f_{\lambda_\delta}
		=
		(T_k+\lambda_\delta I)^{-1}T_kf^*.
		\]
		
		Then, we have
		\[
		\begin{aligned}
			f_{\lambda_\delta}
			&=
			f_\delta^*
			+
			\delta\norm{f_\delta^*}_{\Hk}
			(T_k+\lambda_\delta I)^{-1}T_k
			\left(
			\E_\varepsilon
			\left[
			s_\delta(\cdot,\varepsilon)
			\right]
			\right).
		\end{aligned}
		\]
		
		For every \(h\in L^2(\mu)\), spectral calculus gives
		\[
		\begin{aligned}
			&\norm{
				\lambda_\delta(T_k+\lambda_\delta I)^{-1}T_kh
			}_{\Hk}^2\\
			&\qquad=
			\sum_j
			\frac{\lambda_\delta^2\sigma_j}
			{(\sigma_j+\lambda_\delta)^2}
			\inner{h}{\varphi_j}_{L^2(\mu)}^2
			\le
			\frac{\lambda_\delta}{4}\norm{h}_{L^2(\mu)}^2.
		\end{aligned}
		\]
		
		Equation~\eqref{eq:population-sign-selection} implies
		\[
		\left\|
		\E_\varepsilon
		\left[
		s_\delta(\cdot,\varepsilon)
		\right]
		\right\|_{L^2(\mu)}
		\le
		4\kappa\norm{f_\delta^*-f^*}_{L^2(\mu)}.
		\]
		
		Consequently,
		\[
		\begin{aligned}
			\lambda_\delta\norm{f_{\lambda_\delta}}_{\Hk}
			&\le
			\lambda_\delta\norm{f_\delta^*}_{\Hk}
			+
			2\kappa\delta\norm{f_\delta^*}_{\Hk}
			\sqrt{\lambda_\delta}
			\norm{f_\delta^*-f^*}_{L^2(\mu)}.
		\end{aligned}
		\]

		Moreover,
		\[
		\norm{f_\delta^*}_{\Hk}\sqrt{\lambda_\delta}
		=
		\left(
		\norm{f_\delta^*}_{\Hk}
		\lambda_\delta\norm{f_\delta^*}_{\Hk}
		\right)^{1/2}
		\lesssim
		\sqrt{\delta}
		\]
		by \eqref{eq:uniform-norm-bound} and the bound
		\(\lambda_\delta\norm{f_\delta^*}_{\Hk}\lesssim\delta\) established above.
		Combining this bound with \eqref{eq:uniform-basic-error} yields
		\begin{equation}
			\label{eq:uniform-ridge-size}
			\lambda_\delta\norm{f_{\lambda_\delta}}_{\Hk}
			\lesssim
			\delta.
		\end{equation}
		
		The element
		\(\lambda_\delta(T_k+\lambda_\delta I)^{-1}g\) has RKHS norm at most
		\(R\), and
		\[
		f^*-f_{\lambda_\delta}
		=
		T_k^\alpha
		\lambda_\delta(T_k+\lambda_\delta I)^{-1}g,
		\qquad
		T_k^{\alpha+1}
		\lambda_\delta(T_k+\lambda_\delta I)^{-1}g
		=
		\lambda_\delta f_{\lambda_\delta}.
		\]
		
		Applying \eqref{eq:holder} and
		\eqref{eq:uniform-ridge-size} gives
		\begin{equation}
			\label{eq:uniform-ridge-bias}
			\norm{f^*-f_{\lambda_\delta}}_{L^2(\mu)}
			\lesssim
			\delta^{\frac{\alpha+1/2}{\alpha+1}}.
		\end{equation}
		
		The operator \((T_k+\lambda_\delta I)^{-1}T_k\) has spectral multiplier
		\(\sigma_j/(\sigma_j+\lambda_\delta)\), so its \(L^2(\mu)\)-operator norm
		is at most one.  Then, we have 
		\[
		\begin{aligned}
			\norm{f_\delta^*-f^*}_{L^2(\mu)}
			&\le
			\norm{f_{\lambda_\delta}-f^*}_{L^2(\mu)}
			+
			4\kappa\delta R
			\norm{f_\delta^*-f^*}_{L^2(\mu)}.
		\end{aligned}
		\]
		
		The condition \(8\kappa R\delta\le1\) makes the final coefficient at most
		\(1/2\).  Moving that term to the left proves
		\[
		\norm{f_\delta^*-f^*}_{L^2(\mu)}
		\lesssim
		\delta^{\frac{\alpha+1/2}{\alpha+1}},
		\]
		with a constant uniform over \(\norm{g}_{\Hk}\le R\).
	\end{proof}
	\subsection{Concentration tools for the estimation proofs}
	
	For the estimation proofs, recall the RKHS realizations of the population
	and empirical kernel operators:
	\[
	T_k f=\E[f(X)k_X],
	\qquad
	T_n f=\frac1n\sum_{i=1}^n f(X_i)k_{X_i},
	\qquad f\in\Hk.
	\]
	Thus, for every \(f,h\in\Hk\),
	\[
	\inner{T_kf}{h}_{\Hk}=\E[f(X)h(X)],
	\qquad
	\inner{T_nf}{h}_{\Hk}
	=\frac1n\sum_{i=1}^n f(X_i)h(X_i).
	\]
	We next provide the concentration results used in the estimation arguments.
	
	\begin{theorem}[Regularized kernel operator concentration \citep{rudi2015less}]
		\label{thm:regularized-covariance-embedding}
		For every \(0<\lambda\le\norm{T_k}_{\mathrm{op}}\) and \(0<\eta<1\), with probability
		at least \(1-\eta\),
		\[
		\norm{
			(T_k+\lambda I)^{-1/2}
			(T_n-T_k)
			(T_k+\lambda I)^{-1/2}
		}_{\mathrm{op}}
		\le
		2\left[
		\sqrt{\frac{\log(8/(\eta\lambda))}{n\lambda}}
		+
		\frac{\log(8/(\eta\lambda))}{n\lambda}
		\right].
		\]
	\end{theorem}
	
	\begin{proof}[Proof of Theorem~\ref{thm:regularized-covariance-embedding}]
		Proposition~8 in \citet{rudi2015less}, applied with
		\(Q=T_k\), \(Q_n=T_n\), \(v=k_X\), and confidence parameter \(\eta/2\),
		gives the claimed normalized operator bound with
		\(b=\log\{8\operatorname{Tr}(T_k)/(\eta\lambda)\}\) and
		\[
		\mathcal N_\infty(\lambda)
		=
		\sup_{x\in\mathcal X}
		\left\langle k_x,(T_k+\lambda I)^{-1}k_x\right\rangle_{\Hk}.
		\]
		The bounded-kernel assumption and
		\(\operatorname{Tr}(T_k)=\E k(X,X)\le1\) give
		\[
		\mathcal N_\infty(\lambda)
		\le\lambda^{-1},
		\qquad
		b\le\log\left(\frac8{\eta\lambda}\right).
		\]
		Since \(\lambda\le\norm{T_k}_{\mathrm{op}}\le1\), also
		\(1+\mathcal N_\infty(\lambda)\le2\lambda^{-1}\).
		Substitution into Proposition~8 and a loosening of the numerical constants
		give the displayed inequality.
	\end{proof}

	\begin{theorem}[Hilbert-valued Bernstein inequality, \citep{pinelis1986remarks,pinelis1994optimum}]
		\label{thm:hilbert-vector-bernstein}
		Let \(\mathcal H\) be a separable Hilbert space, and let
		\(\xi_1,\ldots,\xi_n\) be independent mean-zero \(\mathcal H\)-valued random
		variables.  Suppose that, for some \(v,L>0\),
		\[
		\sum_{i=1}^n
		\E\norm{\xi_i}_{\mathcal H}^{p}
		\le
		\frac{p!}{2}\,v\,L^{p-2},
		\qquad p=2,3,\ldots .
		\]
		Then, for every \(0<\eta<1\), with probability at least \(1-\eta\),
		\[
		\left\|
		\frac1n\sum_{i=1}^n \xi_i
		\right\|_{\mathcal H}
		\le
		\frac{\sqrt{2v\log(2/\eta)}}{n}
		+
		\frac{2L\log(2/\eta)}{n}.
		\]
	\end{theorem}
	This is the high-probability form of Corollary~1 of
	\citet{pinelis1986remarks}, obtained by setting
	\(u=\log(2/\eta)\) in its Bernstein tail bound.
	
	Assumption~\ref{ass:noise-tail} implies
	\[
	\E\exp\left(\frac{|\varepsilon|}{2\sigma}\right)
	\le
	1+\frac{\E|\varepsilon|}{2\sigma}
	+\sum_{p\ge2}\frac{\E|\varepsilon|^p}{p!\,(2\sigma)^p}
	\le \frac74<2.
	\]
	Consequently, Markov's inequality yields
	\begin{equation}
		\label{eq:noise-exponential-tail}
		\mathbb P(|\varepsilon|>u)\le 2e^{-u/(2\sigma)},
		\qquad u>0 .
	\end{equation}
	
	\begin{lemma}[Regularized noise concentration]
		\label{lem:regularized-noise-concentration}
		Suppose Assumptions~\ref{ass:model} and \ref{ass:noise-tail} hold.  For
		every \(\lambda>0\) and \(0<\eta<1\), with probability at least \(1-\eta\),
		\[
		\norm{
			(T_k+\lambda I)^{-1/2}
			\frac1n\sum_{i=1}^n\varepsilon_i k_{X_i}
		}_{\Hk}
		\le
		\sigma
		\sqrt{\frac{2\mathcal N(\lambda)\log(2/\eta)}{n}}
		+
		\frac{2\sigma\log(2/\eta)}{n\sqrt\lambda}.
		\]
	\end{lemma}
	
	\begin{proof}[Proof of Lemma \ref{lem:regularized-noise-concentration}]
		Put \(\xi_i=\varepsilon_i(T_k+\lambda I)^{-1/2}k_{X_i}\).  These are
		independent \(\Hk\)-valued random elements, and
		Assumption~\ref{ass:model} yields \(\E\xi_i=0\).
		
		Since \(T_k+\lambda I\succeq\lambda I\), the bounded-kernel assumption
		gives, for every \(x\in\mathcal X\), the pointwise bound already used in the
		proof of Theorem~\ref{thm:regularized-covariance-embedding} in the form
		\(\mathcal N_\infty(\lambda)\le\lambda^{-1}\),
		\begin{equation}
			\label{eq:regularized-feature-bound}
			\norm{(T_k+\lambda I)^{-1/2}k_x}_{\Hk}^2
			=
			\inner{k_x}{(T_k+\lambda I)^{-1}k_x}_{\Hk}
			\le
			\frac{k(x,x)}{\lambda}
			\le
			\frac1\lambda,
		\end{equation}
		while
		\begin{equation}
			\label{eq:regularized-feature-second-moment}
			\E\norm{(T_k+\lambda I)^{-1/2}k_X}_{\Hk}^2
			=
			\operatorname{Tr}\left\{(T_k+\lambda I)^{-1}T_k\right\}
			=
			\mathcal N(\lambda).
		\end{equation}
		
		Combining \eqref{eq:regularized-feature-bound} and
		\eqref{eq:regularized-feature-second-moment} gives, for every integer \(p\ge2\),
		\[
		\E\norm{(T_k+\lambda I)^{-1/2}k_X}_{\Hk}^p
		\le
		\lambda^{-(p-2)/2}\mathcal N(\lambda).
		\]
		
		Independence of \(\varepsilon\) and \(X\) and
		Assumption~\ref{ass:noise-tail} therefore yield
		\[
		\sum_{i=1}^n\E\norm{\xi_i}_{\Hk}^p
		=
		n\,\E|\varepsilon|^p\,
		\E\norm{(T_k+\lambda I)^{-1/2}k_X}_{\Hk}^p
		\le
		\frac{p!}{2}
		n\sigma^2\mathcal N(\lambda)
		\left(\frac{\sigma}{\sqrt\lambda}\right)^{p-2}.
		\]
		
		The moment condition in Theorem~\ref{thm:hilbert-vector-bernstein} holds with
		\(v=n\sigma^2\mathcal N(\lambda)\) and \(L=\sigma\lambda^{-1/2}\), and the
		theorem gives the stated bound.
	\end{proof}
	
	\subsection{Proof of Theorem~\ref{thm:estimation-error}: \nameref*{thm:estimation-error}}
	
\textbf{Proof strategy.}  The proof compares the empirical objective at
	\(\widehat f_\delta\) with the same objective at \(f_\delta^*\).
	\begin{enumerate}[label=(\roman*)]
		\item We first prove that every empirical minimizer has RKHS norm at most a
		fixed constant \(B\).  It is therefore enough to compare the two objectives
		on this bounded set.
		\item For \(h=f-f_\delta^*\), the operator \(A_\delta\) defined in \eqref{def:adelta} is chosen
		so that there is a constant \(c_{\mathrm{curv}}>0\), depending only on
		\(R,\sigma,\kappa\), for which the population objective at
		\(f_\delta^*+h\) exceeds its minimum at \(f_\delta^*\) by at least
		\[
		c_{\mathrm{curv}}\norm{A_\delta^{1/2}h}_{\Hk}^2.
		\]
		Thus, on the boundary \(\norm{A_\delta^{1/2}h}_{\Hk}=t\), the population
		objective increases by at least \(c_{\mathrm{curv}}t^2\).
		\item We next control the error caused by replacing the expectation with the
		sample average.  For a constant \(C_{\mathrm{dev}}>0\), depending only on
		\(R,\sigma,\kappa\), this error is uniformly at most
		\[
		C_{\mathrm{dev}}\left(
		t\sqrt{
			\frac{(\mathcal N(\delta)+1)\log^3(2n/\eta)}{n}}
		+
		\frac{\log^2(2n/\eta)}{n}
		\right).
		\]
		\item Take \(t\) to be a sufficiently large constant multiple of
		\(\sqrt{(\mathcal N(\delta)+1)\log^3(2n/\eta)/n}\).  The lower bound
		\(c_{\mathrm{curv}}t^2\) then exceeds the sampling error in (iii), so every boundary point
		has larger empirical objective than \(f_\delta^*\).  If
		\(\widehat f_\delta\) were outside the boundary, the line segment from
		\(f_\delta^*\) to \(\widehat f_\delta\) would contain a boundary point whose
		empirical objective is no larger than that at \(f_\delta^*\), by convexity
		and the optimality of \(\widehat f_\delta\).  This contradiction proves the
		bound.  Finally,
		\(\norm{h}_{L^2(\mu)}\le\norm{A_\delta^{1/2}h}_{\Hk}\), so the same bound
		holds in \(L^2(\mu)\).
	\end{enumerate}
	
	\begin{lemma}[RKHS-norm bound for empirical minimizers]
		\label{lem:full-empirical-norm}
		Suppose Assumptions~\ref{ass:model}, \ref{ass:noise},
		\ref{ass:noise-tail}, and \ref{ass:source} hold.  There are constants
		\(B\ge\max\{1,2R\}\) and \(c,C>0\), depending only on
		\(R,\sigma,\kappa\), such that the following holds for every
		\(0<\delta\le c\) and every \(0<\eta<1\).  If
		\begin{equation}
			\label{eq:new-norm-conditions}
			n\ge C\log^2\left(\frac2\eta\right)
			\qquad\text{and}\qquad
			n\delta
			\ge
			C\left(\mathcal N(\delta)+1\right)
			\log\left(\frac{2n}\eta\right),
		\end{equation}
		then, with probability at least \(1-\eta\), every empirical minimizer
		\(\widehat f_\delta\) of problem~\eqref{problem} satisfies
		\[
		\norm{\widehat f_\delta}_{\Hk}\le B.
		\]
	\end{lemma}
	
	\begin{proof}[Proof of Lemma \ref{lem:full-empirical-norm}]
		The proof adapts the basic comparison used for linear adversarial training
		in \citet[Lemma~A.1]{xie2024highdimensional} to the regularized RKHS norm.
		Choose \(B\ge\max\{1,2R\}\), set \(\lambda=\delta/B\), and take the upper
		bound \(c\) on \(\delta\) small enough that
		\[
		\delta B\le1,
		\qquad
		\delta R\le\frac1{16\kappa}.
		\]
		
		Minimize the empirical objective over
		\(\{f:\norm{f}_{\Hk}\le B\}\).  A minimizer exists by the representer
		theorem and compactness in the finite-dimensional sample span.  We show that
		no such minimizer lies on the boundary.
		
		We use scalar Bernstein, Lemma~\ref{lem:regularized-noise-concentration},
		and Theorem~\ref{thm:regularized-covariance-embedding}, each with failure
		probability \(\eta/3\), and take a union bound.  On the resulting event,
		scalar Bernstein and
		Lemma~\ref{lem:quadratic-absolute-risk} give
		\begin{equation}
			\label{eq:empirical-noise-absolute-mean}
			\frac1n\sum_{i=1}^n|\varepsilon_i|
			\ge
			\frac12\E|\varepsilon|
			\ge
			\frac1{8\kappa}.
		\end{equation}
		
		Termwise comparison in the spectrum gives
		\[
		\mathcal N(\delta/B)\le B\mathcal N(\delta).
		\]
		
		Consequently, Lemma~\ref{lem:regularized-noise-concentration} and
		\((a+b)^2\le2a^2+2b^2\) give
		\[
		\left\|
		(T_k+\lambda I)^{-1/2}
		\frac1n\sum_{i=1}^n\varepsilon_i k_{X_i}
		\right\|_{\Hk}^2
		\le
		B\left(
		\frac{4\sigma^2\mathcal N(\delta)\log(24/\eta)}{n}
		+
		\frac{8\sigma^2\log^2(24/\eta)}{n^2\delta}
		\right).
		\]
		
		Choose a fixed \(c_0>0\) such that
		\[
		4c_0+6\sqrt{2c_0}\le\frac{1}{16\kappa}.
		\]
		
		After increasing \(C\) in \eqref{eq:new-norm-conditions}, the right-hand side
		above is at most \(c_0\delta B\).  Hence
		\begin{equation}
			\label{eq:boundary-noise}
			\left\|
			(T_k+\lambda I)^{-1/2}
			\frac1n\sum_{i=1}^n\varepsilon_i k_{X_i}
			\right\|_{\Hk}^2
			\le c_0\delta B.
		\end{equation}
		
		It remains to compare the empirical and population \(L^2\) norms.  If
		\(\lambda\le\norm{T_k}_{\mathrm{op}}\), it is enough that
		\[
		n\delta
		\ge
		25B\log\left(\frac{96B}{\eta\delta}\right).
		\]
		
		Condition~\eqref{eq:new-norm-conditions} implies \(\delta\ge n^{-1}\), and
		hence
		\(\log(96B/(\eta\delta))\lesssim\log(2n/\eta)\).  Increasing \(C\) therefore
		ensures this lower bound on \(n\delta\).  Theorem~\ref{thm:regularized-covariance-embedding},
		applied with failure probability \(\eta/3\), then bounds the normalized
		operator by \(1/2\).  Rearranging its quadratic-form inequality gives,
		simultaneously for every \(g\in\Hk\),
		\begin{equation}
			\label{eq:empirical-l2-comparison}
			\norm{g}_{L^2(\mu)}^2
			\le
			\frac2n\sum_{i=1}^n g(X_i)^2
			+\lambda\norm{g}_{\Hk}^2.
		\end{equation}
		If \(\lambda>\norm{T_k}_{\mathrm{op}}\), the same inequality follows
		deterministically from
		\(\norm{g}_{L^2(\mu)}^2\le\lambda\norm{g}_{\Hk}^2\).
		
		Let \(\widehat f_B\) be a constrained minimizer and suppose that
		\(\norm{\widehat f_B}_{\Hk}=B\).  Set
		\[
		h=\widehat f_B-f^*,
		\qquad
		t=\left(\frac1n\sum_{i=1}^n h(X_i)^2\right)^{1/2}.
		\]
		
		Since \(\norm{f^*}_{\Hk}\le R\), comparison of the empirical objective at
		\(\widehat f_B\) and \(f^*\) gives
		\[
		\begin{aligned}
			t^2
			+
			2\delta B\frac1n\sum_{i=1}^n
			|Y_i-\widehat f_B(X_i)|
			+
			\delta^2B^2
			\le
			2\frac1n\sum_{i=1}^n\varepsilon_i h(X_i)
			+
			2\delta R\frac1n\sum_{i=1}^n|\varepsilon_i|
			+
			\delta^2R^2.
		\end{aligned}
		\]
		
		Since
		\[
		\frac1n\sum_{i=1}^n|Y_i-\widehat f_B(X_i)|
		\ge
		\frac1n\sum_{i=1}^n|\varepsilon_i|-t,
		\]
		the last inequality implies
		\begin{equation}
			\label{eq:boundary-basic}
			\left(t-\delta B\right)^2
			+
			2\delta B\frac1n\sum_{i=1}^n|\varepsilon_i|
			\le
			2\frac1n\sum_{i=1}^n\varepsilon_i h(X_i)
			+
			2\delta R\frac1n\sum_{i=1}^n|\varepsilon_i|
			+
			\delta^2R^2.
		\end{equation}
		
		Equation~\eqref{eq:empirical-l2-comparison} and
		\(\norm{h}_{\Hk}\le B+R\le2B\) give
		\[
		\norm{(T_k+\lambda I)^{1/2}h}_{\Hk}
		\le
		\sqrt2\,t+2\sqrt{2\delta B}.
		\]
		
		The empirical noise term equals
		\[
		\inner{
			(T_k+\lambda I)^{-1/2}n^{-1}\sum_{i=1}^n\varepsilon_i k_{X_i}
		}{
			(T_k+\lambda I)^{1/2}h
		}_{\Hk}.
		\]
		
		Since \(t\le|t-\delta B|+\delta B\), Cauchy--Schwarz and
		\eqref{eq:boundary-noise} give
		\[
		\begin{aligned}
			2\left|\frac1n\sum_{i=1}^n\varepsilon_i h(X_i)\right|
			\le
			2\sqrt{c_0\delta B}
			\left(\sqrt2\,t+2\sqrt{2\delta B}\right)\le
			\frac12(t-\delta B)^2
			+
			\left(4c_0+6\sqrt{2c_0}\right)\delta B.
		\end{aligned}
		\]
		
		Together with \eqref{eq:empirical-noise-absolute-mean} and the choice of
		\(c_0\), this yields
		\[
		2\left|\frac1n\sum_{i=1}^n\varepsilon_i h(X_i)\right|
		\le
		\frac12(t-\delta B)^2
		+
		\frac{\delta B}{2n}\sum_{i=1}^n|\varepsilon_i|.
		\]
		
		Substituting this estimate into \eqref{eq:boundary-basic} yields
		\[
		\frac12(t-\delta B)^2
		+
		\frac{3\delta B}{2n}\sum_{i=1}^n|\varepsilon_i|
		\le
		\frac{2\delta R}{n}\sum_{i=1}^n|\varepsilon_i|
		+
		\delta^2R^2.
		\]
		
		By \eqref{eq:empirical-noise-absolute-mean} and the choice of the upper
		bound on \(\delta\),
		\[
		\delta R
		\le
		\frac12\left(\frac1n\sum_{i=1}^n|\varepsilon_i|\right).
		\]
		Dropping the square term therefore gives \(3B/2\le5R/2\), contradicting
		\(B\ge2R\).
		Thus every constrained minimizer has norm strictly smaller than \(B\).
		
		If an unconstrained minimizer had norm larger than \(B\), the line
		segment joining it to a constrained minimizer would meet the boundary of
		the ball.  Convexity of the objective would make that boundary point a
		constrained minimizer, contradicting the boundary argument.  Hence every
		empirical minimizer has RKHS norm at most \(B\).
	\end{proof}
	
	Lemma~\ref{lem:full-empirical-norm} gives the fixed RKHS bound used below.  We next
	identify the quadratic form governing population growth on that ball.  If
	\(\Hk=\{0\}\), the estimation conclusion is immediate, so assume that \(\Hk\)
	is nontrivial.  Let \(v_\delta\) and \(s_\delta\) be the
	subgradient and measurable selection in
	Lemma~\ref{lem:population-foc}.  If \(v_\delta\ne0\), let
	\(u_\delta=v_\delta/\norm{v_\delta}_{\Hk}\); otherwise choose any unit
	vector \(u_\delta\in\Hk\).  Thus, when \(f_\delta^*\ne0\),
	\(u_\delta=f_\delta^*/\norm{f_\delta^*}_{\Hk}\).  The mixed term contributes
	order-\(\delta\) curvature in directions orthogonal to \(u_\delta\).  Define
	the positive operator
	\begin{equation}\label{def:adelta}
	A_\delta h
	=
	T_kh+(\delta^2+\delta)h
	-\delta\inner{h}{u_\delta}_{\Hk}u_\delta.
	\end{equation}
	The quadratic form is
	\begin{equation}
		\label{eq:Adelta-quadratic-form}
		\norm{A_\delta^{1/2}h}_{\Hk}^2
		=
		\norm{h}_{L^2(\mu)}^2
		+\delta^2\norm{h}_{\Hk}^2
		+\delta\left(
		\norm{h}_{\Hk}^2-\inner{h}{u_\delta}_{\Hk}^2
		\right).
	\end{equation}
	The first two terms are the prediction and quadratic-penalty contributions.
	The third term is the additional curvature in directions orthogonal to
	\(u_\delta\).
	
	\begin{lemma}[Quadratic growth around the population minimizer]
		\label{lem:full-quadratic-growth}
		Suppose Assumptions~\ref{ass:model}, \ref{ass:noise},
		and~\ref{ass:noise-tail} hold and \(f^*=T_k^\alpha g\) with
		\(\norm{g}_{\Hk}\le R\).  Let \(B\) be the constant in
		Lemma~\ref{lem:full-empirical-norm}.  There are
		constants \(c,C>0\), depending only on \(R,\sigma,\kappa\), such that, if
		\(0<\delta\le c\) and
		\(\norm{f_\delta^*+h}_{\Hk}\le B\), then
		\[
		\begin{aligned}
			&\E\left[
			\left(
			|Y-f_\delta^*(X)-h(X)|
			+
			\delta\norm{f_\delta^*+h}_{\Hk}
			\right)^2
			\right]
			\\
			&\quad-
			\E\left[
			\left(
			|Y-f_\delta^*(X)|
			+
			\delta\norm{f_\delta^*}_{\Hk}
			\right)^2
			\right]
			\ge
			C^{-1}\norm{A_\delta^{1/2}h}_{\Hk}^2.
		\end{aligned}
		\]
	\end{lemma}
	
	\begin{proof}[Proof of Lemma \ref{lem:full-quadratic-growth}]
		The subgradient inequality for the absolute loss gives, for every \(h\in\Hk\),
		\[
		\E|Y-f_\delta^*(X)-h(X)|
		-
		\E|Y-f_\delta^*(X)|
		-
		\E[s_\delta(X,\varepsilon)h(X)]
		\ge0.
		\]
		
		Multiplying this inequality by \(\norm{f_\delta^*+h}_{\Hk}\), and using
		\[
		\left|
		\norm{f_\delta^*+h}_{\Hk}-\norm{f_\delta^*}_{\Hk}
		\right|
		\le
		\norm{h}_{\Hk},
		\]
		yields
		\[
		\begin{aligned}
			&\norm{f_\delta^*+h}_{\Hk}\E|Y-f_\delta^*(X)-h(X)|
			-\norm{f_\delta^*}_{\Hk}\E|Y-f_\delta^*(X)|\\
			&\qquad-
			\norm{f_\delta^*}_{\Hk}\E[s_\delta(X,\varepsilon)h(X)]
			-
			\E|Y-f_\delta^*(X)|\inner{v_\delta}{h}_{\Hk}\\
			&\quad\ge
			\E|Y-f_\delta^*(X)|
			\left(
			\norm{f_\delta^*+h}_{\Hk}
			-\norm{f_\delta^*}_{\Hk}
			-\inner{v_\delta}{h}_{\Hk}
			\right)
			-
			4\kappa\norm{f_\delta^*-f^*}_{L^2(\mu)}
			\norm{h}_{L^2(\mu)}\norm{h}_{\Hk},
		\end{aligned}
		\]
		where \eqref{eq:population-sign-selection} was used in the last term.
		
		We next bound the remainder in the RKHS norm inequality.  If
		\(f_\delta^*\ne0\), then \(v_\delta=u_\delta\) and
		whenever the denominator below is positive,
		\[
		\begin{aligned}
			&\norm{f_\delta^*+h}_{\Hk}
			-\norm{f_\delta^*}_{\Hk}
			-\inner{v_\delta}{h}_{\Hk}\\
			&\quad
			=
			\frac{
				\norm{h}_{\Hk}^2-\inner{h}{u_\delta}_{\Hk}^2
			}{
				\norm{f_\delta^*+h}_{\Hk}
				+\norm{f_\delta^*}_{\Hk}
				+\inner{h}{u_\delta}_{\Hk}
			}
			\ge
			\frac{
				\norm{h}_{\Hk}^2-\inner{h}{u_\delta}_{\Hk}^2
			}{2B}.
		\end{aligned}
		\]
		where the last inequality uses
		\(\norm{f_\delta^*+h}_{\Hk}\le B\).  If the denominator is zero, the
		numerator is zero, and the same inequality follows directly from the
		subgradient inequality for the norm.
		If \(f_\delta^*=0\), then
		\(\norm{v_\delta}_{\Hk}\le1\), and
		for \(h\ne0\),
		\[
		\norm{h}_{\Hk}-\inner{v_\delta}{h}_{\Hk}
		\ge
		\norm{h}_{\Hk}-\left|\inner{h}{u_\delta}_{\Hk}\right|
		=
		\frac{
			\norm{h}_{\Hk}^2-\inner{h}{u_\delta}_{\Hk}^2
		}{\norm{h}_{\Hk}+\left|\inner{h}{u_\delta}_{\Hk}\right|}
		\ge
		\frac{
			\norm{h}_{\Hk}^2-\inner{h}{u_\delta}_{\Hk}^2
		}{2B}.
		\]
		For \(h=0\), this inequality is immediate.
		Moreover,
		\[
		\E|Y-f_\delta^*(X)|
		\ge
		\E|\varepsilon|
		\ge
		\frac1{4\kappa}.
		\]
		
		Expand the difference of the two population objectives.  The linear terms
		vanish by \eqref{populationfirstorder}; the bounds above give
		\[
		\begin{aligned}
			&\E\left[
			\left(
			|Y-f_\delta^*(X)-h(X)|
			+\delta\norm{f_\delta^*+h}_{\Hk}
			\right)^2
			\right]-
			\E\left[
			\left(
			|Y-f_\delta^*(X)|
			+\delta\norm{f_\delta^*}_{\Hk}
			\right)^2
			\right]\\
			&\ge
			\norm{h}_{L^2(\mu)}^2
			+\delta^2\norm{h}_{\Hk}^2
			+\frac{\delta}{4\kappa B}
			\left(
			\norm{h}_{\Hk}^2-\inner{h}{u_\delta}_{\Hk}^2
			\right)
			-
			8\delta\kappa
			\norm{f_\delta^*-f^*}_{L^2(\mu)}
			\norm{h}_{L^2(\mu)}\norm{h}_{\Hk}.
		\end{aligned}
		\]
		
		Young's inequality bounds the last term by
		\[
		\frac12\norm{h}_{L^2(\mu)}^2
		+
		32\delta^2\kappa^2
		\norm{f_\delta^*-f^*}_{L^2(\mu)}^2
		\norm{h}_{\Hk}^2.
		\]
		
		Finally, \eqref{eq:population-objective-error} and
		\(\norm{f^*}_{\Hk}\le R\) imply
		\[
		\norm{f_\delta^*-f^*}_{L^2(\mu)}^2
		\le
		\delta^2R^2+2\delta R\sigma.
		\]
		For \(\delta\) bounded above by a constant depending only on
		\(R,\sigma,\kappa\), the second term from Young's inequality is absorbed by
		\(\delta^2\norm{h}_{\Hk}^2\).  The remaining three terms are a fixed
		positive multiple of \(\norm{A_\delta^{1/2}h}_{\Hk}^2\).
	\end{proof}
	
	The complexity of balls in the quadratic form is controlled by the following
	trace bound.
	
	\begin{lemma}[Trace bound for \(A_\delta\)]
		\label{lem:full-trace-bound}
		\[
		\operatorname{Tr}(T_kA_\delta^{-1})
		\le
		\mathcal N(\delta^2+\delta)+1
		\le
		\mathcal N(\delta)+1.
		\]
	\end{lemma}
	
	\begin{proof}[Proof of Lemma \ref{lem:full-trace-bound}]
		The operator \(A_\delta\) is obtained from
		\(T_k+(\delta^2+\delta)I\) by subtracting the positive rank-one operator
		\(h\mapsto\delta\inner{h}{u_\delta}_{\Hk}u_\delta\).  The resolvent
		identity therefore implies that
		\[
		A_\delta^{-1}
		-
		\{T_k+(\delta^2+\delta)I\}^{-1}
		\]
		is positive and has rank at most one.  Also \(A_\delta\succeq T_k\), so
		\[
		0
		\preceq
		T_k^{1/2}
		\left[
		A_\delta^{-1}
		-
		\{T_k+(\delta^2+\delta)I\}^{-1}
		\right]
		T_k^{1/2}
		\preceq I.
		\]
		
		The trace of this positive rank-one contraction is at most one.  Hence
		\[
		\begin{aligned}
			\operatorname{Tr}(T_kA_\delta^{-1})
			\le
			\operatorname{Tr}
			\left[
			T_k\{T_k+(\delta^2+\delta)I\}^{-1}
			\right]+1=
			\mathcal N(\delta^2+\delta)+1
			\le
			\mathcal N(\delta)+1.
		\end{aligned}
		\]
	\end{proof}
	
	The final auxiliary lemma compares empirical and population objective
	differences uniformly over each ball defined by
	\eqref{eq:Adelta-quadratic-form}.
	
	\begin{lemma}[Uniform deviation of objective differences]
		\label{lem:full-objective-increments}
		Suppose Assumptions~\ref{ass:model}, \ref{ass:noise},
		\ref{ass:noise-tail}, and \ref{ass:source} hold.  Independence of
		\(\varepsilon\) and \(X\) is used only to control the truncation error.  Let
		\(0<\delta\le1\), \(0<\eta<1\), and let \(B\) be as
		in Lemma~\ref{lem:full-empirical-norm}.  There is a constant
		\(C>0\), depending only on \(R,\sigma,\kappa\), such that, for every
		\(t>0\), with probability at least \(1-\eta\),
		\[
		\begin{aligned}
			&
			\sup_{\substack{
					\norm{A_\delta^{1/2}h}_{\Hk}\le t\\
					\norm{f_\delta^*+h}_{\Hk}\le B
			}}
			\Bigg|
			\frac1n\sum_{i=1}^n
			\left(
			|Y_i-f_\delta^*(X_i)-h(X_i)|
			+\delta\norm{f_\delta^*+h}_{\Hk}
			\right)^2
			\\
			&\qquad
			-\frac1n\sum_{i=1}^n
			\left(
			|Y_i-f_\delta^*(X_i)|
			+\delta\norm{f_\delta^*}_{\Hk}
			\right)^2
			\\
			&\qquad
			-\E
			\left(
			|Y-f_\delta^*(X)-h(X)|
			+\delta\norm{f_\delta^*+h}_{\Hk}
			\right)^2
			\\
			&\qquad
			+\E
			\left(
			|Y-f_\delta^*(X)|
			+\delta\norm{f_\delta^*}_{\Hk}
			\right)^2
			\Bigg|
			\\
			&\quad\le
			C\log\left(\frac{16n}{\eta}\right)
			\left[
			t
			\sqrt{
				\frac{
					(\mathcal N(\delta)+1)\log(16/\eta)
				}{n}
			}
			+
			\frac{B\log(16/\eta)}{n}
			\right].
		\end{aligned}
		\]
	\end{lemma}
	
	\begin{proof}[Proof of Lemma \ref{lem:full-objective-increments}]
		
		For every \(h\) in the index set,
		Equation~\eqref{eq:Adelta-quadratic-form}, Lemma~\ref{normcontrol}, the
		triangle inequality, and \(B\ge R\) give
		\begin{equation}
			\label{eq:objective-increment-norm-bounds}
			\norm{h}_{L^2(\mu)}\le t,
			\qquad
			\delta\norm{h}_{\Hk}\le t,
			\qquad
			\norm{h}_{\Hk}\le\min\left\{\frac t\delta,2B\right\}.
		\end{equation}
		Because \(\Hk\) is separable and the objective increments are continuous
		in the RKHS norm, the supremum below equals the supremum over a countable
		dense subset of the index set and is measurable.
		
		The difference-of-squares identity also gives
		\begin{equation}
			\label{eq:full-increment-lipschitz}
			\begin{aligned}
				&\left|
				\left(
				|Y-f_\delta^*(X)-h(X)|
				+\delta\norm{f_\delta^*+h}_{\Hk}
				\right)^2
				-
				\left(
				|Y-f_\delta^*(X)|
				+\delta\norm{f_\delta^*}_{\Hk}
				\right)^2
				\right|\\
				&\quad
				\le
				C(1+|\varepsilon|+R)
				\left(|h(X)|+\delta\norm{h}_{\Hk}\right).
			\end{aligned}
		\end{equation}

		Set
		\[
		\begin{aligned}
			M&=2\sigma\log(16n/\eta),\\
			\overline\varepsilon
			&=\sgn(\varepsilon)\min\{|\varepsilon|,M\},\\
			\overline Y&=f^*(X)+\overline\varepsilon.
		\end{aligned}
		\]
		
		Equation~\eqref{eq:noise-exponential-tail} gives
		\[
		\mathbb P\left(\max_{1\le i\le n}|\varepsilon_i|>M\right)
		\le\frac{\eta}{4}.
		\]
		
		On the complementary event, the empirical objective differences based on \(Y_i\) and
		\(\overline Y_i\) agree.  Independence of \(X\) and \(\varepsilon\),
		\eqref{eq:objective-increment-norm-bounds}, and
		\eqref{eq:full-increment-lipschitz} further give
		\begin{equation}
			\label{eq:full-truncation-bias}
			\begin{aligned}
				&
				\sup_h
				\Bigg|
				\E
				\left(
				|Y-f_\delta^*(X)-h(X)|
				+\delta\norm{f_\delta^*+h}_{\Hk}
				\right)^2
				\\
				&\qquad
				-\E
				\left(
				|Y-f_\delta^*(X)|
				+\delta\norm{f_\delta^*}_{\Hk}
				\right)^2
				\\
				&\qquad
				-\E
				\left(
				|\overline Y-f_\delta^*(X)-h(X)|
				+\delta\norm{f_\delta^*+h}_{\Hk}
				\right)^2
				\\
				&\qquad
				+\E
				\left(
				|\overline Y-f_\delta^*(X)|
				+\delta\norm{f_\delta^*}_{\Hk}
				\right)^2
				\Bigg|
				\\
				&\quad
				\le
				\frac{Ct}{n}\log\left(\frac{16n}{\eta}\right).
			\end{aligned}
		\end{equation}
		Here the supremum is over the index set in the lemma.  From
		\eqref{eq:noise-exponential-tail} and tail integration,
		\[
		\begin{aligned}
		\E\!\left[
		(1+|\varepsilon|+R)\mathbf 1\{|\varepsilon|>M\}
		\right]
		&\le C(1+M+R)\exp\left(-\frac{M}{2\sigma}\right)\\
		&\le \frac{C}{n}\log\left(\frac{16n}{\eta}\right),
		\end{aligned}
		\]
		which proves the stated truncation bound.

		For the truncated responses, the objective difference is
		\(C\log(16n/\eta)\)-Lipschitz in the two coordinates
		\[
		h(X_i)
		\quad\text{and}\quad
		\delta\left(
		\norm{f_\delta^*+h}_{\Hk}-\norm{f_\delta^*}_{\Hk}
		\right),
		\]
		and vanishes when both coordinates are zero.  Symmetrization and the vector
		contraction inequality \citep{maurer2016vector}, applied with independent
		Rademacher sequences for the two coordinates, therefore bound the expected
		supremum of the centered process by
		\[
		\begin{aligned}
			C\log\left(\frac{16n}{\eta}\right)
			\Bigg[
			t\,\E\left\|
			A_\delta^{-1/2}\frac1n\sum_{i=1}^n\xi_i k_{X_i}
			\right\|_{\Hk}+
			\delta\sup_{\norm{A_\delta^{1/2}h}_{\Hk}\le t}
			\norm{h}_{\Hk}
			\E\left|\frac1n\sum_{i=1}^n\xi_i\right|
			\Bigg],
		\end{aligned}
		\]
		where each occurrence of \(\xi_i\) denotes the corresponding independent
		Rademacher sequence.  Independence and
		Lemma~\ref{lem:full-trace-bound} give
		\[
		\begin{aligned}
			\E\left\|
			A_\delta^{-1/2}\frac1n\sum_{i=1}^n\xi_i k_{X_i}
			\right\|_{\Hk}^2
			=
			\frac1n\operatorname{Tr}(T_kA_\delta^{-1})\le
			\frac{\mathcal N(\delta)+1}{n}.
		\end{aligned}
		\]
		
		Jensen's inequality then gives
		\[
		\E\left\|
		A_\delta^{-1/2}\frac1n\sum_{i=1}^n\xi_i k_{X_i}
		\right\|_{\Hk}
		\le
		\sqrt{\frac{\mathcal N(\delta)+1}{n}}.
		\]
		
		The second term is at most \(t/\sqrt n\) by
		\eqref{eq:objective-increment-norm-bounds}.  Hence the expected supremum is at most
		\[
		Ct\log\left(\frac{16n}{\eta}\right)
		\sqrt{\frac{\mathcal N(\delta)+1}{n}}.
		\]

		Equation~\eqref{eq:full-increment-lipschitz} and
		\eqref{eq:objective-increment-norm-bounds} bound the variance of each centered truncated
		increment by
		\[
		Ct^2\log^2\left(\frac{16n}{\eta}\right)
		\]
		and its envelope by
		\[
		CB\log\left(\frac{16n}{\eta}\right).
		\]
		The Talagrand--Bousquet inequality \citep{bousquet2002bennett}, applied to
		the centered truncated increments with confidence level \(\eta/2\), bounds
		their supremum by twice its expectation plus
		\[
		C\log\left(\frac{16n}{\eta}\right)
		\left[
		t\sqrt{\frac{\log(4/\eta)}{n}}
		+\frac{B\log(4/\eta)}{n}
		\right].
		\]
		Combining this with the truncation event and
		\eqref{eq:full-truncation-bias} proves the stated bound.
	\end{proof}
	\begin{proof}[Proof of Theorem~\ref{thm:estimation-error}]
		After increasing the constant, the sample-size conditions in the theorem
		imply those of Lemma~\ref{lem:full-empirical-norm} at confidence level
		\(\eta/2\).  Thus, every empirical minimizer has RKHS norm at most \(B\) with
		probability at least \(1-\eta/2\).
		
		Let
		\[
		r_n
		=
		\sqrt{
			\frac{
				(\mathcal N(\delta)+1)\log^3(2n/\eta)
			}{n}
		}.
		\]
		
		Choose a sufficiently large constant \(M\), depending only on
		\(R,\sigma,\kappa\).  Apply Lemma~\ref{lem:full-objective-increments} with
		confidence level \(\eta/2\) and radius \(Mr_n\).  The right-hand side is at
		most
		\[
		C(M+1)r_n^2,
		\]
		because \(B\) is fixed, \(\mathcal N(\delta)+1\ge1\), and all logarithms in
		Lemma~\ref{lem:full-objective-increments} are bounded by a constant multiple of \(\log(2n/\eta)\).
		Thus, the norm and deviation bounds hold together with probability at
		least \(1-\eta\).
		
		We show that an empirical minimizer cannot leave the ball of
		radius \(Mr_n\) around \(f_\delta^*\) in the quadratic form
		\eqref{eq:Adelta-quadratic-form}.
		Suppose on this event that an empirical minimizer satisfies
		\[
		\norm{A_\delta^{1/2}
			(\widehat f_\delta-f_\delta^*)}_{\Hk}
		>
		Mr_n.
		\]
		
		The line segment from \(f_\delta^*\) to \(\widehat f_\delta\) meets the
		boundary of this ball at \(f_\delta^*+h\), where
		\[
		h
		=
		\frac{Mr_n}{
			\norm{A_\delta^{1/2}(\widehat f_\delta-f_\delta^*)}_{\Hk}
		}
		(\widehat f_\delta-f_\delta^*).
		\]
		
		In particular, \(\norm{A_\delta^{1/2}h}_{\Hk}=Mr_n\).
		Both endpoints have RKHS norm at most \(B\), so
		\(\norm{f_\delta^*+h}_{\Hk}\le B\).  Convexity of the empirical objective
		and optimality of \(\widehat f_\delta\) give
		\[
		\frac1n\sum_{i=1}^n
		\left(
		|Y_i-f_\delta^*(X_i)-h(X_i)|
		+\delta\norm{f_\delta^*+h}_{\Hk}
		\right)^2
		\le
		\frac1n\sum_{i=1}^n
		\left(
		|Y_i-f_\delta^*(X_i)|
		+\delta\norm{f_\delta^*}_{\Hk}
		\right)^2.
		\]
		
		Lemma~\ref{lem:full-quadratic-growth} gives population increase of order
		\(M^2r_n^2\) at this boundary point, while
		Lemma~\ref{lem:full-objective-increments} bounds the empirical-to-population
		difference by order \((M+1)r_n^2\).  Consequently,
		\[
		\begin{aligned}
			&
			\frac1n\sum_{i=1}^n
			\left(
			|Y_i-f_\delta^*(X_i)-h(X_i)|
			+\delta\norm{f_\delta^*+h}_{\Hk}
			\right)^2
			-\frac1n\sum_{i=1}^n
			\left(
			|Y_i-f_\delta^*(X_i)|
			+\delta\norm{f_\delta^*}_{\Hk}
			\right)^2
			\\
			&\quad
			\ge
			\left(cM^2-CM-C\right)r_n^2.
		\end{aligned}
		\]
		
		Taking \(M\) large enough makes the right-hand side positive, a
		contradiction.  Therefore
		\[
		\norm{A_\delta^{1/2}
			(\widehat f_\delta-f_\delta^*)}_{\Hk}
		\le
		Mr_n.
		\]
		
		Equation~\eqref{eq:Adelta-quadratic-form} gives
		\(A_\delta\succeq T_k\).  Therefore,
		\[
		\norm{\widehat f_\delta-f_\delta^*}_{L^2(\mu)}
		\le
		Mr_n.
		\]
	\end{proof}
	\subsection{Proof of Corollary~\ref{cor:generalization-error}: \nameref*{cor:generalization-error}}
	\begin{proof}
		The triangle inequality and
		Theorems~\ref{thm:full-approximation}--\ref{thm:estimation-error} give
		\[
		\begin{aligned}
			\norm{\widehat f_\delta-f^*}_{L^2(\mu)}
			&\le
			\norm{\widehat f_\delta-f_\delta^*}_{L^2(\mu)}
			+
			\norm{f_\delta^*-f^*}_{L^2(\mu)}
			\\
			&\lesssim
			\delta^{\frac{\alpha+1/2}{\alpha+1}}
			+
			\sqrt{
				\frac{
					(\mathcal N(\delta)+1)\log^3(2n/\eta)
				}{n}
			}.
		\end{aligned}
		\]
		Both component bounds are uniform over \(\norm{g}_{\Hk}\le R\), which proves
		Corollary~\ref{cor:generalization-error}.
		
		For the polynomial-eigenvalue specialization stated after the corollary, the
		spectral decomposition of \(T_k\) gives
		\[
		\mathcal N(\delta)
		=
		\sum_{j\ge1}\frac{\sigma_j}{\sigma_j+\delta}.
		\]

		Since \(\sigma_j/(\sigma_j+\delta)\le
		\min\{1,\sigma_j/\delta\}\), splitting the sum at
		\(\lceil\delta^{-1/\beta}\rceil\) and applying
		Assumption~\ref{ass:eigendecay} gives
		\[
		\begin{aligned}
			\mathcal N(\delta)
			&\le
			\left\lceil\delta^{-1/\beta}\right\rceil
			+
			\delta^{-1}
			\sum_{j>\lceil\delta^{-1/\beta}\rceil}\sigma_j\\
			&\le
			\left\lceil\delta^{-1/\beta}\right\rceil
			+
			C_\beta\delta^{-1}
			\sum_{j>\lceil\delta^{-1/\beta}\rceil}j^{-\beta}\\
			&\le
			\left\lceil\delta^{-1/\beta}\right\rceil
			+
			\frac{C_\beta}{\beta-1}\delta^{-1}
			\left\lceil\delta^{-1/\beta}\right\rceil^{1-\beta}
			\lesssim
			\delta^{-1/\beta},
		\end{aligned}
		\]
		where the third line is due to \(\beta>1\).
		
		Since
		\(\delta^{-1/\beta}\ge1\) on this range,
		\[
		\mathcal N(\delta)+1\lesssim\delta^{-1/\beta}.
		\]
		Substituting this inequality into the operator-form bound gives the
		polynomial bound in the main text.
	\end{proof}
	
	\subsection{Proof of Theorem~\ref{thm:full-at-matching-lower}: \nameref*{thm:full-at-matching-lower}}
	
	\begin{proof}
		\proofstep{\textbf{Step 1: construct the model.}}
		Let \(\mathcal X=\mathbb N\), and suppose that \(X\) has probability mass
		function
		\[
		p_j:=\mathbb P(X=j)
		=
		\frac{j^{-\beta}}{\sum_{\ell\ge1}\ell^{-\beta}},
		\qquad j\ge1.
		\]
		
		On this space, consider the kernel
		\(k(j,\ell)=\mathbf 1\{j=\ell\}\).
		Then \(\Hk=\ell^2(\mathbb N)\), \(T_k\) is diagonal with eigenvalues
		\(p_j\asymp j^{-\beta}\), and
		\[
		\norm{f}_{L^2(\mu)}^2
		=
		\sum_{j\ge1}p_jf(j)^2.
		\]
		
		Let \(S\) take values \(1\) and \(-1\) with equal probability.  Independently,
		let \(U\sim\operatorname{Unif}[0,1/8]\), and set
		\(\varepsilon=S(2+U)\).  This distribution is symmetric about zero and
		bounded.  A direct calculation gives
		\[
		\frac{d}{dt}\E|\varepsilon-t|
		=
		\begin{cases}
			0, & |t|<2,\\
			8\operatorname{sgn}(t)(|t|-2), & 2\le |t|\le17/8,\\
			\operatorname{sgn}(t), & |t|>17/8.
		\end{cases}
		\]
		
		The derivative has the same sign as \(t\), and the absolute value of the
		derivative is at most \(|t|/2\).  Integrating from zero to \(t\) yields
		\[
		0
		\le
		\E|\varepsilon-t|-\E|\varepsilon|
		\le
		\frac{t^2}{4}.
		\]
		
		Thus Assumption~\ref{ass:noise} holds with \(\kappa=1/4\), while boundedness
		gives Assumption~\ref{ass:noise-tail}.
		
		For sufficiently small \(\delta\), set
		\[
		J_\delta
		=
		\{j:\delta\le p_j\le2\delta\}.
		\]
		
		The definition of \(p_j\) shows that
		\[
		J_\delta
		=
		\left\{
		j:
		\left(
		2\delta\sum_{\ell\ge1}\ell^{-\beta}
		\right)^{-1/\beta}
		\le j\le
		\left(
		\delta\sum_{\ell\ge1}\ell^{-\beta}
		\right)^{-1/\beta}
		\right\}.
		\]
		
		Accounting for the integer endpoints gives
		\[
		|J_\delta|
		=
		\left(
		\sum_{\ell\ge1}\ell^{-\beta}
		\right)^{-1/\beta}
		\left(1-2^{-1/\beta}\right)\delta^{-1/\beta}
		+O(1).
		\]
		
		The definition of \(J_\delta\) also gives
		\[
		\delta|J_\delta|
		\le
		\sum_{j\in J_\delta}p_j
		\le
		2\delta|J_\delta|.
		\]
		Finally,
		\[
		\mathcal N(\delta)
		\ge
		\sum_{j\in J_\delta}\frac{p_j}{p_j+\delta}
		\ge
		\frac12|J_\delta|.
		\]
		
		For the reverse inequality, split the sum defining
		\(\mathcal N(\delta)\) according to whether \(p_j\ge\delta\).  This gives
		\[
		\begin{aligned}
			\mathcal N(\delta)
			\le
			\big|\{j:p_j\ge\delta\}\big|
			+
			\frac1\delta\sum_{j:p_j<\delta}p_j\lesssim
			\delta^{-1/\beta},
		\end{aligned}
		\]
		where the last inequality follows from
		\(\sum_{j>m}j^{-\beta}\lesssim m^{1-\beta}\) and \(\beta>1\).

		Consequently,
		\begin{equation}
			\label{eq:lower-critical-band}
			|J_\delta|\asymp\delta^{-1/\beta},
			\qquad
			\delta|J_\delta|
			\le
			\sum_{j\in J_\delta}p_j
			\le
			2\delta|J_\delta|,
			\qquad
			\mathcal N(\delta)\asymp|J_\delta|.
		\end{equation}
		
		\proofstep{\textbf{Step 2: population lower bound.}}
		Choose \(j_\delta\) so that
		\[
		\frac12
		\left(
		\frac{\delta\E|\varepsilon|}{2R}
		\right)^{1/(\alpha+1)}
		\le
		p_{j_\delta}
		\le
		\left(
		\frac{\delta\E|\varepsilon|}{2R}
		\right)^{1/(\alpha+1)}.
		\]
		
		Such an index exists for all sufficiently small \(\delta\) because
		\(p_{j+1}/p_j\to1\).  Define
		\(g_\delta(j_\delta)=R\) and \(g_\delta(j)=0\) otherwise.  Then
		\[
		\norm{T_k(T_k^\alpha g_\delta)}_{\Hk}
		=
		Rp_{j_\delta}^{\alpha+1}
		\le
		\frac{\delta}{2}\E|\varepsilon|.
		\]
		
		The noise calculation in Step 1 gives
		\(\E|Y|\ge\E|\varepsilon|\).
		For every nonzero \(h\in\Hk\), the right directional derivative at zero of
		one half of the population objective is bounded below by
		\[
		\begin{aligned}
			-\inner{T_k(T_k^\alpha g_\delta)}{h}_{\Hk}
			+\delta\E|Y|\norm{h}_{\Hk}
			&\ge
			\left(
			\delta\E|\varepsilon|
			-\norm{T_k(T_k^\alpha g_\delta)}_{\Hk}
			\right)\norm{h}_{\Hk}\\
			&\ge
			\frac{\delta}{2}\E|\varepsilon|\norm{h}_{\Hk}
			>0.
		\end{aligned}
		\]
		
		Convexity therefore makes zero the unique population minimizer.  Hence
		\[
		\sup_{\norm{g}_{\Hk}\le R}
		\norm{f_\delta^*-T_k^\alpha g}_{L^2(\mu)}^2
		\ge
		R^2p_{j_\delta}^{2\alpha+1}
		\gtrsim
		\delta^{\frac{2\alpha+1}{\alpha+1}}.
		\]
		
		For the empirical bound, assume
		\begin{equation}
			\label{eq:lower-sample-condition}
			n\delta\ge C\log(1/\delta).
		\end{equation}
		The balancing sequence in the theorem satisfies this condition for all
		sufficiently large \(n\).
		
		\proofstep{\textbf{Step 3: fix the signal and establish bounds holding with high probability.}}
		Set
		\[
		q=\frac14\min\left\{Rp_1^\alpha,1\right\}.
		\]
		Consider the fixed regression function
		\[
		f^{(0)}
		=
		q\mathbf 1_{\{1\}}
		=
		T_k^\alpha
		\left(
		qp_1^{-\alpha}\mathbf 1_{\{1\}}
		\right).
		\]
		
		The source element on the right has RKHS norm at most \(R/4\), and
		\(f^{(0)}\) vanishes on \(J_\delta\).  Write
		\(\varepsilon_i=S_i(2+U_i)\) for the independent noise variables in the
		sample.  For \(j\ge1\), define
		\[
		\widehat p_j
		=
		\frac1n\sum_{i=1}^n\mathbf 1\{X_i=j\},
		\qquad
		A_j
		=
		\frac1n\sum_{i=1}^n\varepsilon_i\mathbf 1\{X_i=j\},
		\qquad
		B_j
		=
		\frac1n\sum_{i=1}^nS_i\mathbf 1\{X_i=j\}.
		\]
		
		The empirical objective at \(f^{(0)}\) is at most
		\((17/8+\delta q)^2\le(19/8)^2\).  Optimality therefore gives, for every
		empirical minimizer,
		\[
		\delta\norm{\widehat f_\delta}_{\Hk}\le\frac{19}{8},
		\qquad
		\frac1n\sum_{i=1}^n
		(Y_i-\widehat f_\delta(X_i))^2
		\le
		\left(\frac{19}{8}\right)^2.
		\]
		To bound the first coordinate, let \(f_{q/2}\) agree with
		\(\widehat f_\delta\) except that \(f_{q/2}(1)=q/2\).  The observations with
		\(X_i\ne1\) contribute at most \((19/8)^2\) to the empirical mean squared
		residual of \(f_{q/2}\), while
		\[
		|Y_i-f_{q/2}(X_i)|
		=
		|\varepsilon_i+q/2|
		\le\frac94.
		\]
		
		Hence, the empirical mean squared residual of \(f_{q/2}\) is less than \(16\),
		and the corresponding empirical mean absolute residual is at most \(4\).
		For observations with \(X_i=1\), the subgradient of the absolute loss at \(q/2\)
		is \(-S_i\), since \(|\varepsilon_i|\ge2\) and \(q\le1/4\).
		The derivative of the corresponding objective along this coordinate at \(q/2\),
		divided by two, equals
		\[
		-\frac q2\widehat p_1-A_1+\frac{\delta^2q}{2}
		+
		\delta\frac{q/2}{\norm{f_{q/2}}_{\Hk}}
		\frac1n\sum_{i=1}^n|Y_i-f_{q/2}(X_i)|
		-
		\delta\norm{f_{q/2}}_{\Hk}B_1.
		\]
		
		Since only the first coordinate was changed,
		\[
		\delta^2\norm{f_{q/2}}_{\Hk}^2
		\le
		\delta^2\norm{\widehat f_\delta}_{\Hk}^2
		+
		\frac{\delta^2q^2}{4}
		\le
		\left(\frac{19}{8}\right)^2+\frac1{64}<9.
		\]
		
		Thus, \(\delta\norm{f_{q/2}}_{\Hk}<3\).

		The count \(n\widehat p_1\) is binomial with parameters \(n\) and \(p_1\),
		so Chernoff's inequality gives
		\[
		\mathbb P\left(\widehat p_1<\frac34p_1\right)\le e^{-np_1/32}.
		\]
		Conditional on \(X_1,\ldots,X_n\), the summands defining \(A_1\) and
		\(B_1\) are independent, centered, and bounded. Hoeffding's inequality
		therefore gives
		\[
		\mathbb P\left(
		|A_1|+3|B_1|>\frac{p_1q}{16}
		\,\middle|\,X_1,\ldots,X_n
		\right)
		\le Ce^{-cn},
		\]
		where \(c>0\) depends only on the fixed constants \(p_1\) and \(q\).
		Consequently, except on an event of probability \(Ce^{-cn}\),
		\[
		\widehat p_1\ge\frac34p_1,
		\qquad
		|A_1|+3|B_1|\le\frac{p_1q}{16}.
		\]
		
		If \(\delta\) is small enough that
		\[
		4\delta+\frac{\delta^2q}{2}
		\le
		\frac{p_1q}{16},
		\]
		the displayed derivative is at most \(-p_1q/4\).  Convexity of the
		coordinate objective therefore gives
		\[
		\norm{\widehat f_\delta}_{\Hk}
		\ge
		\widehat f_\delta(1)
		>
		\frac q2.
		\]
		
		For each \(j\in J_\delta\), Chernoff's inequality gives
		\[
		\mathbb P\left(
		\frac12np_j\le n\widehat p_j\le2np_j
		\right)
		\ge1-2e^{-cnp_j}.
		\]
		Conditional on the covariates and on this count event, Hoeffding's
		inequality applied to the bounded centered summands in \(A_j\) and \(B_j\)
		gives
		\[
		\mathbb P\left(
		|A_j|+4|B_j|>\frac12\widehat p_j
		\,\middle|\,X_1,\ldots,X_n
		\right)
		\le Ce^{-cn\widehat p_j}
		\le Ce^{-cn\delta}.
		\]
		A union bound over \(J_\delta\) now yields
		\[
		\begin{aligned}
			\mathbb P\Bigg(
			\frac12p_j\le\widehat p_j\le2p_j,|A_j|+4|B_j|\le\frac12\widehat p_j
			\quad\text{for every }j\in J_\delta
			\Bigg)
			\ge
			1-C|J_\delta|e^{-cn\delta}.
		\end{aligned}
		\]
		
		By \eqref{eq:lower-critical-band} and
		\eqref{eq:lower-sample-condition}, the constants can be chosen so that,
		with probability at least \(99/100\), every empirical minimizer satisfies
		\begin{equation}
			\label{eq:lower-simultaneous-bounds}
			\norm{\widehat f_\delta}_{\Hk}>\frac q2,
			\qquad
			\frac12p_j\le\widehat p_j\le2p_j,
			\qquad
			|A_j|+4|B_j|\le\frac12\widehat p_j
			\quad
			\text{for every }j\in J_\delta.
		\end{equation}
		
		\proofstep{\textbf{Step 4: derive the coordinate equations for
				\(j\in J_\delta\).}}
		Assume that \eqref{eq:lower-simultaneous-bounds} holds.  Replacing the
		\(j\)th coordinate of \(\widehat f_\delta\) by either \(1\) or \(-1\)
		changes no other coordinate.  The squared scaled norm of either modified
		function is therefore at most
		\[
		\delta^2\norm{\widehat f_\delta}_{\Hk}^2+\delta^2
		\le
		\left(\frac{19}{8}\right)^2+1
		<9.
		\]
		Thus both modified functions have RKHS norm smaller than \(3/\delta\).  Since
		\(|\varepsilon_i|\ge2\), the subgradient of the absolute loss on observations
		with \(X_i=j\) is \(-S_i\).  The remaining norm terms have the sign of the
		replacement coordinate, and the term involving \(B_j\) has absolute value
		at most \(3|B_j|\).  After division by two, the derivatives at \(1\) and
		\(-1\) therefore satisfy
		\[
		\widehat p_j-|A_j|-3|B_j|>0,
		\qquad
		-\widehat p_j+|A_j|+3|B_j|<0.
		\]
		
		Convexity gives \(|\widehat f_\delta(j)|<1\) for every \(j\in J_\delta\).
		The first order condition for this coordinate is therefore
		\begin{equation}
			\label{eq:lower-band-equation}
			\left(
			\widehat p_j+\delta^2
			+
			\delta
			\frac{
				n^{-1}\sum_{i=1}^n|Y_i-\widehat f_\delta(X_i)|
			}{
				\norm{\widehat f_\delta}_{\Hk}
			}
			\right)
			\widehat f_\delta(j)
			=
			A_j+\delta\norm{\widehat f_\delta}_{\Hk}B_j.
		\end{equation}
		
		The Cauchy Schwarz inequality and the bound in Step 3 show that the empirical
		mean absolute residual is at most \(19/8\).  Together with
		\eqref{eq:lower-simultaneous-bounds}, \(p_j\le2\delta\), and
		\(\delta\le1\), this gives
		\begin{equation}
			\label{eq:lower-band-coefficient}
			\widehat p_j+\delta^2
			+
			\delta
			\frac{
				n^{-1}\sum_{i=1}^n|Y_i-\widehat f_\delta(X_i)|
			}{
				\norm{\widehat f_\delta}_{\Hk}
			}
			\le
			\left(
			5+\frac{19}{4q}
			\right)\delta.
		\end{equation}
		
		\proofstep{\textbf{Step 5: derive a lower bound for the random score over \(J_\delta\).}}
		Since
		\[
		\varepsilon_i
		=
		\frac{33}{16}S_i
		+S_i\left(U_i-\frac1{16}\right),
		\]
		we have
		\[
		A_j
		=
		\frac{33}{16}B_j
		+
		\frac1n\sum_{i=1}^n
		S_i\left(U_i-\frac1{16}\right)\mathbf 1\{X_i=j\}.
		\]
		
		Conditional on \(X_1,\ldots,X_n\), let
		\[
		N_{J_\delta}
		=
		\sum_{i=1}^n\mathbf 1\{X_i\in J_\delta\}.
		\]
		
		Independence and symmetry of the signs give
		\[
		\E\left[
		\sum_{j\in J_\delta}B_j^2
		\,\middle|\,X_1,\ldots,X_n
		\right]
		=
		\frac{N_{J_\delta}}{n^2},
		\]
		and
		\[
		\E\left[
		\left(\sum_{j\in J_\delta}B_j^2\right)^2
		\,\middle|\,X_1,\ldots,X_n
		\right]
		\le
		3\frac{N_{J_\delta}^2}{n^4}.
		\]
		If \(N_{J_\delta}=0\), the score bound below is immediate.  Suppose
		\(N_{J_\delta}>0\).  The Paley--Zygmund inequality then gives
		\[
		\mathbb P\left(
		\sum_{j\in J_\delta}B_j^2
		\ge
		\frac{N_{J_\delta}}{2n^2}
		\,\middle|\,X_1,\ldots,X_n
		\right)
		\ge\frac1{12}.
		\]
		
		Since \(\operatorname{Var}(U)=1/768\), Markov's inequality also gives
		\[
		\E\left[
		\sum_{j\in J_\delta}
		\left[
		\frac1n\sum_{i=1}^n
		S_i\left(U_i-\frac1{16}\right)\mathbf 1\{X_i=j\}
		\right]^2
		\,\middle|\,X_1,\ldots,X_n
		\right]
		=
		\frac{N_{J_\delta}}{768n^2}.
		\]
		
		Therefore,
		\[
		\mathbb P\left(
		\sum_{j\in J_\delta}
		\left[
		\frac1n\sum_{i=1}^n
		S_i\left(U_i-\frac1{16}\right)\mathbf 1\{X_i=j\}
		\right]^2
		\le
		\frac{N_{J_\delta}}{8n^2}
		\,\middle|\,X_1,\ldots,X_n
		\right)
		\ge\frac{95}{96}.
		\]
		
		The two events intersect with conditional probability at least \(7/96\).
		On the intersection, the reverse triangle inequality gives, for every
		\(t\ge0\),
		\[
		\begin{aligned}
			\left(
			\sum_{j\in J_\delta}(A_j+tB_j)^2
			\right)^{1/2}
			&\ge
			\left(\frac{33}{16}+t\right)
			\left(
			\sum_{j\in J_\delta}B_j^2
			\right)^{1/2}\\
			&\quad-
			\left\{
			\sum_{j\in J_\delta}
			\left[
			\frac1n\sum_{i=1}^n
			S_i\left(U_i-\frac1{16}\right)\mathbf 1\{X_i=j\}
			\right]^2
			\right\}^{1/2}\\
			&\ge
			\frac{25}{16\sqrt2}\frac{\sqrt{N_{J_\delta}}}{n}.
		\end{aligned}
		\]
		Consequently, with conditional probability at least \(7/96\),
		\begin{equation}
			\label{eq:lower-score}
			\sum_{j\in J_\delta}(A_j+tB_j)^2
			\ge
			\frac{N_{J_\delta}}{n^2}
			\qquad
			\text{for every }t\ge0.
		\end{equation}
		
		\proofstep{\textbf{Step 6: conclude and specialize the radius.}}
		On the event in \eqref{eq:lower-simultaneous-bounds},
		\[
		N_{J_\delta}
		=
		n\sum_{j\in J_\delta}\widehat p_j
		\ge
		\frac{n\delta|J_\delta|}{2}.
		\]
		The bound in \eqref{eq:lower-score} holds for every \(t\ge0\), so it applies
		with \(t=\delta\norm{\widehat f_\delta}_{\Hk}\).  The simultaneous bounds
		in \eqref{eq:lower-simultaneous-bounds} and the score bound in
		\eqref{eq:lower-score} therefore hold together with probability at least
		\(7/96-1/100>1/16\).  Whenever both bounds hold,
		\eqref{eq:lower-band-equation}, \eqref{eq:lower-band-coefficient}, and
		\(p_j\ge\delta\) give
		\[
		\begin{aligned}
			\norm{\widehat f_\delta-f^{(0)}}_{L^2(\mu)}^2
			&\ge
			\sum_{j\in J_\delta}p_j\widehat f_\delta(j)^2\\
			&\ge
			\frac{1}{
				\left(5+19/(4q)\right)^2\delta
			}
			\sum_{j\in J_\delta}
			\left(
			A_j+\delta\norm{\widehat f_\delta}_{\Hk}B_j
			\right)^2\\
			&\ge
			\frac{|J_\delta|}{
				2\left(5+19/(4q)\right)^2n
			}
			\asymp
			\frac{\mathcal N(\delta)}{n}.
		\end{aligned}
		\]
		Since
		\(\norm{qp_1^{-\alpha}\mathbf 1_{\{1\}}}_{\Hk}\le R\), the fixed regression
		function \(f^{(0)}\) belongs to the source class.  Taking expectations in
		the last bound and using the probability lower bound gives
		\[
		\sup_{\norm{g}_{\Hk}\le R}
		\E\norm{\widehat f_\delta-T_k^\alpha g}_{L^2(\mu)}^2
		\gtrsim
		\frac{\mathcal N(\delta)}{n}.
		\]
		
		Finally, for
		\[
		\delta_n
		\asymp
		n^{-\frac{\beta(\alpha+1)}
			{\beta(2\alpha+1)+\alpha+1}},
		\]
		we have
		\(n\delta_n/\log(1/\delta_n)\to\infty\) and
		\(\mathcal N(\delta_n)\asymp\delta_n^{-1/\beta}\).  Hence
		\[
		\frac{\mathcal N(\delta_n)}{n}
		\asymp
		n^{-\frac{\beta(2\alpha+1)}
			{\beta(2\alpha+1)+\alpha+1}},
		\]
		which proves the empirical statement.
	\end{proof}

	\subsection{Proof of Theorem~\ref{thm:ndat-approx}: \nameref*{thm:ndat-approx}}
	
	We first record some auxiliary lemmas.
	
	\begin{lemma}[Spectral bias bound]\label{lem:spectral-bias}
		Let \(f^*=T_k^\alpha g\) with \(g\in\Hk\) and \(0<\alpha\le1\).  For
		\(0<\lambda\le1\),
		\[
		\norm{\lambda(T_k+\lambda I)^{-1}f^*}_{L^2(\mu)}
		\le
		\lambda^{\min\{\alpha+1/2,1\}}\norm{g}_{\Hk}.
		\]
	\end{lemma}
	
	\begin{proof}[Proof of Lemma~\ref{lem:spectral-bias}]
		Write \(g=\sum_j g_j\varphi_j\).  Since \(0<\sigma_j\le1\) and
		\(f^*=T_k^\alpha g=\sum_j\sigma_j^\alpha g_j\varphi_j\),
		\[
		\begin{aligned}
			\norm{\lambda(T_k+\lambda I)^{-1}f^*}_{L^2(\mu)}^2
			&=
			\sum_j
			\left(\frac{\lambda\sigma_j^\alpha}{\sigma_j+\lambda}\right)^2
			g_j^2\\
			&=
			\sum_j
			\left(\frac{\lambda\sigma_j^{\alpha+1/2}}
			{\sigma_j+\lambda}\right)^2
			\frac{g_j^2}{\sigma_j}\\
			&\le
			\left(
			\sup_{0<t\le1}
			\frac{\lambda t^{\alpha+1/2}}{t+\lambda}
			\right)^2
			\norm{g}_{\Hk}^2.
		\end{aligned}
		\]
		For \(0<\alpha<1/2\), splitting at \(t=\lambda\) gives
		\(\lambda t^{\alpha+1/2}/(t+\lambda)\le\lambda^{\alpha+1/2}\).
		For \(\alpha\ge1/2\), \(t^{\alpha+1/2}\le t\), so the same multiplier is
		at most \(\lambda\).  Taking square roots proves the claim.
	\end{proof}
	
	\begin{lemma}[Centered absolute-loss bound]
		\label{lem:absolutelossincrease}
		Under Assumptions~\ref{ass:model} and \ref{ass:noise}, every \(f\in\Hk\)
		satisfies
		\[
		0\le \E|Y-f(X)|-\E|\varepsilon|
		\le
		\kappa\norm{f-f^*}_{L^2(\mu)}^2 .
		\]
	\end{lemma}
	
	\begin{proof}[Proof of Lemma~\ref{lem:absolutelossincrease}]
		By independence of \(X\) and \(\varepsilon\), apply
		Assumption~\ref{ass:noise} conditionally on \(X\), with
		\(t=f(X)-f^*(X)\), and then average over \(X\):
		\[
		0
		\le
		\E|\varepsilon-\left(f(X)-f^*(X)\right)|-\E|\varepsilon|
		\le
		\kappa\E(f(X)-f^*(X))^2 .
		\]
	\end{proof}
	
	\begin{lemma}[Population first-order condition for noise-debiased training]
		\label{lem:ndat-population-foc}
		Under Assumptions~\ref{ass:model} and \ref{ass:noise}, every minimizer
		\(\widetilde f_\delta^*\) of problem~\eqref{prob:populationndat} admits
		\(v_\delta\in\partial\norm{\widetilde f_\delta^*}_{\Hk}\) and a measurable
		function \(\widetilde s_\delta:\mathcal X\times\mathbb R\to[-1,1]\) such
		that
		\[
		\widetilde s_\delta(x,e)
		\in
		\left.\partial_t|f^*(x)+e-t|\right|_{t=\widetilde f_\delta^*(x)}
		\]
		and
		\begin{equation}
			\label{eq:ndat-population-foc}
			\begin{aligned}
				0=
				T_k(\widetilde f_\delta^*-f^*)
				+\delta^2\widetilde f_\delta^*
				+\delta\norm{\widetilde f_\delta^*}_{\Hk}
				T_k\left(
				\E_\varepsilon\left[
				\widetilde s_\delta(\cdot,\varepsilon)
				\right]
				\right) +
				\delta\left(
				\E|Y-\widetilde f_\delta^*(X)|-\E|\varepsilon|
				\right)v_\delta .
			\end{aligned}
		\end{equation}
		Moreover,
		\begin{equation}
			\label{eq:ndat-sign-selection}
			\left|
			\E_\varepsilon[\widetilde s_\delta(x,\varepsilon)]
			\right|
			\le
			4\kappa|\widetilde f_\delta^*(x)-f^*(x)|
			\quad\text{for almost every \(x\) with respect to the distribution of \(X\)}.
		\end{equation}
		In particular, \(\norm{v_\delta}_{\Hk}\le1\), including when
		\(\widetilde f_\delta^*=0\).
	\end{lemma}
	
	\begin{proof}[Proof of Lemma~\ref{lem:ndat-population-foc}]
		The population objective is jointly convex in \((f,r)\).  For fixed
		\(f\), its derivative in \(r\ge\norm{f}_{\Hk}\) is nonnegative by
		Lemma~\ref{lem:absolutelossincrease}; hence its partial minimum is attained
		at \(r=\norm{f}_{\Hk}\).  This also proves convexity of the reduced
		objective.
		
		If \(\widetilde f_\delta^*\ne0\), Fermat's rule and the convex chain rule
		give \eqref{eq:ndat-population-foc}.  Conditioning the selected
		absolute-loss subgradient on \(X=x\) and applying
		Lemma~\ref{lem:quadratic-absolute-risk} give
		\eqref{eq:ndat-sign-selection}.
		
		If \(\widetilde f_\delta^*=0\), nonnegativity of every right directional
		derivative gives
		\[
		\left|\inner{T_kf^*}{h}_{\Hk}\right|
		\le
		\delta\left(\E|Y|-\E|\varepsilon|\right)\norm{h}_{\Hk},
		\qquad h\in\Hk.
		\]
		The quantity in parentheses is nonnegative by
		Lemma~\ref{lem:absolutelossincrease}.  Duality therefore gives
		\(v_\delta\in\partial\norm{0}_{\Hk}\) such that
		\[
		T_kf^*
		=
		\delta\left(\E|Y|-\E|\varepsilon|\right)v_\delta,
		\]
		which is \eqref{eq:ndat-population-foc} at zero.  The measurable selection
		can be chosen as in Lemma~\ref{lem:population-foc}, and
		Lemma~\ref{lem:quadratic-absolute-risk} again yields
		\eqref{eq:ndat-sign-selection}.
	\end{proof}
	
	\begin{proof}[Proof of Theorem~\ref{thm:ndat-approx}]
		Comparison with \(f^*\) and Lemma~\ref{lem:absolutelossincrease} give
		\[
		\begin{aligned}
			\norm{\widetilde f_\delta^*-f^*}_{L^2(\mu)}^2
			+\delta^2\norm{\widetilde f_\delta^*}_{\Hk}^2+
			2\delta\norm{\widetilde f_\delta^*}_{\Hk}
			\left(
			\E|Y-\widetilde f_\delta^*(X)|-\E|\varepsilon|
			\right)
			\le
			\delta^2\norm{f^*}_{\Hk}^2.
		\end{aligned}
		\]
		Since \(\norm{f^*}_{\Hk}\le R\),
		\begin{equation}
			\label{eq:ndat-basic-bounds}
			\norm{\widetilde f_\delta^*}_{\Hk}\le R,
			\qquad
			\norm{\widetilde f_\delta^*-f^*}_{L^2(\mu)}\le\delta R.
		\end{equation}
		
		Rearranging \eqref{eq:ndat-population-foc} gives
		\begin{equation}
			\label{eq:ndat-resolvent}
			\begin{aligned}
				\widetilde f_\delta^*-f^*
				&=
				-\delta^2(T_k+\delta^2I)^{-1}f^*\\
				&\quad-
				\delta\norm{\widetilde f_\delta^*}_{\Hk}
				(T_k+\delta^2I)^{-1}T_k
				\left(
				\E_\varepsilon[
				\widetilde s_\delta(\cdot,\varepsilon)]
				\right)\\
				&\quad-
				\delta\left(
				\E|Y-\widetilde f_\delta^*(X)|-\E|\varepsilon|
				\right)
				(T_k+\delta^2I)^{-1}v_\delta .
			\end{aligned}
		\end{equation}
		The three terms on the right satisfy
		\[
		\norm{\delta^2(T_k+\delta^2I)^{-1}f^*}_{L^2(\mu)}
		\le
		R\delta^{\min\{2\alpha+1,2\}},
		\]
		\[
		\begin{aligned}
			&\delta\norm{\widetilde f_\delta^*}_{\Hk}
			\left\|
			(T_k+\delta^2I)^{-1}T_k
			\E_\varepsilon[
			\widetilde s_\delta(\cdot,\varepsilon)]
			\right\|_{L^2(\mu)}\\
			&\qquad\le
			4\kappa\delta R
			\norm{\widetilde f_\delta^*-f^*}_{L^2(\mu)},
		\end{aligned}
		\]
		and
		\[
		\norm{(T_k+\delta^2I)^{-1}v_\delta}_{L^2(\mu)}
		\le
		\frac1{2\delta}.
		\]
		Here the first two inequalities use Lemma~\ref{lem:spectral-bias},
		\eqref{eq:ndat-sign-selection}, \eqref{eq:ndat-basic-bounds}, and
		\(\norm{(T_k+\delta^2I)^{-1}T_k}_{\mathrm{op}}\le1\).  The last follows from
		\[
		\sup_{t>0}\frac{t}{(t+\delta^2)^2}
		=
		\frac1{4\delta^2}
		\quad\text{and}\quad
		\norm{v_\delta}_{\Hk}\le1.
		\]
		Lemma~\ref{lem:absolutelossincrease} therefore bounds the final term in
		\eqref{eq:ndat-resolvent} by
		\[
		\frac{\kappa}{2}
		\norm{\widetilde f_\delta^*-f^*}_{L^2(\mu)}^2.
		\]
		
		Taking \(L^2(\mu)\)-norms in \eqref{eq:ndat-resolvent} and using
		\eqref{eq:ndat-basic-bounds} yields
		\[
		\norm{\widetilde f_\delta^*-f^*}_{L^2(\mu)}
		\le
		R\delta^{\min\{2\alpha+1,2\}}
		+
		\frac92\kappa R\delta
		\norm{\widetilde f_\delta^*-f^*}_{L^2(\mu)}.
		\]
		If \(8\kappa R\delta\le1\), the final coefficient is at most \(9/16\).
		Rearranging proves
		\[
		\norm{\widetilde f_\delta^*-f^*}_{L^2(\mu)}
		\le
		\frac{16}{7}R\delta^{\min\{2\alpha+1,2\}},
		\]
		uniformly over \(\norm{g}_{\Hk}\le R\).
	\end{proof}
	
	\subsection{Auxiliary concentration result for Theorem~\ref{thm:ndat-generalization}}
	
	\begin{lemma}[Simultaneous concentration inequalities]
		\label{lem:concentration-event}
		Under Assumptions~\ref{ass:model} and \ref{ass:noise-tail}, let
		\(0<\delta\le1\) satisfy
		\(\delta^2\le\norm{T_k}_{\mathrm{op}}\), let \(0<\eta<1\), fix
		\(f,h\in\Hk\), and fix a measurable function
		\(s_f:\mathcal X\times\mathbb R\to[-1,1]\) satisfying
		\[
		s_f(x,e)
		\in
		\left.\partial_t|f^*(x)+e-t|\right|_{t=f(x)}.
		\]
		Set
		\[
		A_{\delta,n,\eta}
		=
		\sqrt{
			\frac{(\mathcal N(\delta^2)+1)\log(64/(\eta\delta^2))}{n}
		}
		+
		\frac{\log(64/(\eta\delta^2))}{n\delta}.
		\]
		
		Then, with probability at least \(1-\eta\), the following inequalities hold simultaneously for every \(u\in\Hk\):
		\begin{enumerate}[label=(\roman*)]
			\item The empirical quadratic form satisfies
			\begin{equation}
				\begin{aligned}\label{eq:empirical-quadratic-form-bound}
					\left|
					\frac1n\sum_{i=1}^n u(X_i)^2-\norm{u}_{L^2(\mu)}^2
					\right|
					&\le
					2\left[
					\sqrt{
						\frac{\log(64/(\eta\delta^2))}{n\delta^2}
					}
					+
					\frac{\log(64/(\eta\delta^2))}{n\delta^2}
					\right]
					\\
					&\quad\times
					\left(
					\norm{u}_{L^2(\mu)}^2
					+
					\delta^2\norm{u}_{\Hk}^2
					\right).
				\end{aligned}
			\end{equation}
			
			\item The empirical noise term satisfies
			\begin{equation}\label{eq:empirical-noise-bound}
				\left|
				\frac1n\sum_{i=1}^n\varepsilon_i u(X_i)
				\right|
				\le
				2\sigma A_{\delta,n,\eta}
				\left(
				\norm{u}_{L^2(\mu)}^2+\delta^2\norm{u}_{\Hk}^2
				\right)^{1/2}.
			\end{equation}
			
			\item The empirical product term satisfies
			\begin{equation}\label{eq:empirical-product-bound}
				\left|
				\frac1n\sum_{i=1}^n h(X_i)u(X_i)-\E[h(X)u(X)]
				\right|
				\le
				4\norm{h}_{\Hk}A_{\delta,n,\eta}
				\left(
				\norm{u}_{L^2(\mu)}^2+\delta^2\norm{u}_{\Hk}^2
				\right)^{1/2}.
			\end{equation}
			
			\item The empirical sign term satisfies
			\begin{equation}\label{eq:empirical-sign-bound}
				\begin{aligned}
					\left|
					\frac1n\sum_{i=1}^n
					s_f(X_i,\varepsilon_i)u(X_i)
					-
					\E\left[
					s_f(X,\varepsilon)u(X)
					\right]
					\right|
					\le
					4A_{\delta,n,\eta}
					\left(
					\norm{u}_{L^2(\mu)}^2+\delta^2\norm{u}_{\Hk}^2
					\right)^{1/2}.
				\end{aligned}
			\end{equation}
			
			\item The empirical absolute loss satisfies
			\begin{equation}
				\left|
				\frac1n\sum_{i=1}^n|Y_i-f(X_i)|
				-
				\E|Y-f(X)|
				\right|
				\le
				16\left(
				\sigma+\norm{f^*}_{\Hk}+\norm{f}_{\Hk}
				\right)A_{\delta,n,\eta}.
			\end{equation}
		\end{enumerate}
	\end{lemma}
	
	\begin{proof}[Proof of Lemma~\ref{lem:concentration-event}]
		Assign failure probability \(\eta/5\) to each of the five bounds.
		
		\proofstep{\textbf{Step 1: empirical quadratic form.}}
		Theorem~\ref{thm:regularized-covariance-embedding}, applied with
		\(\lambda=\delta^2\), gives
		\[
		\left\|
		(T_k+\delta^2I)^{-1/2}(T_n-T_k)(T_k+\delta^2I)^{-1/2}
		\right\|_{\mathrm{op}}
		\le
		2\left[
		\sqrt{\frac{\log(64/(\eta\delta^2))}{n\delta^2}}
		+
		\frac{\log(64/(\eta\delta^2))}{n\delta^2}
		\right].
		\]
		For every \(u\in\Hk\),
		\[
		\begin{aligned}
			\left|
			\frac1n\sum_{i=1}^n u(X_i)^2-\norm{u}_{L^2(\mu)}^2
			\right|
			&=
			\left|\inner{(T_n-T_k)u}{u}_{\Hk}\right|\\
			&\le
			\left\|
			(T_k+\delta^2I)^{-1/2}(T_n-T_k)(T_k+\delta^2I)^{-1/2}
			\right\|_{\mathrm{op}}\\
			&\quad\times
			\norm{(T_k+\delta^2I)^{1/2}u}_{\Hk}^2.
		\end{aligned}
		\]
		Since
		\[
		\norm{(T_k+\delta^2I)^{1/2}u}_{\Hk}^2
		=
		\norm{u}_{L^2(\mu)}^2+\delta^2\norm{u}_{\Hk}^2,
		\]
		this proves \eqref{eq:empirical-quadratic-form-bound}.
		
		\proofstep{\textbf{Step 2: noise term.}}
		Lemma~\ref{lem:regularized-noise-concentration}, with
		\(\lambda=\delta^2\), yields
		\[
		\left\|
		(T_k+\delta^2I)^{-1/2}
		\frac1n\sum_{i=1}^n\varepsilon_i k_{X_i}
		\right\|_{\Hk}
		\le
		2\sigma A_{\delta,n,\eta}.
		\]
		Pairing this vector with \((T_k+\delta^2I)^{1/2}u\) and applying
		Cauchy--Schwarz proves \eqref{eq:empirical-noise-bound}.
		
		\proofstep{\textbf{Step 3: empirical product term.}}
		The independent centered random elements
		\[
		\begin{aligned}
			h(X_i)(T_k+\delta^2I)^{-1/2}k_{X_i}-
			\E\!\left[h(X)(T_k+\delta^2I)^{-1/2}k_X\right]
		\end{aligned}
		\]
		have total second moment at most
		\(n\norm{h}_{\Hk}^2\mathcal N(\delta^2)\).  Moreover,
		\[
		\norm{(T_k+\delta^2I)^{-1/2}k_x}_{\Hk}\le\delta^{-1},
		\qquad x\in\mathcal X,
		\]
		so each centered element has norm at most
		\(2\norm{h}_{\Hk}/\delta\).  Thus, for every integer \(p\ge2\), their
		\(p\)-th moments satisfy the condition of
		Theorem~\ref{thm:hilbert-vector-bernstein} with
		\[
		v=n\norm{h}_{\Hk}^2\mathcal N(\delta^2),
		\qquad
		L=\frac{2\norm{h}_{\Hk}}{\delta}.
		\]
		
		Theorem~\ref{thm:hilbert-vector-bernstein} gives
		\[
		\begin{aligned}
			\Bigg\|
			\frac1n\sum_{i=1}^n
			h(X_i)(T_k+\delta^2I)^{-1/2}k_{X_i}-
			\E\!\left[h(X)(T_k+\delta^2I)^{-1/2}k_X\right]
			\Bigg\|_{\Hk}
			\le
			4\norm{h}_{\Hk}A_{\delta,n,\eta}.
		\end{aligned}
		\]
		Pairing with \((T_k+\delta^2I)^{1/2}u\) proves
		\eqref{eq:empirical-product-bound}.
		
		\proofstep{\textbf{Step 4: empirical sign term.}}
		The same argument applies to
		\[
		\begin{aligned}s_f(X_i,\varepsilon_i)(T_k+\delta^2I)^{-1/2}k_{X_i}-
			\E\!\left[
			s_f(X,\varepsilon)(T_k+\delta^2I)^{-1/2}k_X
			\right].
		\end{aligned}
		\]
		Because \(|s_f|\le1\), the total second moment is at most
		\(n\mathcal N(\delta^2)\), and each centered element has norm at most
		\(2/\delta\).  Hilbert-valued Bernstein therefore bounds the norm of their
		average by \(4A_{\delta,n,\eta}\).  Pairing with
		\((T_k+\delta^2I)^{1/2}u\) proves \eqref{eq:empirical-sign-bound}.
		
		\proofstep{\textbf{Step 5: empirical absolute loss.}}
		For every integer \(p\ge2\), the centering inequality gives
		\[
		\E\left|
		|Y-f(X)|-\E|Y-f(X)|
		\right|^p
		\le2^p\E|Y-f(X)|^p.
		\]
		Assumption~\ref{ass:noise-tail} and
		\[
		|Y-f(X)|
		\le
		|\varepsilon|+\norm{f^*}_{\Hk}+\norm{f}_{\Hk}
		\]
		give
		\[
		\E\left|
		|Y-f(X)|-\E|Y-f(X)|
		\right|^p
		\le
		\frac{p!}{2}
		\left[
		8\left(\sigma+\norm{f^*}_{\Hk}+\norm{f}_{\Hk}\right)
		\right]^p.
		\]
		Theorem~\ref{thm:hilbert-vector-bernstein}, applied in \(\mathbb R\), gives
		the fifth bound.  A union bound completes the proof.
	\end{proof}
	
	\subsection{Proof of Theorem~\ref{thm:ndat-generalization}: \nameref*{thm:ndat-generalization}}
\textbf{Proof strategy.}  We compare the empirical predictor \(\widetilde f_\delta\) with its population
counterpart \(\widetilde f_\delta^*\) in three steps.
\begin{enumerate}[label=(\roman*)]
	\item The constraint \(r\ge\norm{f}_{\Hk}\) may be strict.  We therefore
	represent the unused part of \(r\) by an additional scalar coordinate.  This
	turns the constrained problem into a minimization problem on
	\(\Hk\oplus\mathbb R\) without changing the fitted function.

	\item We compare the empirical and population optimality conditions.  Their
	difference produces a quadratic error that measures both the prediction
	difference and the distance between the two minimizers.  Concentration controls
	the random fluctuations, while convexity and subgradient monotonicity control
	the nonlinear terms introduced by the norm and the absolute loss.

	\item The sample-size conditions make the resulting remainder small enough to
	be absorbed into the quadratic error.  This gives the estimation error between
	\(\widetilde f_\delta\) and \(\widetilde f_\delta^*\), with separate
	contributions from sampling variation and the error in \(\widetilde m\).
	Combining this estimate with the population approximation bound gives the
	prediction error relative to \(f^*\).
\end{enumerate}
	
	\begin{proof}
		Work on \(\overline{\mathcal H}=\Hk\oplus\R\), with
		\[
		\norm{(f,t)}_{\overline{\mathcal H}}^2
		=
		\norm{f}_{\Hk}^2+t^2,
		\qquad
		(f,t)(x)=f(x).
		\]
		
		We continue to write \(T_k\), \(T_n\), and \(k_x\) for their extensions
		\(T_k\oplus0\), \(T_n\oplus0\), and \((k_x,0)\).  Their effective dimension
		is still \(\mathcal N(\delta^2)\).
		
		Consider
		\begin{equation}
			\label{eq:ndat-auxiliary-objective}
			\frac1n\sum_{i=1}^n
			\left(
			|Y_i-f(X_i)|+\delta\norm{(f,t)}_{\overline{\mathcal H}}
			\right)^2
			-
			2\delta\widetilde m\norm{(f,t)}_{\overline{\mathcal H}}.
		\end{equation}

		For every feasible pair \((f,r)\), the choice
		\(t=(r^2-\norm{f}_{\Hk}^2)^{1/2}\) satisfies
		\(\norm{(f,t)}_{\overline{\mathcal H}}=r\).  Thus, let
		\((\widetilde f_\delta,\widetilde r_\delta)\) be any minimizing pair
		in problem~\eqref{eq:ndat-estimator}, and set
		\[
		\widehat w_\delta
		=
		\left(
		\widetilde f_\delta,
		\sqrt{\widetilde r_\delta^2-\norm{\widetilde f_\delta}_{\Hk}^2}
		\right).
		\]
		The constraint in problem~\eqref{eq:ndat-estimator} makes this element
		well defined, and the objective value in
		\eqref{eq:ndat-auxiliary-objective} equals that of the minimizing pair.

		To verify global optimality on \(\overline{\mathcal H}\), let \(P_n\)
		denote the orthogonal projection in \(\Hk\) onto \(S_n\).  For every
		\((h,s)\in\overline{\mathcal H}\),
		\[
		\left(P_nh,\norm{(h,s)}_{\overline{\mathcal H}}\right)
		\]
		is feasible for problem~\eqref{eq:ndat-estimator}.  By the reproducing
		property, \(P_nh(X_i)=h(X_i)\) for every \(i\), so this pair has the same objective
		value as \((h,s)\) in \eqref{eq:ndat-auxiliary-objective}.  Hence
		\(\widehat w_\delta\) is a global minimizer of
		\eqref{eq:ndat-auxiliary-objective}.

		For the population comparison, replace \(\widetilde m\) by
		\(\E|\varepsilon|\).  After removing the constant
		\(\E[\varepsilon^2]\), it equals
		\[
		\norm{f-f^*}_{L^2(\mu)}^2
		+\delta^2\norm{(f,t)}_{\overline{\mathcal H}}^2
		+2\delta\norm{(f,t)}_{\overline{\mathcal H}}
		\left(\E|Y-f(X)|-\E|\varepsilon|\right).
		\]
		Lemma~\ref{lem:absolutelossincrease} shows that, for fixed \(f\), this
		expression is nondecreasing in
		\(\norm{(f,t)}_{\overline{\mathcal H}}\).  It is therefore minimized at
		\(t=0\), with population minimizer
		\[
		w_\delta^*=(\widetilde f_\delta^*,0).
		\]
		
		We next specify the event on which the empirical and population quantities
		are compared.  The bounded-kernel assumption gives
		\(\mathcal N(\delta^2)\le\delta^{-2}\).  Choosing the constant \(C\) in the
		theorem sufficiently large, the sample-size condition implies
		\[
		2\left[
		\sqrt{\frac{\log(64/(\eta\delta^2))}{n\delta^2}}
		+
		\frac{\log(64/(\eta\delta^2))}{n\delta^2}
		\right]
		\le\frac18.
		\]
		
		The proof of Lemma~\ref{lem:concentration-event} remains valid on
		\(\overline{\mathcal H}\): the feature vector is \((k_x,0)\), the
		covariance operator is \(T_k\oplus0\), and the effective dimension is
		unchanged.  Apply that lemma with
		\(f=\widetilde f_\delta^*\),
		\(h=f^*-\widetilde f_\delta^*\), and the selection from
		Lemma~\ref{lem:ndat-population-foc}.  The resulting event depends only on
		the fitting sample and has probability at least \(1-\eta\).  The same choice
		of \(C\) makes \(A_{\delta,n,\eta}\) small
		enough for the absorption in \eqref{eq:ndat-absorption-bound}.
		
		We now derive the identity that compares the empirical and population
		minimizers.  If \(\widehat w_\delta\ne0\), the convex part of
		\eqref{eq:ndat-auxiliary-objective} has a convex subdifferential, while
		\(-2\delta\widetilde m\norm{\cdot}_{\overline{\mathcal H}}\) is differentiable
		at \(\widehat w_\delta\).  Fermat's rule therefore gives
		\(\widehat v_\delta\in
		\partial\norm{\widehat w_\delta}_{\overline{\mathcal H}}\) and
		\[
		\widehat s_i
		\in
		\left.\partial_t|Y_i-t|\right|_{t=\widetilde f_\delta(X_i)}
		\]
		such that
		\begin{equation}
			\label{eq:ndat-empirical-foc}
			\begin{aligned}
				(T_n+\delta^2I)\widehat w_\delta
				=
				\frac1n\sum_{i=1}^nY_i k_{X_i}-
				\delta\left(
				\frac1n\sum_{i=1}^n|Y_i-\widetilde f_\delta(X_i)|-\widetilde m
				\right)\widehat v_\delta-
				\delta\norm{\widehat w_\delta}_{\overline{\mathcal H}}
				\frac1n\sum_{i=1}^n\widehat s_i k_{X_i}.
			\end{aligned}
		\end{equation}
		If \(\widehat w_\delta=0\), choose any
		\(\widehat s_i\in\left.\partial_t|Y_i-t|\right|_{t=0}\).
		Optimality in the direction \((0,1)\) first shows that
		\(n^{-1}\sum_i|Y_i|-\widetilde m\ge0\).  Optimality in directions
		\((h,0)\) and \((-h,0)\), followed by RKHS duality, then gives
		\(\widehat v_\delta\in\partial\norm{0}_{\overline{\mathcal H}}\) for
		which \eqref{eq:ndat-empirical-foc} also holds.
		
		Embed \(v_\delta\) from Lemma~\ref{lem:ndat-population-foc} into
		\(\overline{\mathcal H}\).  The corresponding population identity is
		\begin{equation}
			\label{eq:ndat-population-foc-rearranged}
			\begin{aligned}
				(T_k+\delta^2I)w_\delta^*
				=
				T_kf^*
				-
				\delta\left(
				\E|Y-\widetilde f_\delta^*(X)|-\E|\varepsilon|
				\right)v_\delta -
				\delta\norm{w_\delta^*}_{\overline{\mathcal H}}
				T_k\left(
				\E_\varepsilon[
				\widetilde s_\delta(\cdot,\varepsilon)]
				\right).
			\end{aligned}
		\end{equation}
		
		Subtracting \eqref{eq:ndat-population-foc-rearranged} from
		\eqref{eq:ndat-empirical-foc} and pairing with
		\(\widehat w_\delta-w_\delta^*\) gives
		\begin{equation}
			\label{eq:ndat-error-identity}
			\begin{aligned}
				&\frac1n\sum_{i=1}^n
				(\widetilde f_\delta(X_i)-\widetilde f_\delta^*(X_i))^2
				+
				\delta^2\norm{\widehat w_\delta-w_\delta^*}_{\overline{\mathcal H}}^2\\
				&=
				\inner{\frac1n\sum_{i=1}^n\varepsilon_i k_{X_i}}
				{\widehat w_\delta-w_\delta^*}_{\overline{\mathcal H}}
				+
				\inner{(T_n-T_k)(f^*-\widetilde f_\delta^*)}
				{\widehat w_\delta-w_\delta^*}_{\overline{\mathcal H}}\\
				&\quad-
				\delta\left(
				\frac1n\sum_{i=1}^n|Y_i-\widetilde f_\delta(X_i)|-\widetilde m
				\right)
				\inner{\widehat v_\delta}
				{\widehat w_\delta-w_\delta^*}_{\overline{\mathcal H}}\\
				&\quad+
				\delta\left(
				\E|Y-\widetilde f_\delta^*(X)|-\E|\varepsilon|
				\right)
				\inner{v_\delta}
				{\widehat w_\delta-w_\delta^*}_{\overline{\mathcal H}}\\
				&\quad-
				\delta\Bigg\langle
				\norm{\widehat w_\delta}_{\overline{\mathcal H}}
				\frac1n\sum_{i=1}^n\widehat s_i k_{X_i}
				-
				\norm{w_\delta^*}_{\overline{\mathcal H}}
				T_k\E_\varepsilon[
				\widetilde s_\delta(\cdot,\varepsilon)],
				\widehat w_\delta-w_\delta^*
				\Bigg\rangle_{\overline{\mathcal H}}.
			\end{aligned}
		\end{equation}
		
		The left-hand side of \eqref{eq:ndat-error-identity} is the empirical squared
		prediction difference plus the squared product-space difference weighted by
		\(\delta^2\).  The right-hand side consists of two centered linear terms, two
		norm-subgradient terms, and one absolute-loss subgradient term.
		
		We now bound the terms on the right-hand side of
		\eqref{eq:ndat-error-identity}.  The quadratic-form bound in
		Lemma~\ref{lem:concentration-event} gives
		\begin{equation}
			\label{eq:ndat-norm-comparison}
			\begin{aligned}
				\norm{\widetilde f_\delta-\widetilde f_\delta^*}_{L^2(\mu)}^2
				+
				\delta^2\norm{\widehat w_\delta-w_\delta^*}_{\overline{\mathcal H}}^2 \le
				\frac87\left[
				\frac1n\sum_{i=1}^n
				(\widetilde f_\delta(X_i)-\widetilde f_\delta^*(X_i))^2
				+
				\delta^2\norm{\widehat w_\delta-w_\delta^*}_{\overline{\mathcal H}}^2
				\right].
			\end{aligned}
		\end{equation}
		
		The noise and centered-product bounds in Lemma~\ref{lem:concentration-event} then imply
		\[
		\begin{aligned}
			&\left|
			\inner{\frac1n\sum_{i=1}^n\varepsilon_i k_{X_i}}
			{\widehat w_\delta-w_\delta^*}_{\overline{\mathcal H}}
			\right|
			+
			\left|
			\inner{(T_n-T_k)(f^*-\widetilde f_\delta^*)}
			{\widehat w_\delta-w_\delta^*}_{\overline{\mathcal H}}
			\right|\\
			&\qquad\le
			C A_{\delta,n,\eta}
			\left[
			\frac1n\sum_{i=1}^n
			(\widetilde f_\delta(X_i)-\widetilde f_\delta^*(X_i))^2
			+
			\delta^2\norm{\widehat w_\delta-w_\delta^*}_{\overline{\mathcal H}}^2
			\right]^{1/2}.
		\end{aligned}
		\]
		
		The two norm-subgradient terms in
		\eqref{eq:ndat-error-identity} have the exact decomposition
		\[
		\begin{aligned}
		&-\delta\left(
		\frac1n\sum_{i=1}^n|Y_i-\widetilde f_\delta(X_i)|-\widetilde m
		\right)
		\inner{\widehat v_\delta}
		{\widehat w_\delta-w_\delta^*}_{\overline{\mathcal H}}\\
		&\quad+
		\delta\left(
		\E|Y-\widetilde f_\delta^*(X)|-\E|\varepsilon|
		\right)
		\inner{v_\delta}
		{\widehat w_\delta-w_\delta^*}_{\overline{\mathcal H}}\\
		&=
		-\delta\left(
		\frac1n\sum_{i=1}^n|Y_i-\widetilde f_\delta(X_i)|
		-
		\frac1n\sum_{i=1}^n|Y_i-\widetilde f_\delta^*(X_i)|
		\right)
		\inner{\widehat v_\delta}
		{\widehat w_\delta-w_\delta^*}_{\overline{\mathcal H}}\\
		&\quad-
		\delta\left[
		\frac1n\sum_{i=1}^n|Y_i-\widetilde f_\delta^*(X_i)|
		-\widetilde m
		-\E|Y-\widetilde f_\delta^*(X)|
		+\E|\varepsilon|
		\right]
		\inner{\widehat v_\delta}
		{\widehat w_\delta-w_\delta^*}_{\overline{\mathcal H}}\\
		&\quad-
		\delta\left(
		\E|Y-\widetilde f_\delta^*(X)|-\E|\varepsilon|
		\right)
		\inner{\widehat v_\delta-v_\delta}
		{\widehat w_\delta-w_\delta^*}_{\overline{\mathcal H}}
		\end{aligned}
		\]
		The last line is nonpositive because
		Lemma~\ref{lem:absolutelossincrease} makes its scalar coefficient
		nonnegative and monotonicity of the norm subdifferential gives
		\[
		\inner{\widehat v_\delta-v_\delta}
		{\widehat w_\delta-w_\delta^*}_{\overline{\mathcal H}}\ge0.
		\]
		The Lipschitz
		property of the empirical absolute loss implies
		\[
		\begin{aligned}
			&\delta\left|
			\frac1n\sum_{i=1}^n|Y_i-\widetilde f_\delta(X_i)|
			-
			\frac1n\sum_{i=1}^n|Y_i-\widetilde f_\delta^*(X_i)|
			\right|
			\norm{\widehat w_\delta-w_\delta^*}_{\overline{\mathcal H}}\\
			&\quad\le
			\frac12\left[
			\frac1n\sum_{i=1}^n
			(\widetilde f_\delta(X_i)-\widetilde f_\delta^*(X_i))^2
			+
			\delta^2\norm{\widehat w_\delta-w_\delta^*}_{\overline{\mathcal H}}^2
			\right].
		\end{aligned}
		\]
		
		The fixed-function absolute-loss bound and
		\eqref{eq:ndat-norm-comparison} control the remaining component by
		\[
		C\left(
		A_{\delta,n,\eta}
		+
		|\widetilde m-\E|\varepsilon||
		\right)
		\left[
		\frac1n\sum_{i=1}^n
		(\widetilde f_\delta(X_i)-\widetilde f_\delta^*(X_i))^2
		+
		\delta^2\norm{\widehat w_\delta-w_\delta^*}_{\overline{\mathcal H}}^2
		\right]^{1/2}.
		\]
		These two inequalities control the norm-subgradient terms in
		\eqref{eq:ndat-error-identity}.
		
		The sign term in \eqref{eq:ndat-error-identity} equals
		\[
		\begin{aligned}
		&-\delta\norm{\widehat w_\delta}_{\overline{\mathcal H}}
		\frac1n\sum_{i=1}^n
		\left(\widehat s_i-\widetilde s_\delta(X_i,\varepsilon_i)\right)
		\left(\widetilde f_\delta(X_i)-\widetilde f_\delta^*(X_i)\right)\\
		&\quad-
		\delta\left(
		\norm{\widehat w_\delta}_{\overline{\mathcal H}}
		-\norm{w_\delta^*}_{\overline{\mathcal H}}
		\right)
		\frac1n\sum_{i=1}^n
		\widetilde s_\delta(X_i,\varepsilon_i)
		\left(\widetilde f_\delta(X_i)-\widetilde f_\delta^*(X_i)\right)\\
		&\quad-
		\delta\norm{w_\delta^*}_{\overline{\mathcal H}}
		\Bigg[
		\frac1n\sum_{i=1}^n
		\widetilde s_\delta(X_i,\varepsilon_i)
		\left(\widetilde f_\delta(X_i)-\widetilde f_\delta^*(X_i)\right)\\
		&\hspace{48mm}-
		\E\left[
		\widetilde s_\delta(X,\varepsilon)
		\left(\widetilde f_\delta(X)-\widetilde f_\delta^*(X)\right)
		\right]
		\Bigg].
		\end{aligned}
		\]
		The first line is nonpositive by monotonicity of the absolute-loss
		subdifferential.  The remaining two lines are bounded using
		\eqref{eq:empirical-sign-bound},
		\eqref{eq:ndat-sign-selection}, and
		Theorem~\ref{thm:ndat-approx}.  Together with
		\eqref{eq:ndat-norm-comparison}, they give
		\begin{equation}
			\label{eq:ndat-sign-bound}
			\begin{aligned}
				&-\delta\Bigg\langle
				\norm{\widehat w_\delta}_{\overline{\mathcal H}}
				\frac1n\sum_{i=1}^n\widehat s_i k_{X_i}
				-
				\norm{w_\delta^*}_{\overline{\mathcal H}}
				T_k\E_\varepsilon[
				\widetilde s_\delta(\cdot,\varepsilon)],
				\widehat w_\delta-w_\delta^*
				\Bigg\rangle_{\overline{\mathcal H}}\\
				&\quad\le
				C\left(
				A_{\delta,n,\eta}
				+
				\kappa R\delta^{\min\{2\alpha+1,2\}}
				\right)
				\left[
				\frac1n\sum_{i=1}^n
				(\widetilde f_\delta(X_i)-\widetilde f_\delta^*(X_i))^2
				+
				\delta^2\norm{\widehat w_\delta-w_\delta^*}_{\overline{\mathcal H}}^2
				\right]\\
				&\qquad+
				CR A_{\delta,n,\eta}
				\left[
				\frac1n\sum_{i=1}^n
				(\widetilde f_\delta(X_i)-\widetilde f_\delta^*(X_i))^2
				+
				\delta^2\norm{\widehat w_\delta-w_\delta^*}_{\overline{\mathcal H}}^2
				\right]^{1/2}.
			\end{aligned}
		\end{equation}
		The first term on the right side of \eqref{eq:ndat-sign-bound} follows from
		\[
		\begin{aligned}
			\left|
			\frac1n\sum_{i=1}^n
			\widetilde s_\delta(X_i,\varepsilon_i)
			(\widetilde f_\delta(X_i)-\widetilde f_\delta^*(X_i))
			\right|
			&\le
			\left(
			4A_{\delta,n,\eta}
			+
			4\kappa\norm{\widetilde f_\delta^*-f^*}_{L^2(\mu)}
			\right)\\
			&\quad\times
			\left(
			\norm{\widetilde f_\delta-\widetilde f_\delta^*}_{L^2(\mu)}^2
			+
			\delta^2\norm{\widehat w_\delta-w_\delta^*}_{\overline{\mathcal H}}^2
			\right)^{1/2},
		\end{aligned}
		\]
		and from
		\[
		\delta\left|
		\norm{\widehat w_\delta}_{\overline{\mathcal H}}
		-
		\norm{w_\delta^*}_{\overline{\mathcal H}}
		\right|
		\le
		\left(
		\norm{\widetilde f_\delta-\widetilde f_\delta^*}_{L^2(\mu)}^2
		+
		\delta^2\norm{\widehat w_\delta-w_\delta^*}_{\overline{\mathcal H}}^2
		\right)^{1/2}.
		\]
		The second term in \eqref{eq:ndat-sign-bound} uses
		\(\norm{w_\delta^*}_{\overline{\mathcal H}}\le R\).
		
		Substituting these bounds into
		\eqref{eq:ndat-error-identity} gives
		\begin{equation}
			\begin{aligned}\label{eq:ndat-absorption-bound}
				&\frac1n\sum_{i=1}^n
				(\widetilde f_\delta(X_i)-\widetilde f_\delta^*(X_i))^2
				+
				\delta^2\norm{\widehat w_\delta-w_\delta^*}_{\overline{\mathcal H}}^2\\
				&\quad\le
				\left[
				\frac12+
				C A_{\delta,n,\eta}
				+
				C\kappa R\delta^{\min\{2\alpha+1,2\}}
				\right]
				\left[
				\frac1n\sum_{i=1}^n
				(\widetilde f_\delta(X_i)-\widetilde f_\delta^*(X_i))^2
				+
				\delta^2\norm{\widehat w_\delta-w_\delta^*}_{\overline{\mathcal H}}^2
				\right]\\
				&\qquad+
				C\left(
				A_{\delta,n,\eta}
				+
				|\widetilde m-\E|\varepsilon||
				\right)
				\left[
				\frac1n\sum_{i=1}^n
				(\widetilde f_\delta(X_i)-\widetilde f_\delta^*(X_i))^2
				+
				\delta^2\norm{\widehat w_\delta-w_\delta^*}_{\overline{\mathcal H}}^2
				\right]^{1/2}.
			\end{aligned}
		\end{equation}
		
		The sample-size and radius conditions make the coefficient in square
		brackets in \eqref{eq:ndat-absorption-bound} at most \(3/4\).  Moving this
		term to the left gives
		\[
		\begin{aligned}
		&\frac14\left[
		\frac1n\sum_{i=1}^n
		(\widetilde f_\delta(X_i)-\widetilde f_\delta^*(X_i))^2
		+
		\delta^2\norm{\widehat w_\delta-w_\delta^*}_{\overline{\mathcal H}}^2
		\right]
		\\
		&\qquad\le
		C\left(
		A_{\delta,n,\eta}
		+
		|\widetilde m-\E|\varepsilon||
		\right)
		\left[
		\frac1n\sum_{i=1}^n
		(\widetilde f_\delta(X_i)-\widetilde f_\delta^*(X_i))^2
		+
		\delta^2\norm{\widehat w_\delta-w_\delta^*}_{\overline{\mathcal H}}^2
		\right]^{1/2}.
		\end{aligned}
		\]
		If the square root is zero, the estimation bound is immediate.  Otherwise,
		divide by it and apply \eqref{eq:ndat-norm-comparison} to obtain
		\begin{equation}
			\label{eq:ndat-auxiliary-estimation}
			\norm{\widetilde f_\delta-\widetilde f_\delta^*}_{L^2(\mu)}
			\lesssim
			A_{\delta,n,\eta}
			+
			\left|\widetilde m-\E|\varepsilon|\right|.
		\end{equation}
		
		Since \(\delta^2\le\norm{T_k}_{\mathrm{op}}\),
		\(\mathcal N(\delta^2)\ge1/2\), and hence
		\(\mathcal N(\delta^2)+1\lesssim\mathcal N(\delta^2)\).  The sample-size
		condition also gives
		\[
		\frac{\log(64/(\eta\delta^2))}{n\delta}
		\lesssim
		\sqrt{
			\frac{\mathcal N(\delta^2)
				\log(64/(\eta\delta^2))}{n}
		}.
		\]
		Thus
		\[
		A_{\delta,n,\eta}
		\lesssim
		\sqrt{
			\frac{\mathcal N(\delta^2)
				\log(2/(\eta\delta^2))}{n}
		}.
		\]
		
		Combining this estimate with \eqref{eq:ndat-auxiliary-estimation} gives
		\[
		\norm{\widetilde f_\delta-\widetilde f_\delta^*}_{L^2(\mu)}
		\lesssim
		\sqrt{
			\frac{\mathcal N(\delta^2)\log(2/(\eta\delta^2))}{n}
		}
		+
		|\widetilde m-\E|\varepsilon||.
		\]
		The triangle inequality and Theorem~\ref{thm:ndat-approx} prove the
		prediction bound.
	\end{proof}
	\subsection{Proof of Corollary~\ref{cor:ndat-two-stage-rate}: \nameref*{cor:ndat-two-stage-rate}}
	\begin{proof}
		\proofstep{\textbf{Step 1: preliminary KRR estimator.}}
		Use the finite-sample inequality in
		\citet{caponnetto2007optimal}. In the notation of that result, the source
		and eigenvalue exponents are \(2\alpha+1\) and \(\beta\), respectively.
		Assumptions~\ref{ass:model} and \ref{ass:noise-tail}, together with the
		independence of \(X\) and \(\varepsilon\), imply its conditional Bernstein
		moment condition.  With confidence parameter
		\(\eta=n^{-q}/2\) and the value of \(\lambda_a\) in the corollary, the
		finite-sample requirement reduces to
		\[
		N_a^{\frac{2\beta\alpha}{\beta(2\alpha+1)+1}}
		\ge C\log^2\left(\frac6\eta\right).
		\]
		This holds for all sufficiently large \(n\) because \(N_a\asymp n\).
		The cited inequality then gives
		\begin{equation}
			\label{eq:preliminary-krr-prediction-rate}
			\norm{\widehat f^{\rm KRR}-f^*}_{L^2(\mu)}
			\lesssim
			N_a^{-\frac{\beta(\alpha+1/2)}
				{2\beta(\alpha+1/2)+1}}
			\log\left(\frac6\eta\right)
		\end{equation}
		with probability at least \(1-\eta\), uniformly over
		\(\norm{g}_{\Hk}\le R\).  Since \(\eta=n^{-q}/2\), the logarithm is of
		order \(\log n\).
		
		Write \(Y_j^a=f^*(X_j^a)+\varepsilon_j^a\).
		Equation~\eqref{eq:noise-exponential-tail} and a union bound give
		\[
		\max_{1\le j\le N_a}|\varepsilon_j^a|
		\lesssim\log n
		\]
		with probability at least \(1-n^{-q}/2\).  Comparing the KRR objective at
		\(\widehat f^{\rm KRR}\) and \(f^*\) gives
		\[
		\lambda_a\norm{\widehat f^{\rm KRR}}_{\Hk}^2
		\le
		\frac1{N_a}\sum_{j=1}^{N_a}(\varepsilon_j^a)^2
		+\lambda_a\norm{f^*}_{\Hk}^2,
		\]
		and therefore
		\begin{equation}
			\label{eq:preliminary-krr-norm-rate}
			\norm{\widehat f^{\rm KRR}}_{\Hk}
			\lesssim
			\lambda_a^{-1/2}\log n.
		\end{equation}
		Thus \eqref{eq:preliminary-krr-prediction-rate} and
		\eqref{eq:preliminary-krr-norm-rate} hold simultaneously with probability
		at least \(1-n^{-q}\).
		
		\proofstep{\textbf{Step 2: noise estimation.}}
		Condition on the data used to compute \(\widehat f^{\rm KRR}\).  For an
		observation \((X,Y)\) in the independent sample used to compute
		\(\widetilde m\),
		\[
		Y-\widehat f^{\rm KRR}(X)
		=
		\varepsilon+(f^*-\widehat f^{\rm KRR})(X).
		\]
		Lemma~\ref{lem:absolutelossincrease} gives
		\[
		\begin{aligned}
			0
			&\le
			\E|Y-\widehat f^{\rm KRR}(X)|-\E|\varepsilon|\\
			&\le
			\kappa\norm{\widehat f^{\rm KRR}-f^*}_{L^2(\mu)}^2.
		\end{aligned}
		\]
		Moreover,
		\[
		\left|
		|\varepsilon+(f^*-\widehat f^{\rm KRR})(X)|
		-
		|(f^*-\widehat f^{\rm KRR})(X)|
		\right|
		\le|\varepsilon|,
		\]
		while \(|(f^*-\widehat f^{\rm KRR})(X)|\) is bounded by
		\(\norm{f^*}_{\Hk}+\norm{\widehat f^{\rm KRR}}_{\Hk}\) and has variance at
		most \(\norm{\widehat f^{\rm KRR}-f^*}_{L^2(\mu)}^2\).
		Apply scalar Bernstein separately, after centering, to
		\[
		|\varepsilon+(f^*-\widehat f^{\rm KRR})(X)|
		-|(f^*-\widehat f^{\rm KRR})(X)|
		\]
		and to
		\(|(f^*-\widehat f^{\rm KRR})(X)|\).
		For each fixed preliminary KRR estimator, scalar Bernstein gives, with probability at
		least \(1-n^{-q}\) over this independent sample,
		\[
		\begin{aligned}
			&\left|
			\widetilde m-\E|Y-\widehat f^{\rm KRR}(X)|
			\right|\\
			&\quad\lesssim
			\left(
			1+\norm{\widehat f^{\rm KRR}-f^*}_{L^2(\mu)}
			\right)
			\sqrt{\frac{\log n}{N_e}}
			+
			\left(
			1+\norm{\widehat f^{\rm KRR}}_{\Hk}
			\right)
			\frac{\log n}{N_e}.
		\end{aligned}
		\]
		Substitution of
		\eqref{eq:preliminary-krr-prediction-rate}--\eqref{eq:preliminary-krr-norm-rate},
		together with \(N_a\asymp N_e\asymp n\), gives
		\[
		\begin{aligned}
			\left|\widetilde m-\E|\varepsilon|\right|
			\lesssim{}&
			n^{-\frac{2\beta(\alpha+1/2)}
				{2\beta(\alpha+1/2)+1}}(\log n)^2
			+
			n^{-1/2}\sqrt{\log n}\\
			&+
			n^{-1+\frac{\beta}
				{2\{2\beta(\alpha+1/2)+1\}}}(\log n)^2.
		\end{aligned}
		\]
		The first exponent is twice the desired exponent, the desired exponent is
		strictly smaller than \(1/2\), and the gap between the third exponent and
		the desired exponent is
		\[
		\frac{\beta\alpha+1}{2\beta(\alpha+1/2)+1}>0.
		\]
		Consequently, with probability at least \(1-2n^{-q}\),
		\begin{equation}
			\label{eq:absolute-noise-two-stage-rate}
			\left|\widetilde m-\E|\varepsilon|\right|
			\lesssim
			n^{-\frac{\beta(\alpha+1/2)}
				{2\beta(\alpha+1/2)+1}}
			\sqrt{\log n}.
		\end{equation}
		
		\proofstep{\textbf{Step 3: final estimator.}}
		Take \(\eta=n^{-q}\) and \(\delta=\delta_n\).
		Then
		\[
		\kappa R\delta_n\to0,
		\qquad
		\frac{n\delta_n^2}
		{\log(2/(\eta\delta_n^2))}
		\to\infty,
		\]
		so the conditions of Theorem~\ref{thm:ndat-generalization} hold for all
		sufficiently large \(n\).  Assumption~\ref{ass:eigendecay} gives
		\(\mathcal N(\delta_n^2)\lesssim\delta_n^{-2/\beta}\), and hence
		\[
		\delta_n^{2\alpha+1}
		=
		n^{-\frac{\beta(\alpha+1/2)}
			{2\beta(\alpha+1/2)+1}},
		\]
		and
		\[
		\sqrt{
			\frac{\mathcal N(\delta_n^2)
				\log(2/(\eta\delta_n^2))}{n}
		}
		\lesssim
		n^{-\frac{\beta(\alpha+1/2)}
			{2\beta(\alpha+1/2)+1}}
		\sqrt{\log n}.
		\]
		Combining these bounds with
		\eqref{eq:absolute-noise-two-stage-rate} and taking a union bound proves the
		stated rate with probability at least \(1-3n^{-q}\).
	\end{proof}
	
	\subsection{Proof of Proposition~\ref{prop:ndat-robustness-cost}: \nameref*{prop:ndat-robustness-cost}}
	
	\begin{proof}
		Optimality of \(\widehat f_\delta\) for problem~\eqref{problem} gives
		\[
		\begin{aligned}
			&\frac1n\sum_{i=1}^n
			\left(
			|Y_i-\widehat f_\delta(X_i)|
			+\delta\norm{\widehat f_\delta}_{\Hk}
			\right)^2
			\\
			&\qquad\le
			\frac1n\sum_{i=1}^n
			\left(
			|Y_i-\widetilde f_\delta(X_i)|
			+\delta\norm{\widetilde f_\delta}_{\Hk}
			\right)^2.
		\end{aligned}
		\]
		This proves the first inequality in
		Proposition~\ref{prop:ndat-robustness-cost}.
		
		The pair
		\((\widehat f_\delta,\norm{\widehat f_\delta}_{\Hk})\) is feasible in
		problem~\eqref{eq:ndat-estimator}.  Optimality of
		\((\widetilde f_\delta,\widetilde r_\delta)\) therefore gives
		\[
		\begin{aligned}
			&\frac1n\sum_{i=1}^n
			\left(
			|Y_i-\widetilde f_\delta(X_i)|
			+\delta\widetilde r_\delta
			\right)^2
			-2\delta\widetilde m\widetilde r_\delta
			\\
			&\qquad\le
			\frac1n\sum_{i=1}^n
			\left(
			|Y_i-\widehat f_\delta(X_i)|
			+\delta\norm{\widehat f_\delta}_{\Hk}
			\right)^2
			-2\delta\widetilde m\norm{\widehat f_\delta}_{\Hk}.
		\end{aligned}
		\]
		Because
		\(\norm{\widetilde f_\delta}_{\Hk}\le\widetilde r_\delta\), replacing
		\(\widetilde r_\delta\) by
		\(\norm{\widetilde f_\delta}_{\Hk}\) can only decrease the squared term.
		Rearranging proves the upper bound in the proposition.  Together with the
		first inequality, it also shows that the upper bound is nonnegative.
	\end{proof}
	
	\subsection{Proof of Proposition~\ref{prop:ndat-convex-reformulation}: \nameref*{prop:ndat-convex-reformulation}}
	
	\begin{proof}
		The feasible set
		\(\{(a,r):r\ge0,\ a^\top Ka\le r^2\}\) is a second-order-cone set and is
		therefore convex.  For each \(i\), the map
		\((a,r)\mapsto |Y_i-(Ka)_i|+\delta r\) is nonnegative and convex on this
		set, so its square is convex.  The remaining term
		\(-2\delta\widetilde m r\) is linear.  Thus
		problem~\eqref{eq:ndat-tractable} is a convex program.
		
		The convex problem has a minimizer.  Components of \(a\) in
		\(\ker(K)\) change neither \(Ka\) nor \(a^\top Ka\), so \(a\) may be
		restricted to \(\operatorname{range}(K)\).  On this finite-dimensional
		subspace, the feasible set is closed, and the objective is coercive because
		it is bounded below by
		\(\delta^2r^2-2\delta\widetilde m r\), while bounded \(r\) also bounds
		\(a\) on \(\operatorname{range}(K)\).
		
		Let \((f,r)\) be feasible in problem~\eqref{eq:ndat-estimator}.  Since
		\(f\in S_n\), there is an \(a\in\R^n\) such that
		\(f=\sum_{j=1}^n a_jk_{X_j}\).  The reproducing property gives
		\[
		f(X_i)=(Ka)_i,
		\qquad
		\norm{f}_{\Hk}^2=a^\top Ka.
		\]
		Thus, \((a,r)\) is feasible in problem~\eqref{eq:ndat-tractable}, and the two
		objective values are identical.
		
		Conversely, let \((a,r)\) be feasible in
		problem~\eqref{eq:ndat-tractable} and set
		\(f=\sum_{j=1}^n a_jk_{X_j}\).  Then \(f\in S_n\),
		\(f(X_i)=(Ka)_i\), and
		\(\norm{f}_{\Hk}^2=a^\top Ka\le r^2\).  Thus \((f,r)\) is feasible in
		problem~\eqref{eq:ndat-estimator} with the same objective value.  The two
		maps preserve feasibility and objective values in both directions, so the
		problems have the same infimum and carry minimizers to minimizers.
	\end{proof}
	
	\subsection{Proof of Corollary~\ref{cor:breakdown-radius}: \nameref*{cor:breakdown-radius}}
	
	\begin{proof}
		Independence and \(\E\varepsilon=0\) give, for every \(h\in\Hk\),
		\[
		\E[Yh(X)]
		=
		\E[f^*(X)h(X)]
		=
		\inner{T_kf^*}{h}_{\Hk}.
		\]
		Since
		\(\inner{T_kf^*}{f^*}_{\Hk}=\norm{f^*}_{L^2(\mu)}^2>0\), we have
		\(\norm{T_kf^*}_{\Hk}>0\) and \(\E|Y|>0\).
		
		For \(t\downarrow0\),
		\[
		\left|\E|Y-th(X)|-\E|Y|\right|
		\le t\norm{h}_{\Hk}.
		\]
		Expanding each objective along \(f=th\) therefore shows that their right
		directional derivatives at zero are, respectively,
		\begin{equation}
			\label{eq:breakdown-derivative-at}
			-2\inner{T_kf^*}{h}_{\Hk}
			+2\delta\E|Y|\norm{h}_{\Hk}
		\end{equation}
		and
		\begin{equation}
			\label{eq:breakdown-derivative-ndat}
			-2\inner{T_kf^*}{h}_{\Hk}
			+2\delta\left(\E|Y|-\E|\varepsilon|\right)\norm{h}_{\Hk}.
		\end{equation}
		
		\proofstep{\textbf{Threshold characterization.}}
		The adversarial training objective is convex.  The noise-debiased adversarial training
		objective is also convex: its
		pair formulation is jointly convex, and its derivative with respect to
		\(r\ge\norm{f}_{\Hk}\) is
		\[
		2\delta\left(
		\E|Y-f(X)|-\E|\varepsilon|+\delta r
		\right)\ge0
		\]
		by Lemma~\ref{lem:absolutelossincrease}.  Its partial minimum is therefore
		attained at \(r=\norm{f}_{\Hk}\).
		
		Convexity and RKHS duality applied to
		\eqref{eq:breakdown-derivative-at} give
		\[
		0\text{ is a minimizer}
		\quad\Longleftrightarrow\quad
		\norm{T_kf^*}_{\Hk}\le\delta\E|Y|.
		\]
		The same argument applied to \eqref{eq:breakdown-derivative-ndat} gives
		\[
		0\text{ is a minimizer}
		\quad\Longleftrightarrow\quad
		\norm{T_kf^*}_{\Hk}
		\le\delta\left(\E|Y|-\E|\varepsilon|\right).
		\]
		If either inequality is strict, the corresponding directional derivative is
		positive in every nonzero direction, so zero is the unique minimizer.  These
		two equivalences prove the threshold statements, including the convention
		\(\delta^{\rm ND}=\infty\).
		
		\proofstep{\textbf{Comparison of the thresholds.}}
		The definitions give
		\[
		\frac{\delta^{\rm ND}}{\delta^{\rm AT}}
		=
		\frac{\E|Y|}{\E|Y|-\E|\varepsilon|}.
		\]
		When \(\E|Y|=\E|\varepsilon|\), both sides are interpreted as infinity.
		Applying Lemma~\ref{lem:absolutelossincrease} with \(f=0\) gives
		\[
		0
		\le
		\E|Y|-\E|\varepsilon|
		\le
		\kappa\norm{f^*}_{L^2(\mu)}^2 ,
		\]
		which gives the stated lower bound and completes the proof.
	\end{proof}
	
	\section{Additional Numerical Experiments}
	This section reports two supporting synthetic experiments and the
	implementation details for the real-data study in the main text.

	\subsection{Sharpness of the approximation exponent}
	\label{sec:synth-candidate}
	Theorems~\ref{thm:full-approximation} and
	\ref{thm:full-at-matching-lower} give the sharp approximation order
	\(\delta^{(\alpha+1/2)/(\alpha+1)}\), uniformly over
	\(\norm{g}_{\Hk}\le1\).  We examine this order in the diagonal model used in
	the proof of Theorem~\ref{thm:full-at-matching-lower}, with
	\(\mathbb P(X=j)\propto j^{-\beta}\) and
	\(k(j,\ell)=\mathbf 1\{j=\ell\}\), for \(\beta=1.2\) and \(2\).  For each
	of \(14\) logarithmically spaced values of \(\delta\) from \(10^{-7}\) to
	\(10^{-1}\), we evaluate a logarithmic grid of single-coordinate source
	functions and \(400\) normalized random mixtures with \(\norm{g}_{\Hk}=1\), set
	\(f^*=T_k^{1/4}g\), use the bounded symmetric noise from that proof, and solve
	the population adversarial training problem.  We
	use \(200{,}000\) coordinates and report the largest approximation error over
	this finite collection because the theorem concerns the worst case over \(g\).

	\begin{figure}[!htbp]
		\centering
		\includegraphics[width=\textwidth]{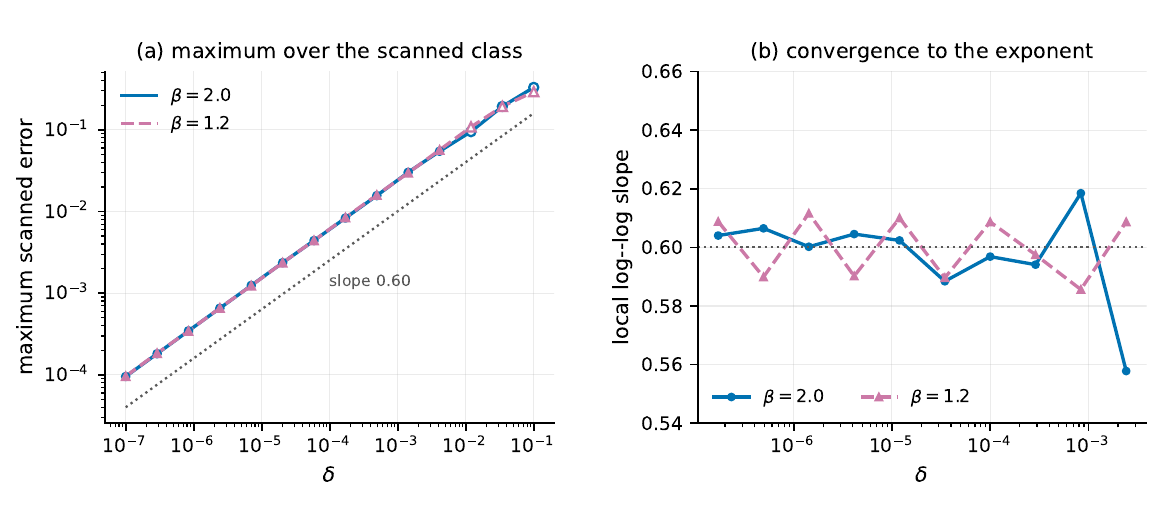}
		\caption{(a) Approximation errors under two eigenvalue decay rates; the
			dotted reference has slope
			\((\alpha+1/2)/(\alpha+1)=0.6\). (b) Corresponding local slopes.}
		\label{fig:candidate}
	\end{figure}

	As shown in Figure~\ref{fig:candidate}, the fitted slopes are \(0.592\) for
	\(\beta=2\) and \(0.591\) for \(\beta=1.2\).  Restricting the fit to radii whose
	maximizers are interior single-coordinate functions gives slopes \(0.599\) and
	\(0.600\), respectively, matching the theoretical exponent \(0.6\).

	\subsection{Sample size at a fixed robustness radius}
	\label{sec:synth-size}
	We fix \(\delta\) and increase the sample size from \(200\) to \(3200\).  The
	values \(\delta=0.15\) and \(0.40\) lie below and above the population
	collapse threshold \(\delta^{\rm AT}=0.2834\), respectively.  Each setting is
	repeated \(25\) times, and both procedures use all \(n\) observations.  As in
	Section~\ref{sec:synth-snr}, noise-debiased adversarial training uses the known
	Gaussian value of \(\E|\varepsilon|\), isolating the effect of increasing the
	sample size.  The population errors are obtained by log-scale interpolation
	of the population radius path computed from the fixed Monte Carlo sample of
	size \(50{,}000\) used in the main synthetic study.

	\begin{figure}[!htbp]
		\centering
		\includegraphics[width=0.92\textwidth]{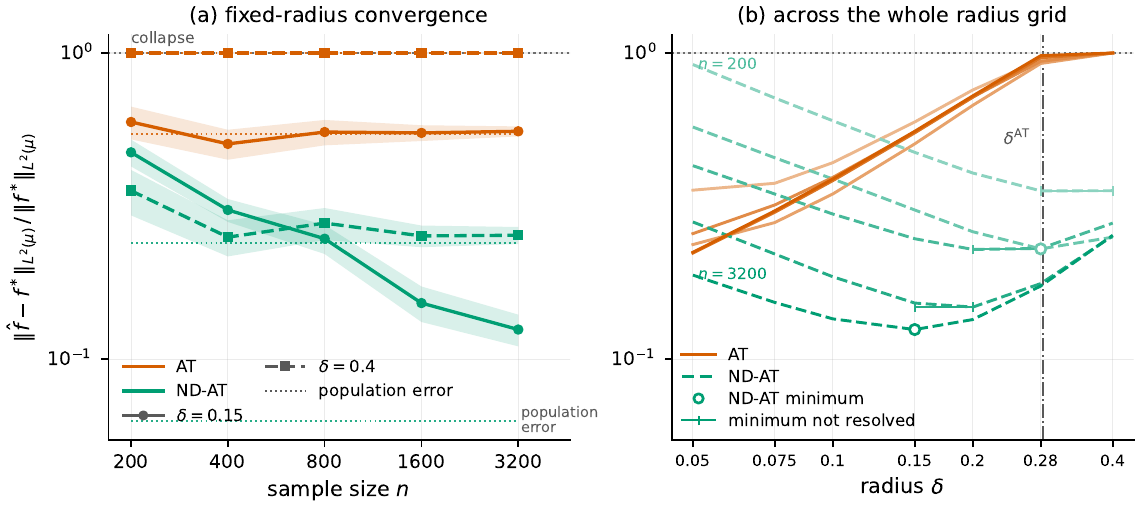}
		\caption{Prediction error as the sample size increases. (a) Relative
			prediction error below the collapse threshold (\(\delta=0.15\)) and above it
			(\(\delta=0.40\)), with bands of \(\pm2\) standard errors.  The
			procedure-specific dotted lines are the corresponding population errors,
			and the gray dotted
			line is the error of the zero predictor. (b) Relative prediction error as a
			function of \(\delta\) for each sample size; the vertical line marks the
			adversarial training collapse threshold.}
		\label{fig:samplesize}
	\end{figure}

	At \(\delta=0.40\), adversarial training returns the zero predictor for every
	sample size, so its relative error remains one.  The noise-debiased error
	decreases from \(0.355\) at \(n=200\) to \(0.254\) at \(n=3200\), close to
	its population error of \(0.240\).  At \(\delta=0.15\), the adversarial
	training error is already close to its population value at \(n=200\), whereas
	the noise-debiased error decreases from \(0.474\) to \(0.125\).  The results
	show that increasing the sample size reduces estimation error but does not
	remove the population effect of a fixed robustness radius.

	\subsection{Real-data implementation details}
	The same \(20\) random splits are used for all methods.  Each method receives
	\(800\) observations.  Noise-debiased adversarial training assigns \(240\)
	observations to preliminary KRR, \(120\) to estimating
	\(\E|\varepsilon|\), and \(440\) to the optimization in
	\eqref{eq:ndat-estimator}.  Adversarial training and both KRR procedures use
	all \(800\) observations.  The means and standard deviations computed from the
	preliminary observations are used to standardize the covariates and responses
	in every subsample, and the same observations determine the common radial basis
	function kernel bandwidth.  Five-fold cross-validation selects the preliminary
	KRR penalty.  The penalty for KRR(CV) is selected using the \(560\) observations
	assigned to noise estimation and optimization, after which KRR(CV) is computed
	from all \(800\) observations.
	
	For Diamonds, a random subset of \(4{,}000\) observations is drawn before
	each split.  Clean RMSE is computed on the full independent test set, which
	contains \(1{,}194\) observations for Communities and Crime and \(2{,}000\)
	for Abalone and Diamonds.  For all three data sets, the training radius ranges
	from \(0.05\) to \(0.50\) in increments of \(0.05\).  An estimator is
	classified as zero when its RKHS norm is below \(10^{-8}\).  The plug-in
	collapse boundary is computed on the \(120\) observations used to estimate
	\(\E|\varepsilon|\) by
	\[
	\frac{
	\norm{N_e^{-1}\sum_{j=1}^{N_e}
	\widehat f^{\rm KRR}(X_j^e)k_{X_j^e}}_{\Hk}}
	{N_e^{-1}\sum_{j=1}^{N_e}|Y_j^e|}.
	\]
	
	For the attack experiment, adversarial training, noise-debiased adversarial
	training, and KRR with \(\lambda=\delta^2\) are computed at \(\delta=0.1\).
	This feature-space training radius is below the estimated adversarial-training
	collapse boundary in every split.  KRR(CV) uses its cross-validated penalty.
	The input-space attack radius \(\rho\) is a separate evaluation parameter and
	is not identified with \(\delta\).  PGD uses input perturbation radii
	\(0,0.05,\ldots,0.50\), with \(30\) projected gradient steps
	of size \(\rho/8\).  One run starts at zero, and two start from random points at
	distance \(\rho/2\).  For each test observation, PGD separately approximates
	\[
	\sup_{\norm{\Delta x_{\rm cont}}_2\le\rho}
	\bigl(Y_i-f(X_i+\Delta x)\bigr)^2,
	\]
	and we record the largest squared error encountered across all iterations and
	starting points.  The reported PGD RMSE is the square root of the average of
	these observation-specific maxima.  The same subset of \(400\) test
	observations and the same random initial perturbations are used for every
	method.  The attack varies all \(99\) Communities and Crime
	covariates, the seven numerical Abalone covariates, and the six physical
	Diamonds covariates.  The categorical covariates and responses remain fixed,
	and no additional coordinatewise box constraints are imposed.

	\FloatBarrier

	\bibliographystyle{apalike}
	\bibliography{references}
	
\end{document}